%% file: main.tex
\documentclass[runningheads]{llncs}
\usepackage{url}            
\usepackage{booktabs}       
\usepackage{amsfonts}       
\usepackage{amsmath}
\usepackage{graphicx}
\usepackage{comment}
\usepackage{txfonts}
\usepackage{xcolor}
\usepackage[linkcolor=red,citecolor=blue,colorlinks=true]{hyperref}
\usepackage{cleveref}
\crefalias{condition}{equation}
\crefname{condition}{condition}{conditions}
\Crefname{condition}{Condition}{Conditions}
\creflabelformat{condition}{#2{\upshape(#1)}#3} 
\usepackage{subfigure}
\usepackage{multirow}
\usepackage{wrapfig}
\usepackage{threeparttable}
\usepackage{graphicx}
\usepackage{tikz}
\usetikzlibrary{shapes.misc, positioning}
\usetikzlibrary{decorations.text}

\makeatletter
\renewcommand\section{\@startsection{section}{1}{\z@}%
	{-8\p@ \@plus -4\p@ \@minus -4\p@}%
	{6\p@ \@plus 4\p@ \@minus 4\p@}%
	{\normalfont\large\bfseries\boldmath
		\rightskip=\z@ \@plus 8em\pretolerance=10000 }}
\renewcommand\subsection{\@startsection{subsection}{2}{\z@}%
	{-8\p@ \@plus -4\p@ \@minus -4\p@}%
	{6\p@ \@plus 4\p@ \@minus 4\p@}%
	{\normalfont\normalsize\bfseries\boldmath
		\rightskip=\z@ \@plus 8em\pretolerance=10000 }}
\renewcommand\subsubsection{\@startsection{subsubsection}{3}{\z@}%
	{-4\p@ \@plus -4\p@ \@minus -4\p@}%
	{-1.5em \@plus -0.22em \@minus -0.1em}%
	{\normalfont\normalsize\bfseries\boldmath}}
\makeatother

\input{macros}

\newcommand{\uniqq}{{\sf UniQQ}}
\newcommand{\inlsec}[1]{\smallskip\noindent\textbf{#1.}}

\begin{document}
	\setlength{\abovedisplayskip}{8pt}
	\setlength{\belowdisplayskip}{8pt}

\title{Unifying Qualitative and  Quantitative Safety Verification of DNN-Controlled  Systems}

\titlerunning{Unifying Qualitative and  Quantitative Safety Verification of DNN-Controlled  Systems}

\author{Dapeng Zhi\inst{1} \and
Peixin Wang$^\dag$\inst{2} \and
Si Liu\inst{3} \and Luke Ong\inst{2} \and Min Zhang\inst{1}}

\institute{East China Normal University \\
	\email{zhi.dapeng@163.com} \\ \email{zhangmin@sei.ecnu.edu.cn} \\
	 \and Nanyang Technological University \\
	\email{\{peixin.wang,luke.ong\}@ntu.edu.sg} 
\and ETH Zurich \\
\email{si.liu@inf.ethz.ch}
}


\maketitle      

\def\thefootnote{\dag}\footnotetext{Corresponding author \\}\def\thefootnote{\arabic{footnote}}   
\def\thefootnote{}\footnotetext{This work is a technical report for the paper with the same name to appear in the 36th International Conference on Computer Aided Verification (CAV 2024).}\def\thefootnote{\arabic{footnote}}

\begin{abstract}

The rapid advance of deep reinforcement learning techniques enables the oversight of safety-critical systems through the utilization of Deep Neural Networks (DNNs). 
This underscores the pressing need to promptly establish certified safety guarantees for such DNN-controlled systems.
Most of the existing verification approaches rely on qualitative approaches, predominantly employing reachability analysis. However, qualitative verification proves inadequate for DNN-controlled systems as their behaviors exhibit stochastic tendencies when operating in open and adversarial environments.
In this paper, we propose a novel framework for unifying both qualitative and quantitative safety verification problems of DNN-controlled systems. This is achieved by formulating the verification tasks as the synthesis of valid neural barrier certificates (NBCs).
Initially, the framework seeks to establish almost-sure safety guarantees through qualitative verification.
In cases where qualitative verification fails,
our quantitative verification method is invoked, yielding precise lower and upper bounds on probabilistic safety
across both infinite and finite time horizons.
To facilitate the synthesis of  NBCs, we  
introduce their $k$-inductive variants.  We also devise a simulation-guided approach for training NBCs, aiming to achieve tightness in computing precise certified lower and upper bounds.  
We prototype our approach into a tool called \uniqq\ and showcase its efficacy on four classic DNN-controlled systems. 

\end{abstract}

\section{Introduction}
\label{sec:intro}

\input{introduction}

\section{Preliminaries}
\label{sec:prelim}
\input{prelim}

\section{Verification Problem and Our Framework}
\label{sec:overview}
\input{overview}

\section{Qualitative and Quantitative Safety Verification}
\label{sec:bc}
\input{barrier}

\section{Relaxed $k$-Inductive Barrier Certificates}
\label{sec:k-induction}

\input{k-induction}

\section{Synthesis of Neural Barrier Certificates}
\label{sec:neural}

\input{neural}

\section{Evaluation}
\label{sec:experiments}

\input{experiment}

\section{Related Work}
\vspace{-1mm}
\label{sec:related}
\input{related}

\section{Conclusion and Future Work}
\label{sec:conclusion}
\input{conclusion}

\section*{Acknowledgments}We thank the anonymous reviewers of CAV 2024 for their valuable comments. The work has been supported by the NSFC Programs (62161146001, 62372176), Huawei Technologies Co., Ltd., the Shanghai International Joint Lab (22510750100), the Shanghai Trusted Industry Internet Software Collaborative Innovation Center, and the National Research Foundation, Singapore, under its RSS Scheme (NRF-RSS2022-009).

 \bibliographystyle{splncs04}
 \bibliography{reference}

\clearpage
\input{appendix}

\end{document}

%% file: macros.tex
\newcommand\Rset{{\mathbb{R}}}
\newcommand\Nset{{\mathbb{N}}}
\newcommand\Zset{{\mathbb{Z}}}
\newcommand\probm{{\mathbb{P}}}
\newcommand\expv{{\mathbb{E}}}

\newcommand\calL{{\mathcal{L}}}

\newcommand{\supp}[1]{\textsf{supp}\left({#1}\right)}

\usepackage{tikz}
\usetikzlibrary{arrows.meta,patterns}
\usetikzlibrary{ipe} 

\usepackage{newtxmath}

\tikzstyle{ipe stylesheet} = [
ipe import,
even odd rule,
line join=round,
line cap=butt,
ipe pen normal/.style={line width=0.4},
ipe pen heavier/.style={line width=0.8},
ipe pen fat/.style={line width=1.2},
ipe pen ultrafat/.style={line width=2},
ipe pen normal,
ipe mark normal/.style={ipe mark scale=3},
ipe mark large/.style={ipe mark scale=5},
ipe mark small/.style={ipe mark scale=2},
ipe mark tiny/.style={ipe mark scale=1.1},
ipe mark normal,
/pgf/arrow keys/.cd,
ipe arrow normal/.style={scale=7},
ipe arrow large/.style={scale=10},
ipe arrow small/.style={scale=5},
ipe arrow tiny/.style={scale=3},
ipe arrow normal,
/tikz/.cd,
ipe arrows, 
<->/.tip = ipe normal,
ipe dash normal/.style={dash pattern=},
ipe dash dotted/.style={dash pattern=on 1bp off 3bp},
ipe dash dashed/.style={dash pattern=on 4bp off 4bp},
ipe dash dash dotted/.style={dash pattern=on 4bp off 2bp on 1bp off 2bp},
ipe dash dash dot dotted/.style={dash pattern=on 4bp off 2bp on 1bp off 2bp on 1bp off 2bp},
ipe dash normal,
ipe node/.append style={font=\normalsize},
ipe stretch normal/.style={ipe node stretch=1},
ipe stretch normal,
ipe opacity 10/.style={opacity=0.1},
ipe opacity 30/.style={opacity=0.3},
ipe opacity 50/.style={opacity=0.5},
ipe opacity 75/.style={opacity=0.75},
ipe opacity opaque/.style={opacity=1},
ipe opacity opaque,
]

\definecolor{red}{rgb}{1,0,0}
\definecolor{blue}{rgb}{0,0,1}
\definecolor{green}{rgb}{0,1,0}
\definecolor{yellow}{rgb}{1,1,0}
\definecolor{orange}{rgb}{1,0.647,0}
\definecolor{gold}{rgb}{1,0.843,0}
\definecolor{purple}{rgb}{0.627,0.125,0.941}
\definecolor{gray}{rgb}{0.745,0.745,0.745}
\definecolor{brown}{rgb}{0.647,0.165,0.165}
\definecolor{navy}{rgb}{0,0,0.502}
\definecolor{pink}{rgb}{1,0.753,0.796}
\definecolor{seagreen}{rgb}{0.18,0.545,0.341}
\definecolor{turquoise}{rgb}{0.251,0.878,0.816}
\definecolor{violet}{rgb}{0.933,0.51,0.933}
\definecolor{darkblue}{rgb}{0,0,0.545}
\definecolor{darkcyan}{rgb}{0,0.545,0.545}
\definecolor{darkgray}{rgb}{0.663,0.663,0.663}
\definecolor{darkgreen}{rgb}{0,0.392,0}
\definecolor{darkmagenta}{rgb}{0.545,0,0.545}
\definecolor{darkorange}{rgb}{1,0.549,0}
\definecolor{darkred}{rgb}{0.545,0,0}
\definecolor{lightblue}{rgb}{0.678,0.847,0.902}
\definecolor{lightcyan}{rgb}{0.878,1,1}
\definecolor{lightgray}{rgb}{0.827,0.827,0.827}
\definecolor{lightgreen}{rgb}{0.565,0.933,0.565}
\definecolor{lightyellow}{rgb}{1,1,0.878}
\definecolor{black}{rgb}{0,0,0}
\definecolor{white}{rgb}{1,1,1}

%% file: introduction.tex
The widespread adoption of deep reinforcement learning techniques has propelled advancements in autonomous systems, endowing them with adaptive decision-making capabilities by Deep Neural Networks (DNNs)~\cite{lillicrap2015continuous}.
Ensuring the safety of these DNN-controlled systems emerges as a critical concern, necessitating the provision of certified safety guarantees.
Formal methods, renowned for their rigorousness and automaticity in delivering verified safety assurances, stand as a promising means to address this concern.
However, 
most of the existing formal verification approaches rely on qualitative approaches, predominantly employing reachability analysis \cite{seshia2022toward}. 
Despite their significance, qualitative results fall short for DNN-controlled systems  due to the constant influence of various uncertainties from different sources, such as environment noises \cite{SAMDP}, unreliable sensors \cite{sensors}, and even malicious attacks \cite{optimal_attack}. 
When qualitative  verification fails, it becomes both desirable and practical to obtain quantitative guarantees, including quantified lower and upper bounds on the safety probabilities of the systems. 
This necessitates the use of quantitative verification engines \cite{seshia2022toward}.

Quantitative verification has proven its efficacy in enhancing the design and deployment across a variety of applications, including
autonomous systems \cite{kwiatkowska2022probabilistic}, self-adaptive systems \cite{calinescu2012self}, distributed  communication protocols \cite{verify_communication_protocols}, and probabilistic programs \cite{WinklerGK22}. 
These applications are commonly modeled using automata-based quantitative formalisms \cite{hahn20192019}, such as Markov chains, timed automata, and hybrid automata, and undergo verification using tools such as  Prism \cite{kwiatkowska2011prism} and \textsc{Storm} \cite{hensel2021probabilistic}. 
Nonetheless, the quantitative verification of DNN-controlled systems is challenging due to the incorporation of intricate and almost inexplicable  decision-making models by DNNs \cite{samek2021explaining}. 
Compounding the issue, the difficulty is amplified by the continuous and infinite state space, as well as the non-linear dynamics inherent in DNN-controlled systems.
First, building a faithful automata-based probabilistic model for a DNN-controlled system is challenging. This difficulty arises as one cannot predict the action a DNN might take until a specific state is provided, and exhaustively enumerating all continuous states is impractical.
Second, even if such a model is constructed under certain constraints, such as bounded steps \cite{bacci2020probabilistic} and state abstractions \cite{jin2022trainify}, verification is susceptible to state exploration issues---a well-known problem in model checking\cite{tschaikowski2014tackling}. For instance, the verification process can take up to 50  minutes for just 7 steps~\cite{bacci2020probabilistic}.

Leveraging barrier certificates (BCs) for verification emerges as a promising technique for formally establishing the safety of non-linear and stochastic systems \cite{prajna2007framework,lavaei2022safety}.
A BC partitions the state space of the system into two parts, ensuring that all trajectories starting from a given initial set, located within one side of the BC, cannot reach a given set of states (deemed to be unsafe), located on the other side, almost surely (i.e., with probability $1$) or with probability at least $p\in [0,1)$.
Once a BC is computed, it can be used to certify systems' safety properties either qualitatively or quantitatively. Recently, studies have shown that BCs can be implemented and trained in neural forms  called Neural Barrier Certificates (NBCs). NBCs facilitate the synthesis of BCs and improve their expressiveness  \cite{zhao2020synthesizing,mathiesen2022safety,xia2022accelerated,meng2021reactive,AbateAEGP21}. A relevant survey is delegated to \cite{dawson2023safe}.

In this paper, we propose a unified framework for both qualitatively and quantitatively verifying the safety of DNN-controlled systems by leveraging NBCs. 
The key idea is to reduce both qualitative and quantitative verification problems into a cohesive synthesis task of their respective NBCs. 
Specifically, we first seek to establish almost-sure safety guarantees through qualitative verification. In cases where qualitative
verification fails, our quantitative verification method is invoked, yielding precise
lower and upper bounds on probabilistic safety across both infinite and finite time
horizons.

We also establish relevant theoretical results. In qualitative verification, we prove that an NBC satisfying corresponding conditions serves as a qualitative safety certificate. In the quantitative counterpart, we establish that
valid NBCs can be utilized to calculate certified upper and lower bounds on the probabilistic safety of systems, encompassing both infinite and finite time horizons.
For infinite time horizons, as the lower bounds on probabilistic safety approach zero, indicating a decreasing trend in safety probabilities along the time horizon, 
we provide both linearly and exponentially decreasing lower and upper bounds on the safety probabilities over finite time horizons.

To facilitate the synthesis of valid NBCs, we further relax their constraints by defining their $k$-inductive variants ~\cite{AnandM0Z22}. This necessitates the conditions to be inductive for $k$-compositions of the transition relation within a specified bound $k$~\cite{DBLP:conf/sas/Brain0KS15}. Consequently, synthesizing a qualified NBC becomes more manageable under these $k$-inductive conditions, while ensuring safety guarantees. 
As valid NBCs are not unique and yield different certified bounds, we devise a simulation-guided approach to train potential NBCs. This approach  aims to enhance their capability to produce more precise certified bounds.
Specifically, we estimate safety probabilities through simulation.
The differences between the simulation results and the bounds provided by potential NBCs are incorporated into the loss function. 
This integration can yield more precise certified bounds after potential NBCs are successfully validated.

We prototype our approach into a tool, called \uniqq,
and apply it to four classic DNN-controlled problems. 
The experimental results showcase the effectiveness of our unified verification approach in delivering both qualitative and quantitative safety guarantees across diverse noise scenarios. 
Additionally, the results underscore the efficacy of $k$-inductive variants in reducing verification overhead, by 25\% on average, and that of our simulation-based training method in yielding tighter safety bounds, with an up to 47.5\% improvement over ordinary training approaches.

\inlsec{Contributions} Overall, we make the following contributions.\begin{enumerate}
    \item We present a novel framework that unifies both  qualitative and quantitative safety verification of  DNN-controlled systems by reducing these verification problems into the cohesive task of synthesizing NBCs.
\item We establish relevant theoretical results,  including new constraints of NBCs for both qualitative and quantitative safety verification and the associated  lower and upper bounds for safety probabilities in both linear and exponential forms.
    \item To accelerate training, we relax the constraints of NBCs by introducing their $k$-inductive variants.  We also  present a simulation-guided  approach designed to train potential NBCs to compute safety bounds as tightly as possible.
    \item We develop a prototype of our approach,  
    showcasing its efficacy across four classic DNN-controlled systems.
\end{enumerate}

\noindent All omitted proofs and supplementary experimental results are given in the Appendix.

%% file: prelim.tex
Let $\Nset$, $\Zset$, and $\Rset$ be the sets of natural numbers, integers, and real numbers, respectively.

\subsection{DNN-controlled Systems}
We consider DNN-controlled systems where the control policies are implemented by deep neural networks and suppose the networks are trained for specific tasks. 
Formally, a DNN-controlled system is a tuple $M=(S,S_0,A,\pi,f,R)$, where $S\subseteq \Rset^n$ is the set of (possibly continuous and infinite) system states, $S_0\subseteq S$ is the set of initial  states, $A$ is the set of actions, $\pi:S\rightarrow A$ is the trained policy implemented by a neural network, $f:S\times A\rightarrow S$ is the system dynamics, and $R:S\times A\times S\rightarrow \Rset$ is the reward function.

\inlsec{Trajectories} A trained DNN-controlled system $M=(S,S_0,A,\pi,f,R)$ is a decision-making system that continuously interacts with the environment. 
At each time step $t\in\Nset_0$, it observes a state $s_t$ and feeds $s_t$ into its planted NN to compute the optimal action $a_t=\pi(s_t)$ that shall be taken. Action $a_t$ is then performed, which transits $s_t$ into the next state $s_{t+1}=f(s_t,a_t)$ via the system dynamics $f$ and earns a reward $r_{t+1} = R(s_t,a_t,s_{t+1})$. Given an initial state $s_0 \in S_0$, a sequence of states generated during interaction is called a \emph{trajectory}, denoted as $\omega=\{s_t\}_{t\in\Nset_0}$.  To ease the notation, we denote by $\omega_t$ the $t$-th element of $\omega$, i.e., $\omega_t=s_t$, 
and by $\Omega$ the set of all trajectories.

\inlsec{State Perturbations}
As DNN-controlled systems collect state information via
sensors, uncertainties inevitably originate from sensor errors,
equipment inaccuracy, or even adversarial attacks~\cite{SAMDP,vulnerability_DRL}.
Therefore, the observed states of the systems can be perturbed and actions are computed based on the perturbed states. Formally, an observed state at time step $t$ is $\hat{s}_t:=s_t+\delta_t$ where $\delta_t\sim \mu$ is a random noise and $\mu$ is a probability distribution over $\Rset^n$. We denote by $W:=\supp{\mu}$ the support of $\mu$.
Due to perturbation, the actual successor state is 
$s_{t+1}:=f(s_t,\hat{a}_t)$ with $\hat{a}_t:=\pi(\hat{s}_t)$ and the reward is $r_{t+1}:=R(s_t,\hat{a}_t,s_{t+1})$. Note that the successor state and the reward are calculated according to the actual state and the action on the perturbed state, and this update is common~\cite{SAMDP}. 
We then denote a DNN-controlled system $M$ perturbed by a noise distribution $\mu$ as $M_\mu=(S,S_0,A,\pi,f,R,\mu)$.

\inlsec{Assumptions} Given a DNN-controlled system $M=(S,S_0,A,\pi,f,R)$, we assume that the state space $S$ is compact in the Euclidean topology of $\Rset^n$, its system dynamics $f$ and trained policy $\pi$ are Lipschitz continuous. This assumption is common in control theory \cite{reach_avoid_martingale}. For perturbation, we assume that the noise distribution $\mu$ either has bounded support or is a product of independent univariate distributions.

\inlsec{Probability Space} Given a DNN-controlled system $M_\mu=(S,S_0,A,\pi,f,R,\mu)$, 
for each initial state $s_0\in S_0$, there exists a \emph{probability space} $(\Omega_{s_0},\mathcal{F}_{s_0},\probm_{s_0})$ 
such that $\Omega_{s_0}$ is the set of all trajectories  starting from $s_0$ by the environmental interaction, $\mathcal{F}_{s_0}$ is a $\sigma$-algebra over $\Omega_{s_0}$ (i.e., a collection of subsets of $\Omega_{s_0}$ that contains the empty set $\emptyset$ and is closed under complementation and countable union), and $\probm_{s_0}:\mathcal{F}_{s_0}\rightarrow [0,1]$ is a probability measure on $\mathcal{F}_{s_0}$. We denote 
the expectation operator in this probability space by $\expv_{s_0}$.

\subsection{Barrier Certificate and its Neural Implementation}

Barrier certificates (BCs) are powerful tools to certify the safety of continuous-time dynamical systems. 
In the following we describe the discrete-time BCs which this work is based upon. We refer readers to~\cite{PrajnaJ04,prajna2005necessity} for details about continuous-time BCs.

\begin{definition}[Discrete-time Barrier Certificates]\label{def:discrete-BC}
    Given a DNN-controlled system $M=(S,S_0,A,f,\pi,R)$ with an unsafe set $S_u\subseteq S$ such that $S_u\cap S_0=\emptyset$.  A \emph{discrete-time barrier certificate} is a real-valued function $B:S\rightarrow \Rset$ such that for some constant $\lambda\in (0,1]$, it holds that:
    \begin{align}
        & B(s)\le 0 \ &\text{for all }s\in S_0, \label[condition]{eq:dis-bc-initial} \\
       & B(s)>0 \ & \text{for all } s\in S_u,  \label[condition]{eq:dis-bc-unsafe} \\
       & B(f(s,\pi(s)))-B(s)+\lambda\cdot B(s)\le 0  & \text{for all } s \in S.  \label[condition]{eq:dis-bc-decrease}   
    \end{align}
\end{definition}

If there exists such a BC for the system $M$, then $M$ is safe, i.e., the system cannot reach a state in the unsafe set $S_u$ from the initial set $S_0$. The intuition is that: 
\Cref{eq:dis-bc-decrease} implies that for any $s\in S$ such that $B(s)\le 0$, $B(f(s,\pi(s)))\le 0$.
Since \Cref{eq:dis-bc-initial} asserts that the initial value of $B$ is not greater than zero, any trajectory $\omega\in \Omega_{s_0}$ starting from an initial state $s_0\in S_0$ cannot enter the unsafe set $S_u$, where $B(s) > 0$ (see \Cref{eq:dis-bc-unsafe}), thereby ensuring the safety of the system.

Finding a BC is restricted to the expressiveness of templates. For example, even if there exists a function satisfying \Cref{eq:dis-bc-initial,eq:dis-bc-unsafe,eq:dis-bc-decrease}, it may be not found under polynomial forms.
Recent work \cite{zhao2020synthesizing,PeruffoAA21,ZHAO_NBC} proposes a neural implementation of BCs as deep neural networks, leveraging the expressiveness of neural networks.
The neural implementation of a BC is called a neural barrier certificate (NBC), which consists of training and validation. First, a learner trains a neural network (NN) to fit over a finite set of samples the conditions for a BC.
After training, an NBC is then checked whether it meets the conditions. This is achieved by a verifier using SMT solvers~\cite{zhao2020synthesizing,PeruffoAA21} or other methods like Sum-of-Squares programming~\cite{ZHAO_NBC}.
If the validation result is false, a set of counterexamples can be generated for future training. This iteration is repeated until a trained candidate is validated or a given timeout is reached. This training and validation iteration is called  
CounterExample-Guided Inductive Synthesis (CEGIS)~\cite{AbateDKKP18}.

%% file: overview.tex
\subsection{Problem Statement}

We consider the safety of DNN-controlled systems from both qualitative and quantitative perspectives. 
Below we fix a DNN-controlled system $M_\mu=(S,S_0,A,\pi,f,R,\mu)$ and an unsafe set $S_u\subseteq S$ such that $S_0\cap S_u =\emptyset$ throughout the paper.

\begin{definition}[Almost-Sure Safety]\label{def:almost-safe}
 The system $M_\mu$ is \emph{almost-surely (a.s.) safe}, 
 if a.s. no trajectories starting from any initial state $s_0\in S_0$ enter $S_u$, i.e., 
    \begin{equation}
    \forall s_0\in S_0.\omega\in\Omega_{s_0} \Longrightarrow \omega_t \not\in S_u \ \forall t\in\Nset.\notag 
    \end{equation}
\end{definition}
This almost-sure safety is a qualitative property and we call it ``almost-sure'' due to the stochasticity from state perturbations. Since the almost-sure safety does not always exist with the increase of state perturbations, we propose the notion of probabilistic safety over infinite time horizons.

\begin{definition}[Probabilistic Safety over Infinite Time Horizons]\label{def:inf-prob-safe}
 The system $M_\mu$ is \emph{probabilistically safe over infinite time horizons} 
with  $[l_{\mathrm{inf}},u_{\mathrm{inf}}]$, where $0\leq  l_{\mathrm{inf}}\leq u_{\mathrm{inf}}\leq 1$,  
 if the probability of not entering $S_u$ falls into  $[l_{\mathrm{inf}},u_{\mathrm{inf}}]$ for all the  trajectories from any initial state $s_0\in S_0$, i.e.,  
	\begin{equation}
		\forall s_0\in S_0.\probm_{s_0}\left[\{\omega\in\Omega_{s_0} \mid \omega_t\not\in S_u\ \text{for all }t\in\Nset \}\right]\in [l_{\mathrm{inf}},u_{\mathrm{inf}}]. \notag 
	\end{equation}
\end{definition}

The probabilistic safety is a quantitative property and $l_{\mathrm{inf}},u_{\mathrm{inf}}$ are called \textit{lower} and \textit{upper} bounds on the safety probabilities over infinite time horizons, respectively. Once both bounds equal one, it implies the almost-sure safety. 
When the lower bound $l_{\mathrm{inf}}=0$, indicating that the system reaches the unsafe region at some time step $T<\infty$, it is significant to figure out how the safety probability decreases over the finite time horizon. Therefore, we present the probabilistic safety over finite time horizons as follows.

\begin{definition}[Probabilistic Safety over Finite Time Horizons]\label{def:fin-prob-safe}
The system $M_\mu$ is \emph{probabilistically safe over a finite time horizon $T\in [0,\infty)$} 
with $[l_{\mathrm{fin}},u_{\mathrm{fin}}]$, where 
$0\leq l_{\mathrm{fin}} \leq u_{\mathrm{fin}}\leq 1$, 
 if the probability of not entering $S_u$ within $T$ falls into $[l_{\mathrm{fin}},u_{\mathrm{fin}}]$ for all the trajectories starting from any initial state $s_0\in S_0$, 
 \begin{equation}
		\forall s_0\in S_0.\probm_{s_0}\left[\{ \omega\in\Omega_{s_0} \mid \omega_t\not\in S_u\ \text{for all }t\le T \}\right]\in [l_{\mathrm{fin}},u_{\mathrm{fin}}]. \notag 
	\end{equation}
\end{definition}

\vskip 2pt 
\noindent\textbf{Safety Verification Problems of DNN-Controlled Systems.} 
Consider a DNN-controlled system $M_\mu=(S,S_0,A,\pi,f,R,\mu)$ with an unsafe set $S_u\in S$ such that $S_0\cap S_u=\emptyset$. 
We formulate the qualitative and quantitative safety verification problems of $M_\mu$ as follows: 
\begin{enumerate}
    \item \textbf{Qualitative Verification (QV):} To answer  whether $M_\mu$ is almost-surely safe. 
    \item \textbf{Quantitative Verification over Infinite Time Horizons (QVITH):} To compute certified  lower and upper bounds $l_{\mathrm{inf}},u_{\mathrm{inf}}$ on the safety probability of $M_\mu$ over infinite time horizons. 
    \item \textbf{Quantitative Verification over Finite Time Horizons (QVFTH):} 
    To compute certified lower and upper bounds $l_{\mathrm{fin}},u_{\mathrm{fin}}$ on the safety probability of $M_\mu$ over a finite time horizon $T$.
\end{enumerate}

\subsection{Overview of Our Framework }
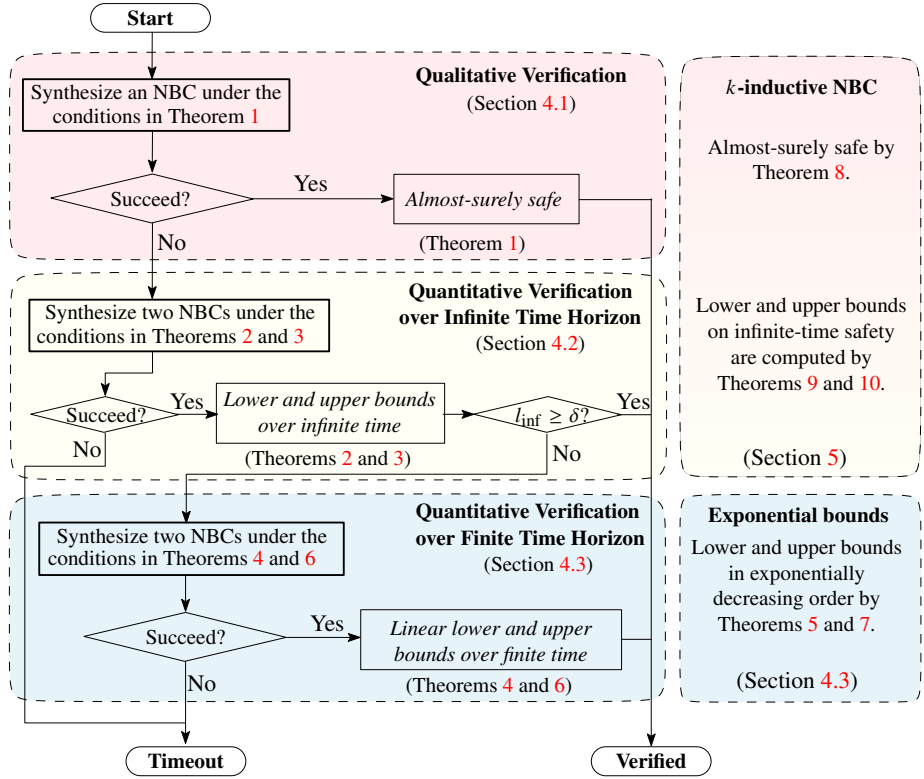
\begin{figure}[t]
\scalebox{0.89}{
\begin{tikzpicture}[ipe stylesheet]
	
	\filldraw[shift={(275.055, 711.998)}, xscale=1.0212, yscale=0.9166, ipe dash dashed, fill=pink!30]
	(0, 0)
	.. controls (112, 0) and (120, 0) .. (124, -2.6667)
	.. controls (128, -5.3333) and (128, -10.6667) .. (128, -24)
	.. controls (128, -37.3333) and (128, -58.6667) .. (128, -72)
	.. controls (128, -85.3333) and (128, -90.6667) .. (125.3333, -93.3333)
	.. controls (122.6667, -96) and (117.3333, -96) .. (74.6667, -96)
	.. controls (32, -96) and (-48, -96) .. (-90.6667, -96)
	.. controls (-133.3333, -96) and (-138.6667, -96) .. (-141.3333, -93.3333)
	.. controls (-144, -90.6667) and (-144, -85.3333) .. (-144, -72)
	.. controls (-144, -58.6667) and (-144, -37.3333) .. (-144, -24)
	.. controls (-144, -10.6667) and (-144, -5.3333) .. (-140, -2.6667)
	.. controls (-136, 0) and (-128, 0) .. (0, 0);
	\filldraw[shift={(465.802, 710.228)}, xscale=0.3751, yscale=1.8566, ipe dash dashed, top  color=pink!30, bottom color=lightyellow!30]
	(0, 0)
	.. controls (112, 0) and (120, 0) .. (124, -2.6667)
	.. controls (128, -5.3333) and (128, -10.6667) .. (128, -24)
	.. controls (128, -37.3333) and (128, -58.6667) .. (128, -72)
	.. controls (128, -85.3333) and (128, -90.6667) .. (125.3333, -93.3333)
	.. controls (122.6667, -96) and (117.3333, -96) .. (74.6667, -96)
	.. controls (32, -96) and (-48, -96) .. (-90.6667, -96)
	.. controls (-133.3333, -96) and (-138.6667, -96) .. (-141.3333, -93.3333)
	.. controls (-144, -90.6667) and (-144, -85.3333) .. (-144, -72)
	.. controls (-144, -58.6667) and (-144, -37.3333) .. (-144, -24)
	.. controls (-144, -10.6667) and (-144, -5.3333) .. (-140, -2.6667)
	.. controls (-136, 0) and (-128, 0) .. (0, 0);
	\filldraw[shift={(275.449, 618.693)}, xscale=1.0212, yscale=0.9166, ipe dash dashed, fill=lightyellow!30]
	(0, 0)
	.. controls (112, 0) and (120, 0) .. (124, -2.6667)
	.. controls (128, -5.3333) and (128, -10.6667) .. (128, -24)
	.. controls (128, -37.3333) and (128, -58.6667) .. (128, -72)
	.. controls (128, -85.3333) and (128, -90.6667) .. (125.3333, -93.3333)
	.. controls (122.6667, -96) and (117.3333, -96) .. (74.6667, -96)
	.. controls (32, -96) and (-48, -96) .. (-90.6667, -96)
	.. controls (-133.3333, -96) and (-138.6667, -96) .. (-141.3333, -93.3333)
	.. controls (-144, -90.6667) and (-144, -85.3333) .. (-144, -72)
	.. controls (-144, -58.6667) and (-144, -37.3333) .. (-144, -24)
	.. controls (-144, -10.6667) and (-144, -5.3333) .. (-140, -2.6667)
	.. controls (-136, 0) and (-128, 0) .. (0, 0);
	\filldraw[shift={(274.947, 524.762)}, xscale=1.0212, yscale=0.9166, ipe dash dashed, fill=lightblue!30]
	(0, 0)
	.. controls (112, 0) and (120, 0) .. (124, -2.6667)
	.. controls (128, -5.3333) and (128, -10.6667) .. (128, -24)
	.. controls (128, -37.3333) and (128, -58.6667) .. (128, -72)
	.. controls (128, -85.3333) and (128, -90.6667) .. (125.3333, -93.3333)
	.. controls (122.6667, -96) and (117.3333, -96) .. (74.6667, -96)
	.. controls (32, -96) and (-48, -96) .. (-90.6667, -96)
	.. controls (-133.3333, -96) and (-138.6667, -96) .. (-141.3333, -93.3333)
	.. controls (-144, -90.6667) and (-144, -85.3333) .. (-144, -72)
	.. controls (-144, -58.6667) and (-144, -37.3333) .. (-144, -24)
	.. controls (-144, -10.6667) and (-144, -5.3333) .. (-140, -2.6667)
	.. controls (-136, 0) and (-128, 0) .. (0, 0);
	\filldraw[shift={(465.848, 524.546)}, xscale=0.3751, yscale=0.9166, ipe dash dashed, fill=lightblue!30]
	(0, 0)
	.. controls (112, 0) and (120, 0) .. (124, -2.6667)
	.. controls (128, -5.3333) and (128, -10.6667) .. (128, -24)
	.. controls (128, -37.3333) and (128, -58.6667) .. (128, -72)
	.. controls (128, -85.3333) and (128, -90.6667) .. (125.3333, -93.3333)
	.. controls (122.6667, -96) and (117.3333, -96) .. (74.6667, -96)
	.. controls (32, -96) and (-48, -96) .. (-90.6667, -96)
	.. controls (-133.3333, -96) and (-138.6667, -96) .. (-141.3333, -93.3333)
	.. controls (-144, -90.6667) and (-144, -85.3333) .. (-144, -72)
	.. controls (-144, -58.6667) and (-144, -37.3333) .. (-144, -24)
	.. controls (-144, -10.6667) and (-144, -5.3333) .. (-140, -2.6667)
	.. controls (-136, 0) and (-128, 0) .. (0, 0);
	\node[ipe node, font=\small]
	at (300, 699.443) {\textbf{Qualitative Verification}};
	\draw[shift={(188.492, 660.313)}, cm={1.2951,-0.3212,1.3193,0.3118,(0,0)}]
	(0, 0) rectangle (32.908, -32.591);
	\node[ipe node, font=\small]
	at (170.675, 646.938) {Succeed?};
	\draw[shift={(290.717, 660.196)}, xscale=0.9799, yscale=0.655]
	(0, 0) rectangle (80, -32);
	\node[ipe node, font=\small]
	at (295.714, 646.753) {\textit{Almost-surely safe}};
	\draw[shift={(136.369, 607.07)}, xscale=1.4104, yscale=0.6843,line width=0.8]
	(0, 0) rectangle (90, -32);
	\node[ipe node, font=\small]
	at (143, 598.73) {Synthesize two NBCs under the };

	\node[ipe node, font=\small]
	at (137, 692.005) {Synthesize an NBC under the };

	\draw[shift={(140.546, 512.629)}, xscale=1.4652, yscale=0.6951,line width=0.8]
	(0, 0) rectangle (88, -32);

	\node[ipe node, anchor=north west]
	at (278.484, 471.18) {
		\begin{minipage}{106.652bp}\small\kern0pt
			\centering
			\textit{Linear lower and upper bounds over finite time}
		\end{minipage}
	};
	\node[ipe node, font=\small]
	at (278.062, 688.555) {\textbf{}};
	\node[ipe node, font=\small]
	at (321.653, 687.624) {(Section \ref{subsec:quali})};
	\node[ipe node, font=\small]
	at (297, 607.738) {\textbf{Quantitative Verification}};
	\node[ipe node, font=\small]
	at (293, 596.577) {\textbf{over Infinite Time Horizon }};
	\node[ipe node, font=\small]
	at (300, 514.437) {\textbf{Quantitative Verification}};
	\node[ipe node, font=\small]
	at (300, 503.276) {\textbf{over Finite Time Horizon }};
	\node[ipe node, font=\small]
	at (301.45, 628.266) {(Theorem \ref{thm:almost-safe})};
	\node[ipe node, font=\small]
	at (325.442, 585.549) {(Section \ref{subsec:quanti})};
\node[ipe node, font=\small]
at (227.609, 536.9) {(Theorems \ref{thm:inf-lower-safe} and  \ref{thm:inf-upper-safe})};
\node[ipe node, font=\small]
at (328.845, 492.062) {(Section \ref{subsec:quanti2})};
\node[ipe node, font=\small]
at (295.768, 440.794) {(Theorems \ref{thm:fin-lower-safe-linear} and  \ref{thm:fin-upper-safe-linear})};
\draw[shift={(276.664, 475.685)}, xscale=1.3806, yscale=0.7938]
(0, 0) rectangle (80, -32);
\draw[shift={(188.489, 678.515)}, xscale=0.8454, yscale=1.1092, -{ipe pointed[ipe arrow small]}]
(0, 0) -- (-0.119, -15.589);

\draw[shift={(188.489, 639.619)}, xscale=1.6981, yscale=0.9542, -{ipe pointed[ipe arrow small]}]
(0, 0) -- (-0.063, -33.566);

\draw[shift={(188.489, 585.223)}, xscale=1.1743, yscale=1.5765, -{ipe pointed[ipe arrow small]}]
(0, 0)
-- (0.03, -5.856) -- (-17.03, -5.856) -- (-17.03, -11.856);

\draw[shift={(355.489, 550.223)}, xscale=1.1743, yscale=1.5765, -{ipe pointed[ipe arrow small]}]
(0, 0)
-- (0, -10.856) -- (-130.03, -10.856) -- (-130.03, -24);

\draw[shift={(168.489, 550.687)}, xscale=1.1743, yscale=1.5765, -]
(0, 0)
-- (0, -9) -- (-29, -9) -- (-29, -78) -- (29,-78);


\draw[shift={(202.5, 490.612)}, yscale=0.8753, -{ipe pointed[ipe arrow small]}]
(0, 0)
-- (0, -17.765);
\draw[shift={(202.5, 453.818)}, xscale=0, yscale=2.3255, -{ipe pointed[ipe arrow small]}]
(0, 0) -- (0, -14.953);

\draw[shift={(215.543, 571.641)}, xscale=1.3806, yscale=0.7938]
(0, 0) rectangle (70, -32);

\node[ipe node, anchor=north west]
at (210.373, 568.139) {
	\begin{minipage}{106.28bp}\small\kern0pt
		\centering
		\textit{Lower and upper bounds over infinite time}
	\end{minipage}
};

\draw[shift={(168.489, 566)}, cm={1.2951,-0.3212,1.3193,0.3118,(0,0)}]
(0, 0) rectangle (23.908, -23.908);
\draw[shift={(356, 566))}, cm={1.2951,-0.3212,1.3193,0.3118,(0,0)}]
(0, 0) rectangle (23.908, -23.908);

\node[ipe node, font=\small]
at (342, 555.336) {$l_{\mathrm{inf}}\geq\delta?$};
\node[ipe node, font=\small]
at (152, 555.336) {Succeed?};

\draw[shift={(202.489, 474.543)}, cm={1.2951,-0.3212,1.3193,0.3118,(0,0)}]
(0, 0) rectangle (32.908, -32.591);

\node[ipe node, font=\small]
at (185.625, 461.169) {Succeed?};
\draw[shift={(229.919, 649.823)}, xscale=0.944, yscale=1.3858, -{ipe pointed[ipe arrow small]}]
(0, 0)
-- (64, 0);

\draw[shift={(232.536, 558.5)}, xscale=0.6574, yscale=47.0951, -{ipe pointed[ipe arrow small]}]
(-52, 0)
-- (-25, 0);

\draw[shift={(232.536, 558.5)}, xscale=0.6574, yscale=47.0951, -{ipe pointed[ipe arrow small]}]
(120, 0)
-- (140, 0);

\draw[shift={(245.231, 463.947)}, xscale=0.7086, yscale=3.09, -{ipe pointed[ipe arrow small]}]
(0, 0) -- (44, 00);
\draw[shift={(188.386, 720.428)}, xscale=1.2792, yscale=0.587, -{ipe pointed[ipe arrow small]}]
(0, 0)
-- (-0.063, -33.566);
\node[ipe node]
at (248.252, 653.142) {Yes};
\node[ipe node]
at (195.87, 560.822) {Yes};
\node[ipe node]
at (384.87, 560.822) {Yes};
\node[ipe node]
at (357.87, 538.822) {No};
\node[ipe node]
at (255.206, 466.694) {Yes};
\node[ipe node]
at (203.384, 441.137) {No};
\node[ipe node]
at (154.701, 540.278) {No};
\node[ipe node]
at (190.565, 628.303) {No};

\draw[shift={(399.674, 649.944)}, xscale=1.0922, yscale=0.7247, -{ipe pointed[ipe arrow small]}]
(0, 0) -- (0, -320);
\draw[shift={(369.049, 649.467)}, xscale=0.969]
(0, 0) -- (32, 0);
\draw[shift={(385.185, 558.319)}, xscale=0.4514, yscale=0.2583]
(0, 0) -- (32, 0);

\draw[shift={(387.201, 463.144)}, xscale=0.3935, yscale=-1.2281]
(0, 0)-- (32, 0);

\node[ipe node, font=\small]
at (139.5, 588.605) {conditions in Theorems \ref{thm:inf-lower-safe} and  \ref{thm:inf-upper-safe}};
\node[ipe node, font=\small]
at (148, 503.873) {Synthesize two NBCs under the };
\node[ipe node, font=\small]
at (144.257, 494.742) {conditions in Theorems  \ref{thm:fin-lower-safe-linear} and  \ref{thm:fin-upper-safe-linear}};
\draw[shift={(133.643, 700.496)}, xscale=1.4104, yscale=0.6843,line width=0.8]
(0, 0) rectangle (80, -32);
\node[ipe node, font=\small]
at (146.257, 682.023) {conditions in Theorem \ref{thm:almost-safe}};

\node[ipe node, anchor=north west]
at (415.306, 518.983) {
	\begin{minipage}{91.997bp}\small\kern0pt
		\textbf{\hspace{3mm}Exponential bounds}\\\vspace{-3mm}

\centering		Lower and upper bounds in exponentially decreasing order by 
Theorems \ref{thm:fin-lower-safe-exp} and \ref{thm:fin-upper-safe-exp}.
	\end{minipage}
};
\node[ipe node, anchor=north west]
at (417.308, 700.263) {
	\begin{minipage}{91.997bp}\small\kern0pt
		\centering \textbf{$k$-inductive NBC}\\
		\vspace{5mm}
		
		Almost-surely safe by Theorem \ref{thm:almost-safe-induc}. 
		
            \vspace{16mm}
		Lower and upper bounds on infinite-time safety are computed by Theorems \ref{thm:k-ind-lower} and  \ref{thm:k-induc-upper}.
	\end{minipage}
};
\node[ipe node]
at (438.643, 536.848) {(Section \ref{sec:k-induction})};
\node[ipe node]
at (435.83, 442.353) {(Section  \ref{subsec:quanti2})};
\node at (187.659, 726.473) [draw, rounded rectangle, minimum width=2cm] {\textbf{Start}};
\node at (202.659, 411.473) [draw, rounded rectangle, minimum width=2cm] {\textbf{Timeout}};
\node at (399.659, 411.473) [draw, rounded rectangle, minimum width=2cm] {\textbf{Verified}};
\end{tikzpicture}}
	\vspace{-4mm}
	\caption{Our unified verification framework.}
	\label{fig:unified-framework}
 \vspace{-3mm}
\end{figure}

We first provide an overview of our unified framework designed to address the three safety verification problems.
Our framework builds on two fundamental results: 
(i) all the problems can be reduced to the task of defining BCs under specific conditions, and the defined BCs can be used to certify almost-sure safety for \textbf{QV} or safety bounds for \textbf{QVITH} and \textbf{QVFTH}, respectively, and (ii) these BCs can be implemented and trained in neural forms. The fundamental results are presented in \Cref{sec:bc,sec:k-induction,sec:neural}, respectively.

The synthesis of NBCs has a preset timeout threshold, i.e., it will fail if NBCs cannot be successfully synthesized within the time threshold. 
The procedure of our framework is sketched in \Cref{fig:unified-framework}, which consists of the following three steps:

\vskip 1mm\noindent
 \textbf{Step 1: QV}. We try to synthesize an NBC satisfying conditions in~\Cref{thm:almost-safe}. If such an NBC is successfully synthesized, we can conclude that the system $M_\mu$ is almost-surely safe by~\Cref{thm:almost-safe} and  finish the verification.  Alternatively, we can resort to synthesizing a $k$-inductive NBC in \Cref{thm:almost-safe-induc} whose conditions are weaker than those in \Cref{thm:almost-safe}. If the synthesis fails, we proceed to quantitative verification.

\vskip 1mm\noindent
 \textbf{Step 2: QVITH}. We try to synthesize two NBCs  under the conditions in \Cref{thm:inf-lower-safe,thm:inf-upper-safe}, respectively. If the synthesis fails, a timeout will be reported and the process will be terminated.
 Otherwise, we can obtain the lower bound $l_{\mathrm{inf}}$ and the upper bound $u_{\mathrm{inf}}$ on probabilistic safety over infinite time horizons.  Alternatively, we can choose to synthesize the $k$-inductive variants of NBCs in \Cref{thm:k-ind-lower,thm:k-induc-upper}. If the lower bound $l_{\mathrm{inf}}$ is no less than some preset safety threshold $\delta\in (0,1)$, we terminate the verification.  The purpose of setting $\delta$ is to prevent the verification from returning a meaningless lower bound such as $0$. If $l_{\mathrm{inf}}$ is less than $\delta$, we resort to computing safety bounds over finite time horizons. 

\vskip 1mm\noindent \textbf{Step 3: QVFTH}.  We try to synthesize two NBCs satisfying conditions in \Cref{thm:fin-lower-safe-linear,thm:fin-upper-safe-linear}, respectively. If the synthesis fails,  a timeout will be reported and the verification will terminate. Otherwise, we can compute  the linear lower and upper bounds on probabilistic safety over finite time horizons according to the synthesized NBCs. Alternatively, we can choose to synthesize two NBCs satisfying conditions in \Cref{thm:fin-lower-safe-exp,thm:fin-upper-safe-exp} to achieve exponential bounds, which might be tighter than linear ones.

%% file: barrier.tex
In this section, we reduce all three safety verification problems of DNN-controlled systems into a cohesive problem of defining corresponding BCs. 
We establish specific conditions for candidate BCs and provide formulas for computing lower and upper bounds for quantitative verification based on the defined BCs.

\subsection{Qualitative Safety Verification} \label{subsec:quali}

\begin{theorem}[Almost-Sure Safety]\label{thm:almost-safe}
Given an $M_\mu$ with an initial set $S_0$ and an unsafe set $S_u$, if there exists a barrier certificate $B:S\rightarrow \Rset$ such that for some constant $\lambda\in (0,1]$, the following conditions hold:
     \begin{align}
       & B(s)\le 0 \ &\text{for all }s\in S_0, \label{eq:thm-almost-initial} \\
       & B(s)>0 \ & \text{for all } s\in S_u,  \label{eq:thm-almost-unsafe} \\
       & B(f(s,\pi(s+\delta)))-B(s)+\lambda\cdot B(s)\le 0  & \text{for all } (s,\delta)\in S\times W,  \label[condition]{eq:thm-almost-decrease} 
    \end{align}
    then $M_\mu$ is almost-surely safe, i.e., $\forall s_0\in S_0. \  \omega\in\Omega_{s_0} \Longrightarrow \omega_t \not\in S_u \ \forall t\in\Nset.$
\end{theorem}

\noindent\emph{Intuition.} The BC in \Cref{thm:almost-safe} is similar to that in \Cref{def:discrete-BC} except \Cref{eq:thm-almost-decrease},
in which we consider all stochastic behaviors of the system from state perturbations. We prove the theorem by contradiction and the proof resembles that in \cite[Proposition~2]{prajna2007framework}.

\subsection{Quantitative Safety  Verification over Infinite Time Horizon}\label{subsec:quanti}

Below we present the state-dependent lower and upper bounds on probabilistic safety over infinite time horizons.

\begin{theorem}[Lower Bounds on Infinite-time Safety]\label{thm:inf-lower-safe}
   Given an $M_\mu$ with an initial set $S_0$ and an unsafe set $S_u$, if there exists a barrier certificate $B:S\rightarrow \Rset$ such that for some constant $\epsilon\in [0,1]$,
    the following conditions hold:
     \begin{align}
     & B(s)\ge 0 \ & \text{for all } s\in S,
       \label[condition]{eq:thm-inf-nonnegative}\\ 
    & B(s) \le \epsilon \ & \text{for all } s\in S_0,  \label[condition]{eq:thm-inf-lower-initial}\\
       & B(s)\ge 1 \ & \text{for all } s\in S_u,
       \label[condition]{eq:thm-inf-lower-unsafe}\\ 
       & \expv_{\delta\sim \mu}[B(f(s,\pi(s+\delta)))\mid s]-B(s)\le 0\  & \text{for all }s\in S \setminus S_u , 
       \label[condition]{eq:thm-inf-lower-decrease}
    \end{align}
then the safety probability over infinite time horizons is bounded from below by
    \begin{equation}
       \forall s_0\in S_0.\  \probm_{s_0}\left[\{\omega\in\Omega_{s_0}\mid \omega_t\not\in S_u \ \text{for all } t \in \Nset  \} \right]\ge 1-B(s_0).
        \label{eq:thm-inf-lowerbound}
    \end{equation}
\end{theorem}

\noindent\emph{Intuition.} A BC under conditions  in \Cref{thm:inf-lower-safe} is a non-negative real-valued function satisfying the supermartingale property, i.e.,  the expected value of the function remains non-increasing at every time step for all states not in $S_u$ (see \Cref{eq:thm-inf-lower-decrease}). We prove the theorem by Ville's Inequality~\cite{ville1939etude} and the proof resembles that in \cite[Theorem 15]{prajna2007framework}.

\begin{theorem}[Upper Bounds on Infinite-time Safety]\label{thm:inf-upper-safe}
Given an $M_\mu$  with an initial set $S_0$ and an unsafe set $S_u$, if there exists a barrier certificate $B:S\rightarrow \Rset$ such that for some constants $\gamma\in (0,1)$, $0\le \epsilon'<\epsilon\le 1$, the following conditions hold:
     \begin{align}
     & 0\le B(s)\le 1 \ & \text{for all } s\in S,
       \label[condition]{eq:thm-inf-upper-bound}\\ 
       & B(s)\ge \epsilon \ &\text{ for all } s\in S_0,\label[condition]{eq:thm-inf-upper-initial} \\
         & B(s)\le \epsilon' \ &\text{ for all } s\in S_u,
         \label[condition]{eq:thm-inf-upper-unsafe}\\
       & B(s)-\gamma\cdot\expv_{\delta\sim \mu}[B(f(s,\pi(s+\delta)))\mid s]\le 0\  & \text{for all }s\in S \setminus S_u, 
       \label[condition]{eq:thm-inf-upper-decrease}
    \end{align}
then the safety probability over infinite time horizons is bounded from above by
    \begin{equation}
       \forall s_0\in S_0.\  \probm_{s_0}\left[\{\omega\in\Omega_{s_0}\mid \omega_t\not\in S_u \ \text{for all } t\in \Nset  \} \right]\le 1-B(s_0).
        \label{eq:thm-inf-upperbound}
    \end{equation}
\end{theorem}

\noindent\emph{Intuition.} A BC under conditions in \Cref{thm:inf-upper-safe} is a bounded non-negative function satisfying the $\gamma$-scaled submartingale property~\cite{DBLP:conf/lics/UrabeHH17}, i.e., the expected value of $B$ is increasing at each time step for states not in $S_u$ (\Cref{eq:thm-inf-upper-decrease}). We prove the theorem by Optional Stopping Theorem~\cite{williams1991}, while former work is based on fix-point theory~\cite{TakisakaOUH18}.

\subsection{Quantitative Safety Verification over Finite Time Horizon}
\label{subsec:quanti2}

When the safety probability over infinite time horizons exhibits a decline,
it becomes advantageous to analyze the decreasing changes over finite time horizons. 
In the following, we present our theoretical results on finite-time safety verification, starting with  two results  related to  lower bounds.

\begin{theorem}[Linear Lower Bounds on Finite-time Safety]\label{thm:fin-lower-safe-linear}
 Given an $M_\mu$  with an initial set $S_0$ and an unsafe set $S_u$, if there exists a barrier certificate $B:S\to \Rset$ such that for some constants $\lambda > \epsilon  \ge 0$ and $c \ge 0$, the following conditions hold:
        \begin{align} 
        & B(s)\ge 0 \ & \text{ for all }s\in S, \label[condition]{eq:thm-fin-lower-linear-nonnegative} \\
       & B(s)\le \epsilon \ &\text{for all }s\in S_0, \label[condition]{eq:thm-fin-lower-linear-initial}\\
       & B(s)\ge \lambda \ & \text{for all } s\in S_u, \label[condition]{eq:thm-fin-lower-linear-unsafe}\\ 
       & \expv_{\delta\sim \mu}[B(f(s,\pi(s+\delta)))\mid s]-B(s)\le c\  & \text{for all }s\in S, \label[condition]{eq:thm-fin-lower-linear-decrease}
        \end{align}
 then the safety probability over a finite time horizon $T$ is bounded from below by
 \begin{equation}\label{eq:fin-lowerbound-linear}
     \forall s_0\in S_0.\    \probm_{s_0}\left[\{\omega\in\Omega_{s_0}\mid \omega_t\not\in S_u \ \text{for all } t\le T \} \right]\ge 1-(B(s_0)+cT)/\lambda. \notag 
    \end{equation}
\end{theorem}

\noindent\emph{Intuition.} 
A BC in~\Cref{thm:fin-lower-safe-linear} satisfies the c-martingale property \cite{SteinhardtT12}, i.e., the expected value of $B$ can increase at every time step as long as it is bounded by a constant $c$ (\Cref{eq:thm-fin-lower-linear-decrease}), which is less conservative than the supermartingle property (\Cref{eq:thm-inf-lower-decrease}), at the cost providing safety guarantees over finite time horizons. We prove the theorem by Ville's Inequality~\cite{ville1939etude} and the proof resembles that in \cite[Theorem 9]{AnandM0Z22}.

\begin{theorem}[Exponential Lower Bounds on Finite-time Safety]\label{thm:fin-lower-safe-exp}
 Given an $M_\mu$  if there exists a function $B:S\rightarrow \Rset$ such that for some constants $\alpha> 0, \beta\in\Rset$,  
 and $\gamma\in [0,1)$, the following conditions hold:
     \begin{align}
      & B(s)\ge 0 \ & \text{ for all }s\in S, \label[condition]{eq:thm-fin-lower-exp-nonnegative} \\
       & B(s)\le \gamma \ &\text{for all }s\in S_0,\label[condition]{eq:thm-fin-lower-exp-initial} \\
       & B(s) \ge 1 \ & \text{for all } s\in S_u,\label[condition]{eq:thm-fin-lower-exp-unsafe} \\ 
       & \alpha \expv_{\delta\sim \mu}[B(f(s,\pi(s+\delta)))\mid s]-B(s) \le \alpha \beta \  & \text{for all }s\in S\setminus S_u.\label[condition]{eq:thm-fin-lower-exp-moni}   
    \end{align}
then the safety probability over a finite time horizon $T$ is bounded from below by
   \begin{equation}\label{eq:fin-lower-expbound}
     \forall s_0\in S_0.\    \probm_{s_0}\left[\{ \omega\in\Omega_{s_0}\mid\omega_t\not\in S_u \ \text{for all } t\le T \} \right]\ge 1-\frac{\alpha\beta}{\alpha-1} + (\frac{\alpha\beta}{\alpha-1}-B(s_0))\cdot \alpha^{-T}.\notag 
    \end{equation}
\end{theorem}

\noindent\emph{Intuition.} A BC in~\Cref{thm:fin-lower-safe-exp} satisfies that its $\alpha$-scaled expectation can increase at most $\alpha\beta$ at every time step (Condition (\ref{eq:thm-fin-lower-exp-moni})).
We establish a new result in discrete-time DNN-controlled systems and prove it by the discrete version of Gronwall's Inequality~\cite{gronwall1919note}, which is inspired by former work~\cite{newframwork} in continuous-time dynamical systems.

Then we propose our two results of upper bounds on safety probabilities.

\begin{theorem}[Linear Upper Bounds on Finite-time Safety]\label{thm:fin-upper-safe-linear}
 Given an $M_\mu$  with an initial set $S_0$ and an unsafe set $S_u$, if there exists a barrier function $B:S\rightarrow \Rset$ such that for some constants $\beta\in (0,1), \beta<\alpha<1+\beta, c\ge 0$, the following conditions hold:
     \begin{align}
      & B(s)\ge 0 \ & \text{ for all }s\in S, \label[condition]{eq:thm-fin-upper-linear-nonnegative} \\
       & B(s)\le \beta \ &\text{for all }s\in S\setminus S_u,\label[condition]{eq:thm-fin-upper-exp-initial} \\
       & \alpha \le B(s) \le 1+\beta \ & \text{for all } s\in S_u,\label[condition]{eq:thm-fin-upper-linear-unsafe} \\ 
       & \expv_{\delta\sim \mu}[B(f(s,\pi(s+\delta)))\mid s]-B(s) \ge c \  & \text{for all }s\in S\setminus S_u.\label[condition]{eq:thm-fin-upper-linear-moni}   
    \end{align}
then the safety probability over a finite time horizon $T$ is bounded from above by
     \begin{equation}\label{eq:fin-upper-linearbound}
      \forall s_0\in S_0.\   \probm_{s_0}\left[\{\omega\in\Omega_{s_0}\mid \omega_t\not\in S_u \ \text{for all } t\le T \} \right]\le  1-B(s_0)-\frac{1}{2}c\cdot T+\beta.\notag 
    \end{equation}
\end{theorem}

\noindent\emph{Intuition.} A BC in~\Cref{thm:fin-upper-safe-linear} is non-negative and its value is bounded when states are in $S_u$ (\Cref{eq:thm-fin-upper-linear-unsafe}). Moreover,  \Cref{eq:thm-fin-upper-linear-moni} is the inverse of the c-martingale property in \Cref{thm:fin-lower-safe-linear}, i.e., the expected value of $B$ should increase at least $c$ at every time step. 

\begin{theorem}[Exponential Upper Bounds on Finite-Time Safety]\label{thm:fin-upper-safe-exp}
  Given an $M_\mu$  with an initial set $S_0$ and an unsafe set $S_u$, if there exists a barrier function $B:S\to \Rset$ such that for some constants $K'\le K<0$, $\epsilon>0$ and a non-empty interval $[a,b]$, the following conditions hold:
         \begin{align}
       & B(s)\ge 0 \ &\text{for all }s\in S\setminus S_u,\label[condition]{eq:thm-fin-upper-exp-nonnegative}\\
       & K'\le B(s) \le K \ & \text{for all } s\in S_u,\label[condition]{eq:thm-fin-upper-exp-unsafe}\\ 
       & \expv_{\delta\sim \mu}[B(f(s,\pi(s+\delta)))\mid s]-B(s)\le -\epsilon \  & \text{for all }s\in S\setminus S_u,  \label[condition]{eq:thm-fin-upper-exp-moni}\\
       &  a\le B(f(s,\pi(s+\delta)))-B(s) \le b \ & \text{for all } s\in S\setminus S_u \text{ and } \delta \in W\label[condition]{eq:thm-fin-upper-exp-diffbound}, 
    \end{align}
then the safety probability over a finite time horizon $T$ is bounded from above by
    \begin{equation}\label{eq:fin-upper-expbound}
       \forall s_0\in S_0.\  \probm_{s_0}\left[\{\omega\in\Omega_{s_0}\mid \omega_t\not\in S_u \ \text{for all } t\le T \} \right]\le exp(-\frac{2(\epsilon \cdot T- B(s_0))^2}{T\cdot (b-a)^2}).\notag 
    \end{equation}
\end{theorem}

\noindent\emph{Intuition.}  A BC under the \Cref{eq:thm-fin-upper-exp-nonnegative,eq:thm-fin-upper-exp-unsafe,eq:thm-fin-upper-exp-moni,eq:thm-fin-upper-exp-diffbound} is a difference-bounded ranking supermartingale~\cite{DBLP:conf/popl/ChatterjeeFNH16}.  \Cref{eq:thm-fin-upper-exp-moni} is the supermartingale difference condition, i.e., the expectation of $B$ should decrease at least $\epsilon$ at each time step, while  \Cref{eq:thm-fin-upper-exp-diffbound} implies that the update of $B$ should be bounded. We prove this theorem by Hoeffding's Inequality on Supermartingales~\cite{hoeffding1994probability} and the proof resembles that in the work \cite{DBLP:conf/popl/ChatterjeeFNH16}.

%% file: k-induction.tex
We now introduce $k$-inductive barrier certificates, capable of offering both qualitative and quantitative safety guarantees,  
while relaxing the strict conditions for safety through the utilization of the $k$-induction principle~\cite{DBLP:conf/sas/DonaldsonHKR11,DBLP:conf/sas/Brain0KS15}.
Prior to presenting our theoretical results, we first define the notion of $k$-inductive update functions as follows.

\begin{definition}[$k$-inductive Update Functions]
Given an $M_\mu=(S,S_0,A,\pi,f,R,\mu)$, a \emph{$k$-inductive update function} $g^k_{\pi,f}$ with respect to $\pi,f$ is defined recursively, i.e., 
$$g^k_{\pi,f}(s_t,\Delta_t^k)=
\begin{cases}
g_{\pi,f}(g^{k-1}_{\pi,f}(s_t,\Delta_t^{k-1}), \delta_{t+k-1})  & \text{ if } k>1 \\
f(s_t,\pi(s_t+\delta_t))& \text{ if } k=1 \\
s_t & \text{ if } k=0
\end{cases}
$$
where $\Delta_t^k=[ \delta_t,\delta_{t+1},\dots, \delta_{t+k-1} ]$ is a noise vector of length $k$ with each $\delta_t\sim\mu$, and $g_{\pi,f}(s_t,\delta_t):=f(s_t,\pi(s_t+\delta_t))$.
\end{definition}
Intuitively, $g_{\pi,f}^k$ computes the value of a state after $k$ steps given a $k$-dimensional noise vector $\Delta^k\in W^k\subseteq\Rset^{n\times k}$, where $W=\supp{\mu}$ is the support of $\mu$. To calculate the expectation w.r.t. $k$-dimensional noises, we denote by $\mu^k$ the product measure on $W^k$.

\subsection{$k$-Inductive Barrier Certificates for Qualitative Safety}\label{subsec:k-induc-almost}

\begin{theorem}[$k$-inductive Variant of Almost-Sure Safety]\label{thm:almost-safe-induc}
     Given an $M_\mu$  with an initial set $S_0$ and an unsafe set $S_u$, if there exists a $k$-inductive barrier certificate $B:S\rightarrow \Rset$ such that the following conditions hold:
     \begin{align}
       & \textstyle\bigwedge_{0\le i<k} B(g_{\pi,f}^{i}(s,\Delta^i))\le 0 \ & \forall (s,\Delta^i)\in S_0\times W^i, \label[condition]{eq:thm-almost-indc-initial} \\
       & B(s)>0 \ & \forall s\in S_u,  \label[condition]{eq:thm-almost-indc-unsafe} \\
       &\textstyle\bigwedge_{0\le i<k} (B(g_{\pi,f}^{i}(s,\Delta^i))\le 0)\Longrightarrow B(g_{\pi,f}^k(s,\Delta^k))\le 0  & \forall (s,\Delta^i)\in S\times W^i,  \label[condition]{eq:thm-almost-indc-decrease} 
    \end{align}
    then the system $M_\mu$ is almost-surely safe, i.e., $\forall s_0\in S_0.\  \omega\in\Omega_{s_0} \Longrightarrow \omega_t \not\in S_u \ \forall t\in\Nset$.
\end{theorem}

\emph{Intuition.} \Cref{eq:thm-almost-indc-initial} implies that the state sequences starting from the safe set will remain in the safe set for the next $k-1$ consecutive time steps, while  
\Cref{eq:thm-almost-indc-decrease} means that for any $k$ consecutive time steps, if the system is safe, then the system will still be safe at the $(k+1)$-th time step. We prove the theorem by contradiction.

Note that \Cref{eq:thm-almost-indc-decrease} contains an implication, in order to compute the $k$-inductive BC, we replace it with its sufficient condition:
\begin{equation}
    -B(g_{\pi,f}^k(s,\Delta^k))-\textstyle\sum_{0\le i<k} \tau_i\cdot (-B(g_{\pi,f}^i(s,\Delta^i)))\ge 0,\ \forall (s,\Delta^i)\in S\times W^i. \label{eq:thm-almost-indc-decrease-replace}
\end{equation}
If there exist $\tau_0,\dots,\tau_{k-1}\ge 0$ satisfying~\Cref{eq:thm-almost-indc-decrease-replace}, \Cref{eq:thm-almost-indc-decrease} is satisfied.

\subsection{$k$-Inductive Barrier Certificates for Quantitative Safety}\label{subsec:induc-quanti}

\begin{theorem}[$k$-inductive Lower Bounds on Infinite-time Safety]\label{thm:k-ind-lower}
   Given an $M_\mu$, if there exists a k-inductive barrier certificate $B:S\to \Rset$ 
  such that for some constants $k\in\Nset_{\ge 1}$, $\epsilon\in [0,1]$ and $c\ge 0$, the following conditions hold:
     \begin{align}
     & B(s) \ge 0 \ & \text{for all } s\in S \label[condition]{eq:thm-inf-lower-induc-nonnegative} \\
       & B(s)\le \epsilon \ &\text{for all }s\in S_0, \label{eq:thm-inf-lower-induc-initial} \\
       & B(s)\ge 1 \ & \text{for all } s\in S_u, \label[condition]{eq:thm-inf-lower-induc-unsafe} \\ 
       & \expv_{\delta\sim\mu}[B(f(s,\pi(s+\delta)))\mid s]-B(s)\le c \  & \text{for all } s\in  S , \label[condition]{eq:thm-inf-lower-induc-c-martin} \\
       & \expv_{\Delta^k\sim \mu^k}[B(f_{\pi,f}^k(s,\Delta^k))\mid s]-B(s)\le 0 \  & \text{for all } s\in  S, \label[condition]{eq:thm-inf-lower-induc-k-martin}
    \end{align}  
then the safety probability over infinite time horizons is bounded from below by
   \begin{equation}\label{eq:inf-lower-induc-bound}
    \forall s_0\in S_0.\     \probm_{s_0}\left[\{\omega_0\in\Omega_{s_0}\mid \omega_t\not\in S_u \ \text{ for all } t\in \Nset  \} \right]\ge 1-kB(s_0)-\frac{k(k-1)c}{2}.\notag 
    \end{equation}
\end{theorem}

\noindent\emph{Intuition.} \Cref{eq:thm-inf-lower-induc-c-martin} requires the barrier certificate to be a c-martingale at every time step and \Cref{eq:thm-inf-lower-induc-k-martin} requires the barrier certificate sampled after every $k$-th step to be a supermartingale. We prove the theorem by Ville's Inequality~\cite{ville1939etude}.

\begin{theorem}[$k$-inductive Upper Bounds on Infinite-time Safety]\label{thm:k-induc-upper}
Given an $M_\mu$, if there exists a barrier certificate $B:S\rightarrow \Rset$ such that for some constant $\gamma\in (0,1)$, $0\le \epsilon'<\epsilon\le 1$, $c\le 0$ the following conditions hold:
     \begin{align}
     & 0\le B(s)\le 1 \ & \text{for all } s\in S
       \label[condition]{eq:thm-inf-upper-induc-bound}\\ 
       & B(s)\ge \epsilon \ & \text{ for all } s\in S_0, \label[condition]{eq:thm-inf-upper-induc-initial}\\
         & B(s)\le \epsilon' \ & \text{ for all } s\in S_u, \label[condition]{eq:thm-inf-upper-induc-unsafe}\\
        &  \expv_{\delta\sim \mu}[B(f(s,\pi(s+\delta)))\mid s]-B(s)\ge c \ 
 & \text{ for all } s\in S, \label[condition]{eq:thm-inf-upper-induc-c-martin} \\
       & B(s)-\gamma^k \cdot\expv_{\Delta^k\sim \mu^k}[B(g_{\pi,f}^k(s,\Delta^k))\mid s]\le 0\  & \text{for all }s\in S \setminus S_u, 
       \label[condition]{eq:thm-inf-upper-induc-decrease}
    \end{align}
then the safety probability over infinite time horizons is bounded from above by
  \begin{equation}\label{eq:inf-upper-induc-bound}
      \forall s_0\in S_0.\   \probm_{s_0}\left[\{\omega\in\Omega_{s_0}\mid \omega_t\not\in X_u \ \text{ for all } t\in \Nset  \} \right]\le 1-k B(s_0)-\frac{k(k-1)c}{2}.\notag 
    \end{equation}
\end{theorem}

\noindent\emph{Intuition.} This BC is non-negative and bounded (Condition \ref{eq:thm-inf-upper-induc-bound}). Condition (\ref{eq:thm-inf-upper-induc-c-martin}) is the inverse of the c-martingale property, while Condition (\ref{eq:thm-inf-upper-induc-decrease}) requires the barrier certificate sampled after every $k$-th step to be a $\gamma^k$-scaled submartingale. We prove the theorem by the Optional Stopping Theorem~\cite{williams1991}.

\begin{remark}
    To make the probabilistic bounds in~\Cref{thm:k-ind-lower} and \Cref{thm:k-induc-upper} 
    non-trivial, the value of $k$ should be bounded by
    \begin{equation*}
        1\le k \le \frac{(c-2B(s_0))+\sqrt{4B(s_0)^2+c^2-4c(B(s_0)-2)}}{2c} .
    \end{equation*}
    
\end{remark}

%% file: neural.tex
In this section, we show that the BCs defined in the previous sections for DNN-controlled systems can be implemented and synthesized in the form of DNNs, akin to those for linear or nonlinear stochastic systems~\cite{ZHAO_NBC}.

We adopt the CEGIS-based method ~\cite{AbateDKKP18}  to train and validate target NBCs. \Cref{fig:flowch} sketches the workflow. In each loop iteration, we train a candidate BC in the form of a neural network which is then passed to the validation. If the validation result is false, we compute a set of counterexamples for future training. This iteration is repeated until a trained candidate is validated or a given timeout is reached. Moreover, we propose a simulation-guided training method by adding additional terms to the loss functions to improve the tightness of upper and lower bounds calculated by the trained NBCs.

We present the synthesis of NBCs in ~\Cref{thm:inf-lower-safe} for probabilistic safety over infinite time horizons, as an example. We defer to Appendix the synthesis of other NBCs.

\begin{wrapfigure}{r}{0.46\textwidth}
	\usetikzlibrary {shapes.geometric}
	\vspace{-1mm}
	\hspace{-2mm}
	\begin{tikzpicture}[ipe stylesheet]
			\footnotesize
		\node (A) at (0, 120) [draw, rounded rectangle, minimum width=1.5cm] {Start};
		\node (B) at (0, 95) [draw,minimum width=2.2cm, rectangle] {Discretize state space};
		\node (C) at (0, 65) [draw, minimum width=2.2cm, rectangle] {\textbf{Train}};		
		\node (D) at (0, 30) [draw, minimum width=2.2cm, rectangle] {\textbf{Validate}};
		\node (E) at (0, -5) [shape aspect=2,scale=0.8,diamond,draw] {Valid?};
		\node (EE) at (90, -5) [shape aspect=2,scale=0.8,diamond,draw] {Timeout?};
		\node (F) at (90, 30) [draw, rectangle,minimum width=2.2cm] {\textbf{Refine}};
		\node (G) at (0, -40) [draw, rounded rectangle,minimum width=1.5cm] {Finish};
		\node (H) at (90, -40) [draw, rounded rectangle,minimum width=1.5cm] {Fail};		
		\draw[-{Stealth[length=2mm]}] (A) -- (B);
		\draw[-{Stealth[length=2mm]}] (B) -- (C);
		\draw[-{Stealth[length=2mm]}] (C) -- node [right] {NBC}  (D);		
		\draw[-{Stealth[length=2mm]}] (D) -- (E);
		\draw[-{Stealth[length=2mm]}] (E) -- node [near start,below]{No} (EE);
		
		\draw[-{Stealth[length=2mm]}] (EE) -- node [right]{No} (F);
		\draw[-{Stealth[length=2mm]}] (E) -- node [right] {Yes} (G);			
		\draw[-{Stealth[length=2mm]},postaction={decorate,decoration={raise=1ex,text along path,reverse path,text align=center,text={Refined state space}}}] (F) |- node [left] {} (C);			
		\draw[-{Stealth[length=2mm]}] (EE) -- node [right] {Yes} (H);			
	\end{tikzpicture}
	
\vspace{-1mm}
	\caption{CEGIS-based NBC synthesis ~\cite{AbateDKKP18}.}
	\label{fig:flowch}
	\vspace{-10mm}
\end{wrapfigure}
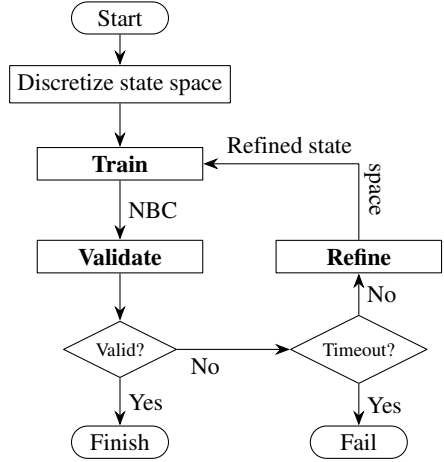

\subsection{Training Candidate NBCs}
\label{subsec:training}

Two pivotal factors in the training phase are the generation of training data and the construction of the loss function.

\inlsec{Training Data Discretization} As the state space $S$ is possibly continuous and infinite, we choose a finite set of states for training candidate NBCs.
This can be achieved by discretizing the state space $S$ and constructing a \emph{discretization} $\tilde{S}\subseteq S$ such that for each $s \in S$, there is a $\tilde{s} \in \tilde{S}$ with $||s-\tilde{s}||_1 < \tau $,
where $ \tau>0$ is called the granularity of $\tilde{S}$. As $S$ is compact and thus bounded, this discretization can be computed by simply picking the vertices of a grid with sufficiently small cells. 
For the re-training after validation failure, $\tilde{S}$ will be reconstructed with counterexamples and a smaller granularity $\tau$.
Once the discretization $\tilde{S}$ is obtained, we construct two finite sets $\tilde{S}_0:=\tilde{S}\cap S_0$ and   $\tilde{S}_u:=\tilde{S}\cap S_u$ 
used for the training process.

\inlsec{Loss Function Construction}
A candidate NBC is initialized as a neural network $h_\theta$ w.r.t. the network parameter $\theta$. $h_\theta$ is trained by minimizing the following loss function:
\begin{equation}
\calL(\theta):=k_1 \cdot \calL_{1}(\theta)+k_2 \cdot \calL_{2}(\theta)+k_3 \cdot \calL_{3}(\theta)+k_4 \cdot \calL_{4}(\theta)+k_5 \cdot \calL_{5}(\theta)\notag 
\end{equation}
where $k_i\in\Rset, \ i=1,\cdots,5$ are the algorithmic parameters balancing the loss terms.

The first loss term is defined via the condition in~\Cref{eq:thm-inf-nonnegative} as:

\begin{equation}
\calL_{1}(\theta)=\frac{1}{|\tilde{S}|} \textstyle\sum\limits_{s\in \tilde{S}} \left(\mathop{\mathrm{max}}\{ 0 - h_\theta(s),0\} \right)\notag 
\end{equation}
Intuitively, a loss will incur if either $h_\theta(s)$ is less than zero for any $s\in \tilde{S}$.

Correspondingly, the second and third loss terms are  defined via \Cref{eq:thm-inf-lower-initial,eq:thm-inf-lower-unsafe} as:
\begin{equation}
\begin{split}
\calL_{2}(\theta)=\frac{1}{|\tilde{S}_0|} \textstyle\sum\limits_{s\in \tilde{S}_0} \left(\mathop{\mathrm{max}}\{  h_\theta(s) - \epsilon ,0\} \right), \mbox{and } \calL_{3}(\theta)=\frac{1}{|\tilde{S}_u|} \textstyle\sum\limits_{s\in \tilde{S}_u} \left(\mathop{\mathrm{max}}\{ 1 - h_\theta(s),0\} \right).\notag 
\end{split}
\end{equation}
The fourth   loss term is defined via the condition in~\Cref{eq:thm-inf-lower-decrease} as:
\begin{equation}
\calL_{4}(\theta)=\frac{1}{|\tilde{S} \setminus \tilde{S}_u|}\textstyle\sum\limits_{s\in \tilde{S} \setminus \tilde{S}_u} \left(  \mathop{\mathrm{max}}\{ \textstyle\sum\limits_{s'\in \mathcal{D}_s}\frac{h_\theta(s')}{N} -h_\theta(s)+\zeta ,0\}  \right)\notag 
\end{equation}
where for each $s\in \tilde{S} \setminus \tilde{S}_u$, $\mathcal{D}_s$ is the set of its successor states such that $\mathcal{D}_s:=\{s'\mid s'=f(s,\pi(s+\delta_i)), \delta_i\sim \mu, i\in [1,N]  \}$,
$N>0$ is the sample number of successor states. We use the mean of $h_\theta(\cdot)$ at the $N$ successor states to approximate the expected value $\expv_{\delta\sim \mu}[B(f(s,\pi(s+\delta)))]$ for each $s\in \tilde{S} \setminus \tilde{S}_u$, 
and $\zeta>0$  to tighten the condition.

\inlsec{Simulation-guided Loss Term}
A trained BC that satisfies the above four conditions can provide lower bounds on probabilistic
safety over infinite time horizons for the system. However, these conditions have nothing to do with the tightness of lower bounds and we may obtain a trivial zero-valued lower bound by the trained BC.

To assure the tightness of lower bounds from trained NBCs, we propose a simulation-guided method based on~\Cref{eq:thm-inf-lowerbound}. 
For each $s_0 \in \tilde{S}_0$, we execute the control system $N'>0$ episodes, and calculate the safety frequency $\varmathbb{f}_s$ of all the $N'$ trajectories over infinite time horizons. Based on the statistical results,   the last loss term is defined as:  
\begin{equation}
\calL_{5}(\theta)=\frac{1}{|\tilde{S}_0|}\textstyle\sum\limits_{s\in \tilde{S}_0} \left( \mathop{\mathrm{max}}\{\varmathbb{f}_s + h_\theta(s) - 1,0\}\right)\notag 
\end{equation}
Intuitively, this term is to enforce the value of the derived lower bound to approach the statistical result as closely as possible, ensuring its tightness. 

We emphasize that our simulation-guided method plus the NBC validation (see next section) is sound, as we will validate the trained BC  to ensure it satisfies all the BC conditions (see also \Cref{thm:soundness}).

\subsection{NBC Validation}

A candidate NBC $h_\theta$ is valid if it meets the~\Cref{eq:thm-inf-nonnegative,eq:thm-inf-lower-initial,eq:thm-inf-lower-unsafe,eq:thm-inf-lower-decrease}. 
The first three conditions condition can be checked by the following constraint 
\begin{equation}
 \mathop{\mathrm{inf}}\limits_{s\in S} h_\theta(s)\ge 0 \   \land \ \mathop{\mathrm{sup}}\limits_{s\in S_0} h_\theta(s)\le \epsilon \  \land \ \mathop{\mathrm{inf}}\limits_{s\in S_u} h_\theta(s)\ge 1 \notag 
\end{equation}
using the interval bound propagation approach  ~\cite{Interval_Bound_Propagation,interval_pr}. 
When any state violates the above equation, it is treated as a counterexample and added to $\tilde{S}$ for future training. 

For \Cref{eq:thm-inf-lower-decrease}, 
~\Cref{theorem:verify_NBC} reduces the validation from infinite states to finite ones, which are easier to check.  

\begin{theorem}
\label{theorem:verify_NBC}
Given an $M_{\mu}$ and a function $B:S\rightarrow\Rset$, 
we have  $ \expv_{\delta\sim \mu}[B(f(s,\pi(s+\delta)))\mid s]-B(s)\le 0$ 
for any state $s\in S \setminus S_u$ if the formula below 
\begin{align}
\expv_{\delta\sim \mu}[B(f(\tilde{s},\pi(\tilde{s}+\delta)))\mid \tilde{s}]\le B(\tilde{s})-\zeta
\label{equ:validation_NBC}
\end{align} holds 
for any state $\tilde{s}\in \tilde{S} \setminus \tilde{S}_u$, where $\zeta=  \tau\cdot L_B\cdot (1+L_f\cdot (1+L_\pi))$ with  $L_f,L_\pi, L_B$ being the Lipschitz constants of $f,\pi$ and $B$,  respectively.
\end{theorem}

To check the satisfiablility of \Cref{equ:validation_NBC} in $h_\theta$ and a state $\tilde{s}$, 
we need to compute the expected value  $\expv_{\delta\sim \mu}[h_\theta(f(\tilde{s},\pi(\tilde{s}+\delta)))\mid \tilde{s}]$. 
However, it is difficult to compute its closed form because $h_\theta$ is provided in the form of neural networks.  
Hence, We bound the  expected value $\expv_{\delta\sim \mu}[h_\theta(f(\tilde{s},\pi(\tilde{s}+\delta)))\mid \tilde{s}]$ 
via interval arithmetic \cite{Interval_Bound_Propagation,interval_pr} instead of computing it, which is inspired by the work~\cite{stab_martingales,reach_avoid_martingale}.
In particular, given the noise distribution $\mu$ and its support $W =\{\delta \in \Rset^n \ | \ \mu(\delta) > 0 \}  $, we first partition $W$ into finitely $m\ge 1$
cells, i.e., $\mathrm{cell(W)} = \{W_1,\cdots,W_m\}$, 
and use  $\mathrm{maxvol} = \mathop{\mathrm{max}}\limits_{W_i \in \mathrm{cell(W)}} \mathrm{vol}(W_i)$  to denote the maximal volume with respect to the Lebesgue measure of any cell in the partition, respectively. For the expected value in \Cref{equ:validation_NBC}, we bound it from above: 
\begin{align}\label{eq:overes-up} 
\expv_{\delta\sim \mu}[h_\theta(f(\tilde{s},\pi(\tilde{s}+\delta)))\mid \tilde{s}]\le \textstyle\sum\limits_{W_i\in \mathrm{cell(W)}}  \mathop{\mathrm{maxvol}} \cdot \mathop{\mathrm{sup}}_{\delta}F(\delta) \notag 
\end{align}
where, $F(\delta)=h_\theta(f(\tilde{s},\pi(\tilde{s}+\delta)))$.   
The supremum can be  calculated via interval arithmetic. We refer interested readers to 
\cite{stab_martingales} and \cite{reach_avoid_martingale}  for more details.

\vspace{-1mm}
\begin{theorem}[Soundness]\label{thm:soundness}
    If a trained NBC is valid, it can certify the almost-sure safety for the qualitative verification, or the derived bound by the NBC is a certified lower/upper bound on the safety probability for the quantitative case.
\end{theorem}

The proof of soundness is straightforward by the NBC validation.

%% file: experiment.tex
Our experimental goals encompass evaluating the effectiveness of (i) the qualitative and quantitative verification methods within our framework, (ii) the $k$-inductive BCs, and (iii) the simulation-guided training method, respectively.

\subsection{Benchmarks and Experimental Setup}

We assess the effectiveness of our approach on four classic DNN-controlled tasks from public benchmarks:
Pendulum and Cartpole from the DRL training platform OpenAI's Gym~\cite{GYM}, while B1 and Tora 
commonly used by the state-of-the-art safety verification tools~\cite{ivanov2021verisig}. 
All experiments are executed on a workstation running Ubuntu 18.04, with a 32-core AMD Ryzen Threadripper CPU, 128GB RAM, and a single 24564MiB GPU .

For the safety verification of DNN-controlled systems, we consider state perturbations of uniform noises with zero means and different radii. 
Specifically, for each  state $s = (s_1,\ldots,s_n)$, we add noises $X_1,\ldots,X_n$ to each dimension of $s$ and obtain the perturbed state $(s_1 + X_1,\ldots,s_n + X_n)$, where $X_i \sim \mathbf{U}(-r, r)$ ($1\le i \le n$, $r\ge 0$).

	\begin{wraptable}{r}{0.44\textwidth}
		\vspace{-8mm}
		\centering
		\footnotesize
		\caption{Qualitative verification results.}	
		\setlength{\tabcolsep}{1pt}
		\begin{tabular}{l|r c r r r}
			\hline
			\textbf{Task} & \centering \textbf{Perturbation} & \centering \textbf{Verification} & $k$.   & \textbf{\#Fail.} \\ 
			\hline
			\multirow{4}{*}{CP} & 0  & $\checkmark$ & 1  &  0  \\
			& $r=0.01$  & \textbf{Unknown} & 1  & 0\\
			& $r=0.01$  & $\checkmark$ & 2  & 0\\    
			& $r=0.03$  & \textbf{Unknown} & 1   & 207 \\
			
			\hline
			\multirow{3}{*}{PD} &  $r=0$   & $\checkmark$ & 1   &  0  \\
			&  $r=0.01$  & \textbf{Unknown} & 1   & 675 \\
			&  $r=0.03$ & \textbf{Unknown} & 1   & 720\\
			\hline
			\multirow{4}{*}{Tora} & $r=0$   & $\checkmark$ & 1  &  0  \\
			& $r=0.02$  & \textbf{Unknown} & 1   & 0\\
			& $r=0.02$  & $\checkmark$  & 2   & 0\\
			& $r=0.04$  &  \textbf{Unknown} & 1   & 1113 \\
			\hline
			\multirow{3}{*}{B1} & $r=0$  & $\checkmark$ & 1  &  0  \\
			&  $r=0.1$ &  $\checkmark$ & 1 & 0 \\
			&  $r=0.2$ & \textbf{Unknown} & 1   &43 \\
			\hline
		\end{tabular}
		\vspace{-7mm}
		\label{table:quantitative_verification}
	\end{wraptable}

For qualitative evaluations, the existence of an NBC in~\Cref{thm:almost-safe} can ensure the almost-sure safety of the whole system. 
Due to the data sparsity of an initial state, we randomly choose 10,000 initial states (instead of a single one) from the initial set $S_0$.
For quantitative evaluations, to measure the quantitative safety probabilities from the system level, we calculate the mean values of lower/upper bounds by NBCs on these 10,000 states under different perturbations. The correctness of such system-level safety bounds is witnessed by~\Cref{thm:soundness} as each lower/upper bound on a single state $s_0$ is a certified bound for the exact safety probability from $s_0$, and thus the same holds on the system level. We also simulate 10,000 episodes starting from each of these 10,000 initial states under different perturbations and use the statistical results as the baseline.

\subsection{Effectiveness of Qualitative Safety Verification}
 
	\Cref{table:quantitative_verification} shows the  qualitative verification results under different perturbation radii $r$'s and induction bounds $k$'s. 
	Given a perturbed DNN-controlled system, we verify its qualitative safety by training an NBC under the conditions in \Cref{thm:almost-safe}. Once such an NBC is trained and validated, the system is verified to be almost-surely safe, marked as $\checkmark$. 
	If no valid barrier certificates are trained  within a given timeout, the 
	result is marked as \textbf{Unknown}.

	\begin{figure}[!ht] 
  \centering
		\addtolength{\subfigcapskip}{2pt}
		\setlength{\tabcolsep}{1pt}
		\hspace{-2mm}		
		\begin{tabular}{cc}
			\subfigure[CP]{
				\includegraphics[width=0.42\textwidth]{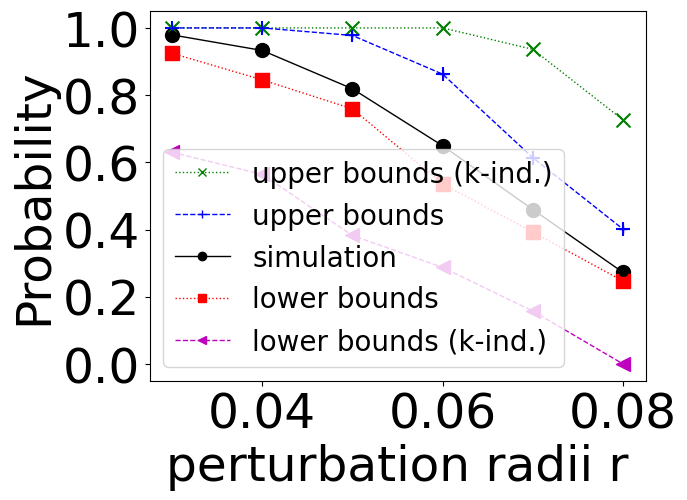}
				
			} \hspace{7mm}
			\subfigure[PD]{
				\includegraphics[width=0.4\textwidth]{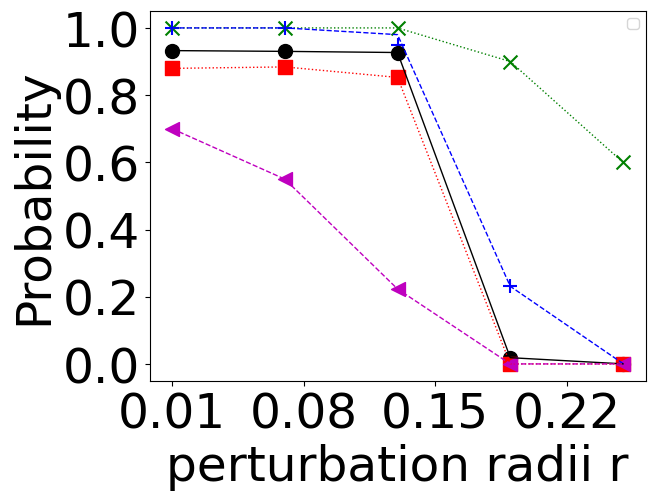}
			}
			\\
			\subfigure[Tora]{
				\includegraphics[width=0.4\textwidth]{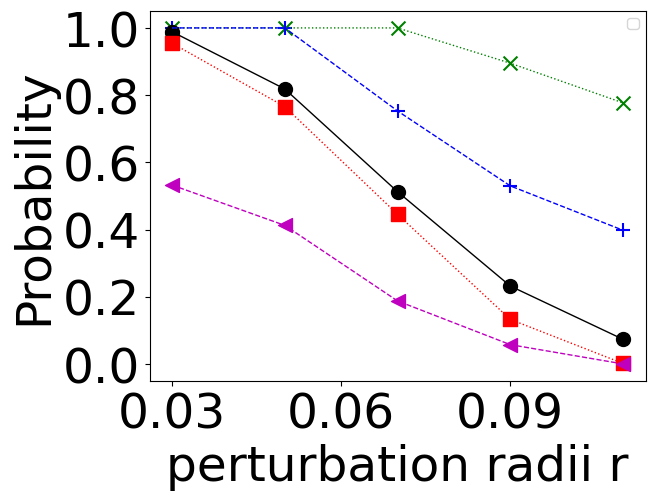}
			}
			\hspace{7mm}
			\subfigure[B1]{
				\includegraphics[width=0.42\textwidth]{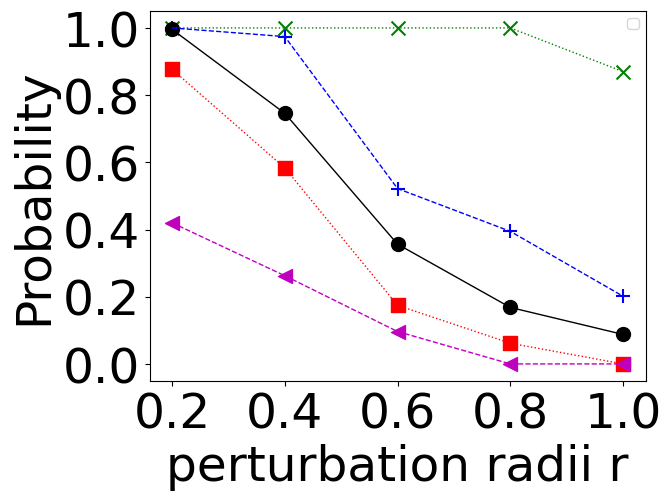}
			}\vspace{-3mm}\\
   
			\subfigure[CP ($r=0.3$)]{
				\includegraphics[width=0.4\textwidth]{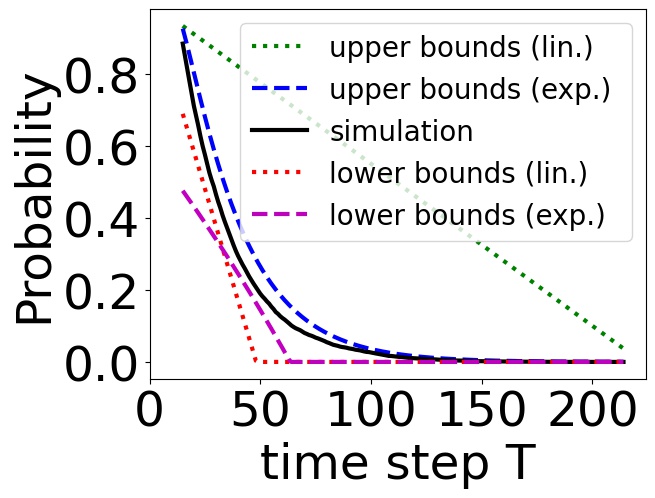}
				
			}\hspace{7mm}
			\subfigure[PD ($r=1.0$)]{
				\includegraphics[width=0.42\textwidth]{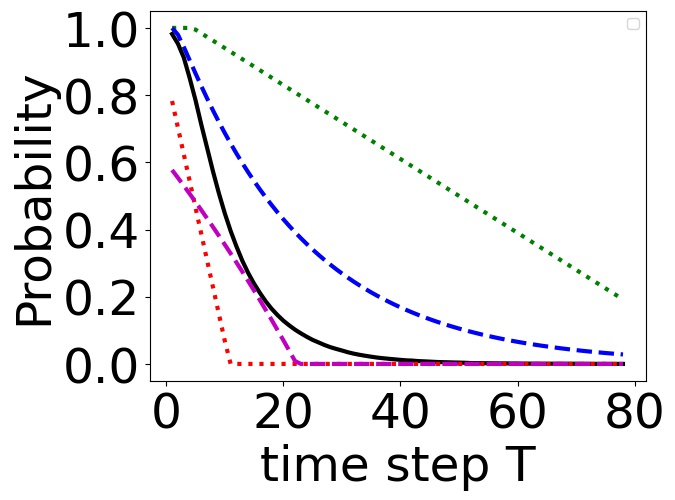}
			}\\
   
			\subfigure[Tora ($r=0.3$)]{
				\includegraphics[width=0.4\textwidth]{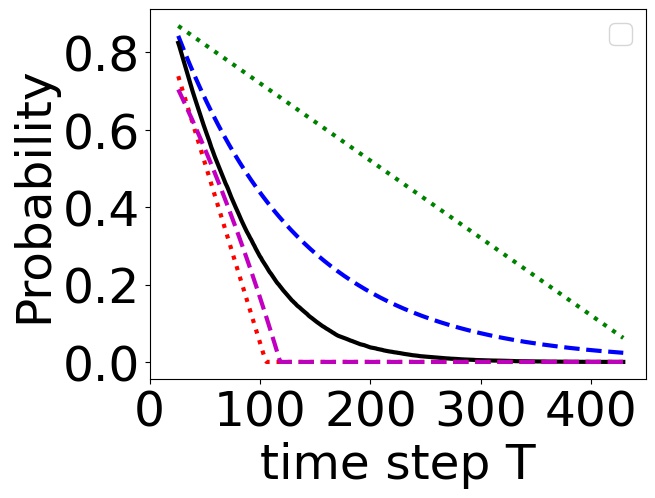}
				
			}\hspace{7mm}
			\subfigure[B1 ($r=1.5$)]{
				\includegraphics[width=0.4\textwidth]{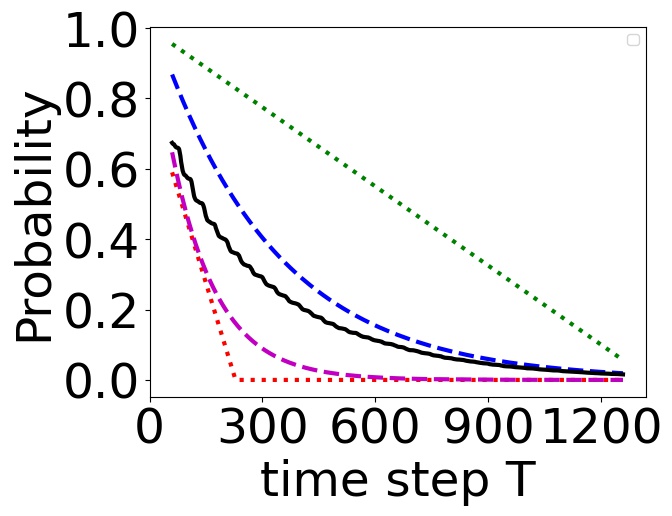}
				
			}
		\end{tabular}	
		\vspace{-5mm}
		\caption{The certified upper and lower bounds over infinite (a-d) and finite (e-h) time horizons, respectively, and their comparison with the simulation results.}
		\label{fig:infi_bound}
		\vspace{-6mm}
	\end{figure}

 As for simulation, 	we record the number of those episodes where the system enters the unsafe region, marked as the column \textbf{\#Fail.} in the table. We can observe that for the systems that are successfully verified by NBCs, no failed episodes are detected by simulation. 
For systems with failed episodes by simulation, no corresponding NBCs can be trained and validated. The consistency experimentally reflects the effectiveness of our approach.

	Furthermore, we note that for CP with $r=0.01$ and Tora with $r =0.02 $, there are no failed episodes, 
	but  no  NBCs in \Cref{thm:almost-safe} can be synthesized for these systems. By applying~\Cref{thm:almost-safe-induc},  we find the $2$-inductive NBCs, which ensures the safety of the systems. 
	It demonstrates that $k$-inductive variants can relax the conditions of NBCs and thus ease the synthesis of valid NBCs for qualitative safety verification.
	
	As the perturbation radius increases, ensuring almost-sure safety becomes challenging, and qualitative verification only results in the conclusion of \textbf{Unknown}. Consequently, we proceed to conduct quantitative verification over infinite time horizons.
	
	\subsection{Effectiveness of Quantitative Safety Verification over Infinite Time Horizon}

	\Cref{fig:infi_bound} (a-d)  show the certified upper and lower bounds and simulation results (i.e., black lines marked with `$\bullet$') over
	infinite time horizons. 
	The red lines marked with `\textcolor{red}{$\blacksquare$}' and blue lines marked with `\textcolor{blue}{+}' represent the mean values of the lower bounds in \Cref{thm:inf-lower-safe} and the upper bounds in \Cref{thm:inf-upper-safe} on the chosen 10,000 initial states calculated by the corresponding NBCs, respectively.
	The purple lines marked with `\textcolor{purple}{$\blacktriangle$}' and green lines marked with `\textcolor{teal}{$\times$}' represent the mean values of the $2$-inductive upper and lower bounds calculated by the corresponding NBCs in \Cref{thm:k-induc-upper,thm:k-ind-lower}, respectively.
	We can find that the certified 
	bounds enclose the simulation outcomes,  demonstrating the effectiveness of our trained NBCs.
	
	\begin{wraptable}{r}{0.45\textwidth}	
		\centering
		\footnotesize
		\setlength{\tabcolsep}{2pt}
		\caption{Synthesis time for different NBCs.}
		\begin{tabular}{l|r r r r}
			\hline
			\textbf{Task} & \textbf{Lower} & \textbf{$2$-Lower} & \textbf{Upper} & \textbf{$2$-Upper}   \\
			\hline
			CP & 2318.5  &  1876.0 & 2891.9  & 2275.3    \\
			PD & 1941.6  &  1524.0 & 2282.7  & 1491.5     \\
			Tora & 280.3  &  218.5 & 895.1  &  650.7   \\
			B1 & 587.4  &  313.6 & 1127.3  &  840.1 \\
			\hline
		\end{tabular}
		\vspace{-7mm}
		\label{table:Synthesis_time}
	\end{wraptable}
	\Cref{table:Synthesis_time} shows
	a comparison of average synthesis time (in seconds) for different NBCs. 
	We observe that the synthesis time of $2$-inductive NBCs is 25\% faster than that of normal NBCs, at the sacrifice of tightness.
	Note that the tightness of certified bounds depends on specific systems and perturbations. Investigating what factors influence the tightness to yield tighter bounds is an interesting future work to explore.

	Approaching zero for infinite time horizons, the lower bounds indicate a declining trend in safety probabilities over time.
	Therefore, we proceed to conduct quantitative verification over finite time horizons, providing both linear and exponential lower and upper bounds.

	\subsection{Effectiveness of Quantitative Safety Verification over Finite Time Horizon}
	\vspace{-1mm}

	\Cref{fig:infi_bound} (e-h) depict the certified upper and lower bounds and simulation results (i.e., black lines) over finite time horizons from the system level. Fix a sufficiently large noise level for each system, the x-axis represents the time horizon, while the y-axis corresponds to the safety probabilities.
	The purple lines  and blue lines  represent the mean values of
	the exponential lower and upper bounds calculated by the corresponding NBCs in ~\Cref{thm:fin-lower-safe-exp} and \Cref{thm:fin-upper-safe-exp}, 
	respectively. 
	The red lines and green lines represent the mean values of
	the linear lower and upper bounds calculated by the corresponding NBCs in ~\Cref{thm:fin-lower-safe-linear} and \Cref{thm:fin-upper-safe-linear}, respectively. 
	The results indicate that our computed certified bounds encapsulate the statistical outcomes.
 Moreover, the exponential upper bounds are always tighter than the linear upper bounds, and the exponential lower bounds become tighter than the linear ones with the increase of time. It is worth exploring the factors to generate tighter results in future work.

	\subsection{Effectiveness of Simulation-guided Loss Term}
	\vspace{-1mm}

\begin{figure}[t] 
		\hspace{-2mm}	
		\addtolength{\subfigcapskip}{2pt}
		\setlength{\tabcolsep}{1pt}
		\begin{tabular}{llll}
			\subfigure[CP]{
				\includegraphics[width=0.48\textwidth]{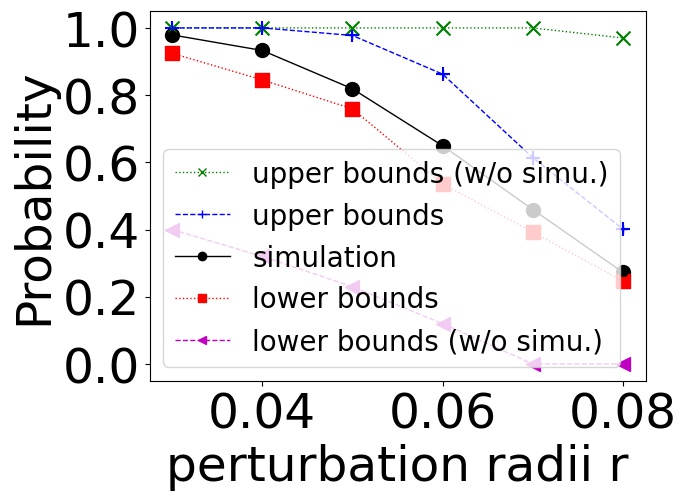}
			}&
			\subfigure[PD]{
				\includegraphics[width=0.46\textwidth]{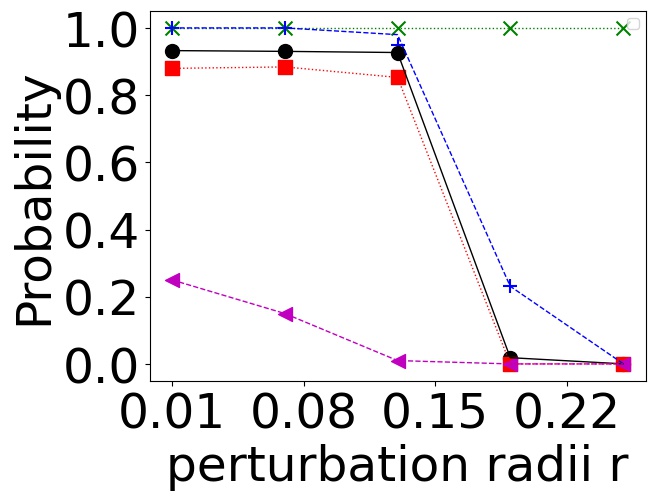}
			}\\
			\subfigure[Tora]{
				\includegraphics[width=0.46\textwidth]{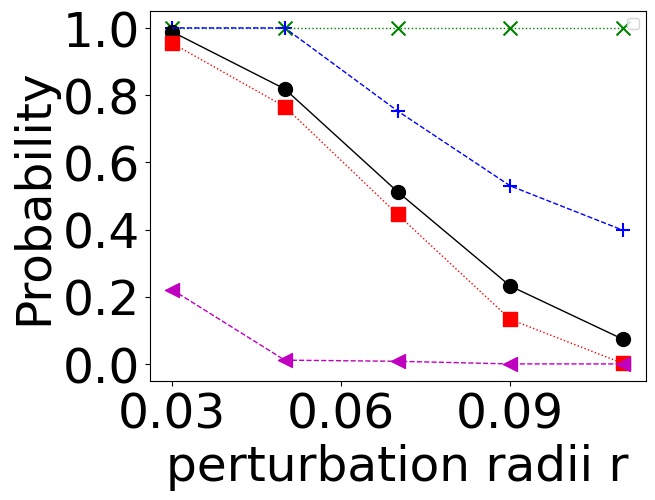}
			}&
			\subfigure[B1]{
				\includegraphics[width=0.48\textwidth]{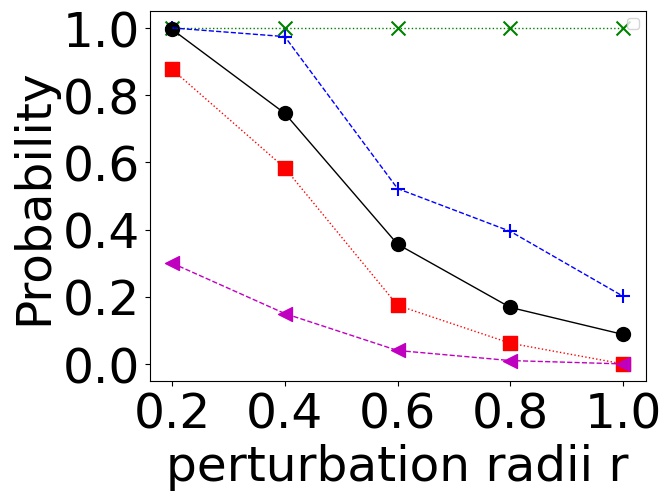}
			}
		\end{tabular}
		\vspace{-5mm}
		\caption{The certified bounds w/ and w/o   simulation-guided loss terms over
			infinite time horizons.}
		\label{fig:infi_bound_compare}
		\vspace{-6mm}
	\end{figure}
 
	The simulation-guided loss term is proposed in \Cref{subsec:training} to tighten the certified bounds calculated by NBCs.
	To evaluate its effectiveness, we choose NBCs in \Cref{thm:inf-lower-safe,thm:inf-upper-safe}, and train them with and without the simulation-guided loss terms. The comparison between them is shown in \Cref{fig:infi_bound_compare}. 
The red lines marked with `\textcolor{red}{$\blacksquare$}' and blue lines marked with `\textcolor{blue}{+}' represent the mean values of the bounds in \Cref{thm:inf-lower-safe,thm:inf-upper-safe} on initial states calculated by the corresponding NBCs trained with the simulation-guided loss terms, respectively. The purple lines with `\textcolor{purple}{$\blacktriangle$}' and green lines with  `\textcolor{teal}{$\times$}' represent the mean values of the bounds calculated by the NBCs trained without the simulation-guided loss terms.
	Apparently, the upper and lower bounds derived by NBCs trained without the simulation-guided loss terms are looser than the bounds trained with these terms. Specifically, 
	the results computed by NBCs with simulation-guided loss terms can achieve an average improvement of 47.5\% for lower bounds and 31.7\% for upper bounds, respectively. 
 Hence, it is fair to conclude that accounting for simulation-guided loss terms is essential when conducting quantitative safety verification.

%% file: related.tex
\vspace{-1mm}

\inlsec{Barrier Certificates for Stochastic Systems} Our unified safety verification framework draws inspiration from research on the formal verification of stochastic systems employing barrier certificates.
Prajna \textit{et al.}~\cite{PrajnaJ04,prajna2007framework,prajna2005necessity} propose the use of barrier certificates in the safety verification of stochastic systems.
This idea has been further expanded through data-driven approaches~\cite{SalamatiLSZ21} and $k$-inductive variants~\cite{AnandM0Z22}. 
As the dual problem of computing safety probabilities, computing reachability probabilities in stochastic dynamical systems has been studied for both infinite~\cite{DBLP:conf/cav/FengC00Z20,DBLP:conf/amcc/0001LZF21,xue2023reach} and finite time horizons~\cite{DBLP:journals/tac/0001FZ20,newframwork}.
Probabilistic programs, viewed as stochastic models, have their reachability and termination probabilities investigated using proof rules~\cite{DBLP:journals/pacmpl/FengCSKKZ23} and Martingale-based approaches~\cite{ChakarovS13,DBLP:conf/popl/ChatterjeeFNH16,AsadiC0GM21}, where the latter are subsequently unified through order-theoretic fixed-point approaches \cite{DBLP:conf/lics/UrabeHH17,TakisakaOUH18,TakisakaOUH21}.

\inlsec{Formal Verification of DNN-controlled Systems}
Modeling DNN-controlled systems as Markov Decision Processes (MDPs) and verifying these models using probabilistic model checkers, such as PRISM \cite{kwiatkowska2011prism} and \textsc{Storm} \cite{hensel2021probabilistic}, constitutes a quantitative verification approach.
 Bacci and Parker \cite{bacci2022verified,bacci2020probabilistic} employ abstract interpretation to construct interval MDPs and yield safety probabilities within bounded time. Carr \textit{et al.}  \cite{carr2021task} propose probabilistic verification of DNN-controlled systems by constraining the analysis to partially observable finite-state models. Amir \textit{et al.} propose a scalable approach based on DNN verification techniques to first support complex  properties such as liveness  \cite{amir2021towards}. 

Reachability analysis is a pivotal qualitative approach in the safety verification of DNN-controlled systems.
Bacci \textit{et al.} \cite{bacci2021verifying} introduce a linear over-approximation-based method for calculating reachable set invariants over an infinite time horizon for DNN-controlled systems. Other reachability analysis approaches, such as Verisig \cite{ivanov2021verisig} and Polar \cite{huang2022polar}, focus solely on  bounded time. These approaches do not consider perturbations as they assume actions on states to be deterministic.

\inlsec{Barrier Certificates for Training and Verifying DNN Controllers} 
BC-based methods~\cite{AbateAEGP21,PeruffoAA21} have recently been investigated for training and verifying DNN controllers. 
The key idea is to train a safe DNN controller through interactive computations of corresponding barrier certificates to ensure qualitative safety   \cite{yang2021iterative,deshmukh2019learning}. Existing BC-based approaches for the  verification of DNN-controlled systems focus solely on qualitative aspects but neglect the consideration of perturbations \cite{sha2021synthesizing,zeng2023safety}. Our approach complements them by accommodating the inherent uncertainty in DNN-controlled systems.

%% file: conclusion.tex
We have systematically studied the BC-based qualitative and quantitative safety verification of DNN-controlled systems.
This involves unifying and transforming the verification problems into a general task of training corresponding neural certificate barriers.
We have also defined the conditions that a trained certificate should satisfy, along with the corresponding lower and upper bounds presented in both linear and exponential forms and $k$-inductive variants.
Through the unification of these verification problems, we have established a comprehensive framework for delivering various safety guarantees, whether qualitatively or quantitatively, in a unified manner.

Our framework sheds light on the quest for scalable and multipurpose safety verification of DNN-controlled systems.
It accommodates both qualitative and quantitative aspects in verified results, spans both finite and infinite time horizons, and encompasses certified bounds presented in both linear and exponential forms.
Our work also showcases the potential to circumvent verification challenges posed by DNN controllers. From our experiments, we acknowledge that both qualitative and quantitative verification results are significantly dependent on the quality of the trained  NBCs.
Our next step is to explore more sophisticated deep learning methods to train valid NBCs for achieving more precise verification results.

%% file: appendix.tex
\begin{center}
	\Large\textbf{Appendix}
\end{center}

\setcounter{section}{0}

\section{Probability Theory}

\input{basic-math}

\section{Supplementary Materials for \Cref{sec:bc}}

\input{barrier-proofs}

\section{Supplementary Materials for \Cref{sec:k-induction}}

\input{k-induction-proofs}

\section{Supplementary Materials for \Cref{sec:neural}}

Similar to those in \Cref{subsec:training}, 
the training phase of the other NBCs also involves two steps, i.e., training data construction and loss
function definition. 
Training data construction can be achieved by discretizing the state space $S$ and constructing a \emph{discretization} $\tilde{S}\subseteq S$ such that for each $s \in S$, there is a $\tilde{s} \in \tilde{S}$ with $||s-\tilde{s}||_1 < \tau $,
where $ \tau>0$ is called the granularity of $\tilde{S}$. 
Once the discretization $\tilde{S}$ is obtained, we construct two finite sets $\tilde{S}_0:=\tilde{S}\cap S_0$ and   $\tilde{S}_u:=\tilde{S}\cap S_u$ 
used for the training process.

We provide  the training loss functions and validation theorems for several barrier certificates, 
The fundamental idea for the others is consistent.
The training loss functions are composed of the conditions of NBCs and the simulation-guided term for tightness. 
In the validation phase, 
we utilize Lipschitz continuity to reduce the validation of the infinite state space to a finite set of states
, by applying stricter conditions.

\subsection{Loss Functions of NBCs}

\paragraph{The loss function of NBCs in \Cref{thm:inf-upper-safe}.}
Similarly, we define the loss function of NBCs in \Cref{thm:inf-upper-safe} as:

\begin{equation}
\begin{split}
&\calL(\theta):=k_1 \cdot \calL_{1}(\theta)+k_2 \cdot \calL_{2}(\theta)+k_3 \cdot \calL_{3}(\theta)+k_4 \cdot \calL_{4}(\theta)+k_5 \cdot \calL_{5}(\theta)\\
&\calL_{1}(\theta)=\frac{1}{|\tilde{S}|} \sum\limits_{s\in \tilde{S}} \left(\mathop{\mathrm{max}}\{ h_\theta(s) - 1,0\} + \mathop{\mathrm{max}}\{ 0 - h_\theta(s),0\}\right)\\
&\calL_{2}(\theta)=\frac{1}{|\tilde{S}_0|} \sum\limits_{s\in \tilde{S}_0} \left(\mathop{\mathrm{max}}\{\epsilon - h_\theta(s),0\} \right)\\
&\calL_{3}(\theta)=\frac{1}{|\tilde{S}_u|} \sum\limits_{s\in \tilde{S}_u} \left(\mathop{\mathrm{max}}\{h_\theta(s)-\epsilon',0\} \right)\\
&\calL_{4}(\theta)=\frac{1}{|\tilde{S} \setminus \tilde{S}_u|}\sum\limits_{s\in \tilde{S} \setminus \tilde{S}_u} \left(  \mathop{\mathrm{max}}\{h_\theta(s) - \gamma\cdot\sum\limits_{s'\in \mathcal{D}_s}\frac{h_\theta(s')}{N}+\zeta ,0\}  \right)\\
&\calL_{5}(\theta)=\frac{1}{|\tilde{S}_0|}\sum\limits_{s\in \tilde{S}_0} \left( \mathop{\mathrm{max}}\{1 - \varmathbb{f}_s - h_\theta(s) ,0\}\right)
\end{split}
\end{equation}
where $\calL_{1}(\theta),\cdots, \calL_{4}(\theta)$ are defined via the four conditions of the barrier certificate, correspondingly. $\calL_{5}(\theta)$ is the regularization term to assure the tightness of upper bounds via the same simulation-based method and \Cref{eq:thm-inf-upperbound}.

\paragraph{The loss function of NBCs in \Cref{thm:fin-lower-safe-linear}.}
The loss function of NBCs in \Cref{thm:fin-lower-safe-linear} is: 

\begin{equation}
\begin{split}
&\calL(\theta):=k_1 \cdot \calL_{1}(\theta)+k_2 \cdot \calL_{2}(\theta)+k_3 \cdot \calL_{3}(\theta)+k_4 \cdot \calL_{4}(\theta)+k_5 \cdot \calL_{5}(\theta)\\
&\calL_{1}(\theta)=\frac{1}{|\tilde{S}|} \sum\limits_{s\in \tilde{S}} \left(\mathop{\mathrm{max}}\{0 - h_\theta(s),0\}\right)\\
&\calL_{2}(\theta)=\frac{1}{|\tilde{S}_0|} \sum\limits_{s\in \tilde{S}_0} \left( \mathop{\mathrm{max}}\{h_\theta(s)-\epsilon,0\}\right)\\
&\calL_{3}(\theta)=\frac{1}{|\tilde{S}_u|} \sum\limits_{s\in \tilde{S}_u} \left(\mathop{\mathrm{max}}\{\lambda - h_\theta(s),0\} \right)\\
&\calL_{4}(\theta)=\frac{1}{|\tilde{S}|}\sum\limits_{s\in \tilde{S}} \left(  \mathop{\mathrm{max}}\{\sum\limits_{s'\in \mathcal{D}_s}\frac{h_\theta(s')}{N} - h_\theta(s) - c + \zeta ,0\}  \right)\\
&\calL_{5}(\theta)=\frac{1}{|\tilde{S}_0|} \sum\limits_{s\in \tilde{S}_0}\sum\limits_{t=1}^{T} \left( \mathop{\mathrm{max}}\{ \varmathbb{f}_t^{s} -1 +(h_\theta(s)+c \cdot t)/\lambda, 0 \}\right)
\end{split}
\end{equation}
where the last term $\calL_{5}(\theta)$ is to assure the tightness of bounds using the simulation-based method. Specifically, we execute the control system from each $s\in \tilde{S}_0$, and calculate the safe frequency at each time step $\varmathbb{f}_t^{s}, t=1,\cdots,T$ of all these trajectories starting from each initial state $s$.

\paragraph{The loss function of NBCs in \Cref{thm:fin-lower-safe-exp}.}

The loss function of NBCs in \Cref{thm:fin-lower-safe-exp} can be defined as: 
\vspace{-2mm}

\begin{equation}
\begin{split}
&\calL(\theta):=k_1 \cdot \calL_{1}(\theta)+k_2 \cdot \calL_{2}(\theta)+k_3 \cdot \calL_{3}(\theta)+k_4 \cdot \calL_{4}(\theta)+k_5 \cdot \calL_{5}(\theta)\\
&\calL_{1}(\theta)=\frac{1}{|\tilde{S}|} \sum\limits_{s\in \tilde{S}} \left(\mathop{\mathrm{max}}\{0 - h_\theta(s),0\} \right)\\
&\calL_{2}(\theta)=\frac{1}{|\tilde{S}_0|} \sum\limits_{s\in \tilde{S}_0} \left(\mathop{\mathrm{max}}\{h_\theta(s)-\gamma,0\}\right)\\
&\calL_{3}(\theta)=\frac{1}{|\tilde{S}_u|} \sum\limits_{s\in \tilde{S}_u} \left(\mathop{\mathrm{max}}\{1 - h_\theta(s),0\} \right)\\
&\calL_{4}(\theta)=\frac{1}{|\tilde{S}\setminus \tilde{S}_u|} \sum\limits_{s\in \tilde{S}\setminus \tilde{S}_u} \left(  \mathop{\mathrm{max}}\{\alpha\sum\limits_{s'\in \mathcal{D}_s}\frac{h_\theta(s')}{N} - h_\theta(s) - \alpha \beta \  + \zeta ,0\}  \right)\\
& \calL_{5}(\theta)=\frac{1}{|\tilde{S}_0|}\sum\limits_{s\in \tilde{S}_0} \sum\limits_{t=1}^{T} \left( \mathop{\mathrm{max}}\{ \varmathbb{f}_t^{s} - 1+\frac{\alpha\beta}{\alpha-1} - (\frac{\alpha\beta}{\alpha-1}-h_\theta(s))\cdot \alpha^{-t},0
\}\right)
\end{split}
\end{equation}
where the last term $\calL_{5}(\theta)$ is to assure the tightness of bounds using the simulation-based method. Specifically, we execute the control system from each $s\in \tilde{S}_0$, and calculate the safe frequency at each time step $\varmathbb{f}_t^s, t=1,\cdots,T$ of all these trajectories.

\vspace{-2mm}

\paragraph{The loss function of NBCs in \Cref{thm:fin-upper-safe-linear}.}
We define the loss function as: 
\vspace{-2mm}
\begin{equation}
\begin{split}
&\calL(\theta):=k_1 \cdot \calL_{1}(\theta)+k_2 \cdot \calL_{2}(\theta)+k_3 \cdot \calL_{3}(\theta)+k_4 \cdot \calL_{4}(\theta)+k_5 \cdot \calL_{5}(\theta)\\
&\calL_{1}(\theta)=\frac{1}{|\tilde{S}|} \sum\limits_{s\in \tilde{S}} \left(\mathop{\mathrm{max}}\{0 - h_\theta(s),0\} \right)\\
&\calL_{2}(\theta)=\frac{1}{|\tilde{S}\setminus \tilde{S}_u|} \sum\limits_{s\in \tilde{S}\setminus \tilde{S}_u} \left(\mathop{\mathrm{max}}\{ h_\theta(s) - \beta,0\}  \right)\\
&\calL_{3}(\theta)=\frac{1}{| \tilde{S}_u|} \sum\limits_{s\in \tilde{S}_u} \left(  \mathop{\mathrm{max}}\{\alpha - h_\theta(s) ,0\} + \mathop{\mathrm{max}}\{ h_\theta(s)  -1 - \beta ,0\} \right)\\
&\calL_{4}(\theta)=\frac{1}{|\tilde{S}\setminus \tilde{S}_u|} \sum\limits_{s\in  \tilde{S}\setminus \tilde{S}_u} \left(   \mathop{\mathrm{max}}\{ c - \sum\limits_{s'\in \mathcal{D}_s}\frac{h_\theta(s')}{N} + h_\theta(s) + \zeta ,0\}  \right)\\
& \calL_{5}(\theta)=\frac{1}{|\tilde{S}_0|} \sum\limits_{s\in \tilde{S}_0}\sum\limits_{t=1}^{T} \left( 1-h_\theta(s)-\frac{1}{2}c\cdot t+\beta - \varmathbb{f}_t^s, 0 \}\right)
\end{split}
\end{equation}

\paragraph{The loss function of NBCs in \Cref{thm:fin-upper-safe-exp}.}
We define the loss function as: 
\vspace{-2mm}
\begin{equation}
\begin{split}
&\calL(\theta):=k_1 \cdot \calL_{1}(\theta)+k_2 \cdot \calL_{2}(\theta)+k_3 \cdot \calL_{3}(\theta)+k_4 \cdot \calL_{4}(\theta)+k_5 \cdot \calL_{5}(\theta)\\
&\calL_{1}(\theta)=\frac{1}{|\tilde{S}\setminus \tilde{S}_u|} \sum\limits_{s\in \tilde{S}\setminus \tilde{S}_u} \left(\mathop{\mathrm{max}}\{0 - h_\theta(s),0\} \right)\\
&\calL_{2}(\theta)=\frac{1}{|\tilde{S}_u|} \sum\limits_{s\in \tilde{S}_u} \left(\mathop{\mathrm{max}}\{K' - h_\theta(s),0\} + \mathop{\mathrm{max}}\{h_\theta(s) -K,0\} \right)\\
&\calL_{3}(\theta)=\frac{1}{|\tilde{S}\setminus \tilde{S}_u|} \sum\limits_{s\in \tilde{S}\setminus \tilde{S}_u} \left(  \mathop{\mathrm{max}}\{\sum\limits_{s'\in \mathcal{D}_s}\frac{h_\theta(s')}{N} - h_\theta(s) + \epsilon \  + \zeta ,0\}  \right)\\
&\calL_{4}(\theta)=\frac{1}{|\tilde{S}\setminus \tilde{S}_u|} \sum\limits_{s\in \tilde{S}\setminus \tilde{S}_u} \left(  \mathop{\mathrm{max}}\{h_\theta(s') - h_\theta(s) - b  + \zeta ,0\} + \mathop{\mathrm{max}}\{a - h_\theta(s') + h_\theta(s)  + \zeta ,0\} \right)\\
& \calL_{5}(\theta)=\frac{1}{|\tilde{S}_0|} \sum\limits_{s\in \tilde{S}_0}\sum\limits_{t=1}^{T} \left( \mathop{\mathrm{max}}\{exp(-\frac{2(\epsilon t- h_\theta(s))^2}{t\cdot (b-a)^2}) - \varmathbb{f}_t^s, 0 \}\right)
\end{split}
\end{equation}

\subsection{NBCs Validation}
\label{subsec:NBC_VALIDATION}

\noindent\textbf{\Cref{theorem:verify_NBC}} Given an $M_{\mu}$ and a function $B:S\rightarrow\Rset$, 
we have  $ \expv_{\delta\sim \mu}[B(f(s,\pi(s+\delta)))\mid s]-B(s)\le 0$ 
for any state $s\in S \setminus S_u$ if the formula below 
\begin{align}
\expv_{\delta\sim \mu}[B(f(\tilde{s},\pi(\tilde{s}+\delta)))\mid \tilde{s}]\le B(\tilde{s})-\zeta
\nonumber
\end{align} holds 
for any state $\tilde{s}\in \tilde{S} \setminus \tilde{S}_u$, where $\zeta=  \tau\cdot L_B\cdot (1+L_f\cdot (1+L_\pi))$ with  $L_f,L_\pi, L_B$ being the Lipschitz constants of $f,\pi$ and $B$,  respectively.

\begin{proof}
 Let $L_f,L_\pi,L_B$ be the Lipschitz constants for the system dynamics $f$, the trained policy $\pi$ and the neural network function $B$, respectively.     Given a state $s\in S \setminus S_u$, let $\tilde{s}$ be such that $||\tilde{s}-s||_1\le \tau$. 
By the Lipschitz continuities, we have that
\begin{align*}
&\expv_{\delta\sim \mu}[B(f(s,\pi(s+\delta)))] \\
\le & \expv_{\delta\sim \mu}[B(f(\tilde{s},\pi(\tilde{s}+\delta)))]
+||f(\tilde{s},\pi(\tilde{s}+\delta))-f(s,\pi(s+\delta))||_1\cdot L_B \\
\le& \expv_{\delta\sim \mu}[B(f(\tilde{s},\pi(\tilde{s}+\delta)))]
+||(\tilde{s},\pi(\tilde{s}+\delta))-(s,\pi(s+\delta))||_1\cdot L_B\cdot L_f   \\
\le& \expv_{\delta\sim \mu}[B(f(\tilde{s},\pi(\tilde{s}+\delta)))]
 + ||\tilde{s}-s||_1\cdot L_B\cdot L_f \cdot (1+L_\pi) \\
\le& \expv_{\delta\sim \mu}[B(f(\tilde{s},\pi(\tilde{s}+\delta)))]
+ \tau\cdot L_B\cdot L_f \cdot (1+L_\pi)
\end{align*}
and 
\begin{align*}
-B(s)\le -B(\tilde{s})+||\tilde{s}-s||_1\cdot L_B\le -B(\tilde{s})+\tau\cdot L_B,
\end{align*}
Thus, we can derive that
\begin{align*}
& \expv_{\delta\sim \mu}[B(f(s,\pi(s+\delta)))]-B(s)\\
\le &\expv_{\delta\sim \mu}[B(f(\tilde{s},\pi(\tilde{s}+\delta)))]
+ \tau\cdot L_B\cdot L_f \cdot (1+L_\pi)  -B(\tilde{s})+\tau\cdot L_B\\
\le &\expv_{\delta\sim \mu}[B(f(\tilde{s},\pi(\tilde{s}+\delta)))]
+ \tau\cdot L_B\cdot L_f \cdot (1+L_\pi)  -B(\tilde{s})+\tau\cdot L_B\\
\le & \expv_{\delta\sim \mu}[B(f(\tilde{s},\pi(\tilde{s}+\delta)))] - B(\tilde{s}) +\zeta \\
\le & 0
\end{align*}
\end{proof}

\paragraph{Validating NBCs in \Cref{thm:inf-upper-safe}.} A candidate NBC in \Cref{thm:inf-upper-safe}  can be validated if it meets the conditions in ~\Cref{eq:thm-inf-upper-bound,eq:thm-inf-upper-initial,eq:thm-inf-upper-unsafe,eq:thm-inf-upper-decrease}. 
For the first three conditions, we can check 

\begin{equation}
 \mathop{\mathrm{inf}}\limits_{s\in S} B(s)\ge 0 \   \land \ 
 \mathop{\mathrm{sup}}\limits_{s\in S} B(s)\le 1 \   \land \ \mathop{\mathrm{inf}}\limits_{s\in S_0} B(s)\ge \epsilon \  \land \ \mathop{\mathrm{sup}}\limits_{s\in S_u} B(s) \le \epsilon'
\end{equation}
using the interval bound propagation approach  ~\cite{Interval_Bound_Propagation,interval_pr}.  
Similarly, for the condition in \Cref{eq:thm-inf-upper-decrease}, 
~\Cref{theorem:verify_infi_upper} reduces validation from infinite states to finite ones, which is easier to check.  

\begin{theorem}
\label{theorem:verify_infi_upper}
Given an $M_{\mu}$ and a function $B:S\rightarrow\Rset$, 
we have  B(s)-$\gamma\cdot\expv_{\delta\sim \mu}[B(f(s,\pi(s+\delta)))\mid s]\le 0\   \text{for all }s\in S \setminus S_u$ if the formula below 
\begin{align}
B(\tilde{s}) - \gamma \cdot \expv_{\delta\sim \mu }[B(f(\tilde{s},\pi(\tilde{s}+\delta)))] 
+ \zeta \le \  0
\end{align} holds 
for any state $\tilde{s}\in \tilde{S} \setminus \tilde{S}_u$, where  $\zeta=\tau\cdot L_B \cdot (  \gamma  \cdot L_f \cdot (1+L_\pi) + 1)$ with  $L_f,L_\pi, L_B$ being the Lipschitz constants of $f,\pi$ and $B$,  respectively.
\end{theorem} 

\begin{proof}
By the Lipschitz continuities and $\gamma\in (0,1)$, we have that
\begin{align*}
& \gamma \cdot \expv_{\delta\sim \mu}[B(f(s,\pi(s+\delta)))] \\
\ge & \gamma \cdot \{\expv_{\delta\sim \mu}[B(f(\tilde{s},\pi(\tilde{s}+\delta)))]
 -||f(\tilde{s},\pi(\tilde{s}+\delta))-f(s,\pi(s+\delta))||_1\cdot L_B\} \\
\ge&  \gamma \cdot \{\expv_{\delta\sim \mu}[B(f(\tilde{s},\pi(\tilde{s}+\delta)))]
 -||(\tilde{s},\pi(\tilde{s}+\delta))-(s,\pi(s+\delta))||_1\cdot L_B\cdot L_f \}  \\
\ge& \gamma \cdot \{\expv_{\delta\sim \mu}[B(f(\tilde{s},\pi(\tilde{s}+\delta)))]
  - ||\tilde{s}-s||_1\cdot L_B\cdot L_f \cdot (1+L_\pi) \}\\
\ge&  \gamma \cdot \{\expv_{\delta\sim \mu}[B(f(\tilde{s},\pi(\tilde{s}+\delta)))]
- \tau\cdot L_B\cdot L_f \cdot (1+L_\pi)\}
\end{align*}
and 
\begin{align*}
-B(s)\ge -B(\tilde{s})-||s-\tilde{s}||_1\cdot L_B\ge -B(\tilde{s})-\tau\cdot L_B 
\end{align*}
Thus, we can derive that
\begin{align*}
    & \gamma \cdot \expv_{\delta\sim \mu}[B(f(s,\pi(s+\delta)))] -B(s) \\ \ge \ 
    &  \gamma  \cdot \{\expv_{\delta\sim \mu}[B(f(\tilde{s},\pi(\tilde{s}+\delta)))]
- \tau\cdot L_B\cdot L_f \cdot (1+L_\pi)\} -B(\tilde{s})-\tau\cdot L_B \\
= \ & \gamma \cdot \expv_{\delta\sim \mu }[B(f(\tilde{s},\pi(\tilde{s}+\delta)))]  - B(\tilde{s})
- \gamma  \cdot \tau \cdot L_B\cdot L_f \cdot (1+L_\pi)-\tau\cdot L_B
\end{align*}
Finally, we have:
\begin{align*}
    & B(s) - \gamma \cdot \expv_{\delta\sim \mu}[B(f(s,\pi(s+\delta)))] \\ \le \ 
    &  B(\tilde{s}) - \gamma \cdot \expv_{\delta\sim \mu }[B(f(\tilde{s},\pi(\tilde{s}+\delta)))] 
+ \gamma  \cdot \tau \cdot L_B\cdot L_f \cdot (1+L_\pi) + \tau\cdot L_B \\ 
=  \ & B(\tilde{s}) - \gamma \cdot \expv_{\delta\sim \mu }[B(f(\tilde{s},\pi(\tilde{s}+\delta)))] 
+ \zeta \\
\le  \  & 0
\end{align*}
\end{proof}

\paragraph{Validating NBCs in \Cref{thm:fin-lower-safe-linear}.} A candidate NBC can be validated if it meets the conditions in ~\Cref{eq:thm-fin-lower-linear-nonnegative,eq:thm-fin-lower-linear-initial,eq:thm-fin-lower-linear-unsafe,eq:thm-fin-lower-linear-decrease}. 
For the first three conditions, we can check 

\begin{equation}
 \mathop{\mathrm{inf}}\limits_{s\in S} B(s)\ge 0 \    \land \ \mathop{\mathrm{sup}}\limits_{s\in S_0} B(s)\le \epsilon \  \land \ \mathop{\mathrm{inf}}\limits_{s\in S_u} B(s) \ge \lambda
\end{equation}
using the interval bound propagation approach  ~\cite{Interval_Bound_Propagation,interval_pr}.  
For the condition in \Cref{eq:thm-fin-lower-linear-decrease}, 
~\Cref{theorem:verify_fini_lower_linear} reduces validation from infinite states to finite ones, which is easier to check. 

\begin{theorem}
\label{theorem:verify_fini_lower_linear}
Given an $M_{\mu}$ and a function $B:S\rightarrow\Rset$, 
we have  $\expv_{\delta\sim \mu}[B(f(s,\pi(s+\delta)))\mid s]-B(s)\le c\   \text{for all }s\in S$ if the formula below 
\begin{align}
\expv_{\delta\sim \mu}[B(f(\tilde{s},\pi(\tilde{s}+\delta)))\mid \tilde{s}]-B(\tilde{s}) + \zeta \le c
\end{align} holds 
for any state $\tilde{s}\in \tilde{S}$, where  $\zeta=  \tau\cdot L_B\cdot (1+L_f\cdot (1+L_\pi))$ with  $L_f,L_\pi, L_B$ being the Lipschitz constants of $f,\pi$ and $B$,  respectively.
\end{theorem} 

\begin{proof}
Given a state $s\in S$, let $\tilde{s}$ be such that $||\tilde{s}-s||_1\le \tau$. 
By the Lipschitz continuities, we have that
\begin{align*}
&\expv_{\delta\sim \mu}[B(f(s,\pi(s+\delta)))] \\
\le & \expv_{\delta\sim \mu}[B(f(\tilde{s},\pi(\tilde{s}+\delta)))]
+||f(\tilde{s},\pi(\tilde{s}+\delta))-f(s,\pi(s+\delta))||_1\cdot L_B \\
\le& \expv_{\delta\sim \mu}[B(f(\tilde{s},\pi(\tilde{s}+\delta)))]
+||(\tilde{s},\pi(\tilde{s}+\delta))-(s,\pi(s+\delta))||_1\cdot L_B\cdot L_f   \\
\le& \expv_{\delta\sim \mu}[B(f(\tilde{s},\pi(\tilde{s}+\delta)))]
 + ||\tilde{s}-s||_1\cdot L_B\cdot L_f \cdot (1+L_\pi) \\
\le& \expv_{\delta\sim \mu}[B(f(\tilde{s},\pi(\tilde{s}+\delta)))]
+ \tau\cdot L_B\cdot L_f \cdot (1+L_\pi)
\end{align*}
and 
\begin{align*}
-B(s)\le -B(\tilde{s})+||\tilde{s}-s||_1\cdot L_B\le -B(\tilde{s})+\tau\cdot L_B,
\end{align*}
Thus, we can derive that
\begin{align*}
& \expv_{\delta\sim \mu}[B(f(s,\pi(s+\delta)))]-B(s) -c\\
\le &\expv_{\delta\sim \mu}[B(f(\tilde{s},\pi(\tilde{s}+\delta)))] -B(\tilde{s})
+ \tau\cdot L_B\cdot L_f \cdot (1+L_\pi)  +\tau\cdot L_B -c\\
\le & \expv_{\delta\sim \mu}[B(f(\tilde{s},\pi(\tilde{s}+\delta)))] - B(\tilde{s}) +\zeta -c \\
\le & 0
\end{align*}
\end{proof}

\paragraph{Validating NBCs in \Cref{thm:fin-lower-safe-exp}.} A candidate NBC can be validated if it meets the conditions in ~\Cref{eq:thm-fin-lower-exp-nonnegative,eq:thm-fin-lower-exp-initial,eq:thm-fin-lower-exp-unsafe,eq:thm-fin-lower-exp-moni}. 
For the first three conditions, we can check 

\begin{equation}
 \mathop{\mathrm{inf}}\limits_{s\in S } B(s)\ge 0 \    \land \ \mathop{\mathrm{sup}}\limits_{s\in S_0} B(s)\le \gamma \  \land \ \mathop{\mathrm{inf}}\limits_{s\in S_u} B(s) \ge 1
\end{equation}
using the interval bound propagation approach  ~\cite{Interval_Bound_Propagation,interval_pr}.  
For the condition in \Cref{eq:thm-fin-lower-exp-moni}, 
~\Cref{theorem:verify_fini_lower_exp_t} reduces validation from infinite states to finite ones, which is easier to check. 

\begin{theorem}
\label{theorem:verify_fini_lower_exp_t}
Given an $M_{\mu}$ and a function $B:S\rightarrow\Rset$, 
we have  $\alpha \cdot \expv_{\delta\sim \mu}[B(f(s,\pi(s+\delta)))\mid s]-B(s) \le \alpha \beta $ if the formula below 
\begin{align}
\alpha \cdot \expv_{\delta\sim \mu}[B(f(\tilde{s},\pi(\tilde{s}+\delta)))\mid \tilde{s}]-B(\tilde{s})\le \alpha \beta 
\end{align} holds 
for any state  $\tilde{s}\in \tilde{S}\setminus \tilde{S}_u$, where  $\zeta=  \alpha \cdot  \tau\cdot L_B\cdot L_f \cdot (1+L_\pi)  + \tau\cdot L_B$ with  $L_f,L_\pi, L_B$ being the Lipschitz constants of $f,\pi$ and $B$,  respectively.
\end{theorem} 

\begin{proof}
Given a state $s\in S$, let $\tilde{s}$ be such that $||\tilde{s}-s||_1\le \tau$. 
By the Lipschitz continuities, we have that
\begin{align*}
&\expv_{\delta\sim \mu}[B(f(s,\pi(s+\delta)))] \\
\le & \expv_{\delta\sim \mu}[B(f(\tilde{s},\pi(\tilde{s}+\delta)))]
+||f(\tilde{s},\pi(\tilde{s}+\delta))-f(s,\pi(s+\delta))||_1\cdot L_B \\
\le& \expv_{\delta\sim \mu}[B(f(\tilde{s},\pi(\tilde{s}+\delta)))]
+||(\tilde{s},\pi(\tilde{s}+\delta))-(s,\pi(s+\delta))||_1\cdot L_B\cdot L_f   \\
\le& \expv_{\delta\sim \mu}[B(f(\tilde{s},\pi(\tilde{s}+\delta)))]
 + ||\tilde{s}-s||_1\cdot L_B\cdot L_f \cdot (1+L_\pi) \\
\le& \expv_{\delta\sim \mu}[B(f(\tilde{s},\pi(\tilde{s}+\delta)))]
+ \tau\cdot L_B\cdot L_f \cdot (1+L_\pi)
\end{align*}
and 
\begin{align*}
-B(s)\le -B(\tilde{s})+||\tilde{s}-s||_1\cdot L_B\le -B(\tilde{s})+\tau\cdot L_B,
\end{align*}
Thus, we can derive that
\begin{align*}
&\alpha \cdot \expv_{\delta\sim \mu}[B(f(s,\pi(s+\delta)))]-B(s)  - \alpha \beta \\
\le & \alpha \cdot \expv_{\delta\sim \mu}[B(f(\tilde{s},\pi(\tilde{s}+\delta)))] -B(\tilde{s}) - \alpha \beta
+ \alpha \cdot  \tau\cdot L_B\cdot L_f \cdot (1+L_\pi)  + \tau\cdot L_B  \\
\le & \expv_{\delta\sim \mu}[B(f(\tilde{s},\pi(\tilde{s}+\delta)))] - B(\tilde{s}) - \alpha \beta +\zeta   \\
\le & 0
\end{align*}
\end{proof}

\paragraph{Validating NBCs in \Cref{thm:fin-upper-safe-linear}.} A candidate NBC can be validated if it meets the conditions in ~\Cref{eq:thm-fin-upper-linear-nonnegative,eq:thm-fin-upper-exp-initial,eq:thm-fin-upper-linear-unsafe,eq:thm-fin-upper-linear-moni}. 
For the first two conditions, we can check 

\begin{equation}
 \mathop{\mathrm{inf}}\limits_{s\in S} B(s)\ge 0 \      \land \ \mathop{\mathrm{sup}}\limits_{s\in S \setminus S_u } B(s) \le \beta
\end{equation}
using the interval bound propagation approach  ~\cite{Interval_Bound_Propagation,interval_pr}.  
The condition in \Cref{eq:thm-fin-upper-linear-unsafe}  is satisfied straightforwardly, Because candidate NBCs are neural networks that are Lipschitz continuous~\cite{RuanHK18}. 
For the condition in \Cref{eq:thm-fin-upper-linear-moni}, 
~\Cref{theorem:verify_fini_upper_linear} reduces validation from infinite states to finite ones, which is easier to check. 

\begin{theorem}
\label{theorem:verify_fini_upper_linear}
Given an $M_{\mu}$ and a function $B:S\rightarrow\Rset$, 
we have  $\expv_{\delta\sim \mu}[B(f(s,\pi(s+\delta)))\mid s]-B(s) \ge c  $ if the formula below 
\begin{align}
\expv_{\delta\sim \mu}[B(f(\tilde{s},\pi(\tilde{s}+\delta)))\mid \tilde{s}]-B(\tilde{s}) \ge c + \zeta 
\end{align} holds 
for any state  $\tilde{s}\in \tilde{S}\setminus \tilde{S}_u$, where  $\zeta=  \tau \cdot L_B\cdot L_f \cdot (1+L_\pi) + \tau\cdot L_B$ with  $L_f,L_\pi, L_B$ being the Lipschitz constants of $f,\pi$ and $B$,  respectively.
\end{theorem} 

\begin{proof}
By the Lipschitz continuities and $\gamma\in (0,1)$, we have that
\begin{align*}
& \expv_{\delta\sim \mu}[B(f(s,\pi(s+\delta)))] \\
\ge & \expv_{\delta\sim \mu}[B(f(\tilde{s},\pi(\tilde{s}+\delta)))]
 -||f(\tilde{s},\pi(\tilde{s}+\delta))-f(s,\pi(s+\delta))||_1\cdot L_B \\
\ge& \expv_{\delta\sim \mu}[B(f(\tilde{s},\pi(\tilde{s}+\delta)))]
 -||(\tilde{s},\pi(\tilde{s}+\delta))-(s,\pi(s+\delta))||_1\cdot L_B\cdot L_f   \\
\ge& \expv_{\delta\sim \mu}[B(f(\tilde{s},\pi(\tilde{s}+\delta)))]
  - ||\tilde{s}-s||_1\cdot L_B\cdot L_f \cdot (1+L_\pi)\\
\ge&  \expv_{\delta\sim \mu}[B(f(\tilde{s},\pi(\tilde{s}+\delta)))]
- \tau\cdot L_B\cdot L_f \cdot (1+L_\pi)
\end{align*}
and 
\begin{align*}
-B(s)\ge -B(\tilde{s})-||s-\tilde{s}||_1\cdot L_B\ge -B(\tilde{s})-\tau\cdot L_B 
\end{align*}
Thus, we can derive that
\begin{align*}
    &  \expv_{\delta\sim \mu}[B(f(s,\pi(s+\delta)))] -B(s) -c \\ \ge \ 
    &  \expv_{\delta\sim \mu}[B(f(\tilde{s},\pi(\tilde{s}+\delta)))]
- \tau\cdot L_B\cdot L_f \cdot (1+L_\pi) -B(\tilde{s})  -\tau\cdot L_B -c\\
= \ &  \expv_{\delta\sim \mu }[B(f(\tilde{s},\pi(\tilde{s}+\delta)))]  - B(\tilde{s})  
- \tau \cdot L_B\cdot L_f \cdot (1+L_\pi)-\tau\cdot L_B -c
\end{align*}
Finally, we have:
\begin{align*}
    & B(s) -  \expv_{\delta\sim \mu}[B(f(s,\pi(s+\delta)))] + c \\ 
    \le \     &  B(\tilde{s}) - \expv_{\delta\sim \mu }[B(f(\tilde{s},\pi(\tilde{s}+\delta)))] 
+  \tau \cdot L_B\cdot L_f \cdot (1+L_\pi) + \tau\cdot L_B +c \\ 
=  \ & B(\tilde{s}) -  \expv_{\delta\sim \mu }[B(f(\tilde{s},\pi(\tilde{s}+\delta)))]  + c
+ \zeta  \\
\le  \  & 0
\end{align*}
\end{proof}

\paragraph{Validating NBCs in \Cref{thm:fin-upper-safe-exp}.} A candidate NBC can be validated if it meets the conditions in ~\Cref{eq:thm-fin-upper-exp-nonnegative,eq:thm-fin-upper-exp-unsafe,eq:thm-fin-upper-exp-moni,eq:thm-fin-upper-exp-diffbound}. 
For the first two conditions, we can check 

\begin{equation}
 \mathop{\mathrm{inf}}\limits_{s\in S \setminus S_u} B(s)\ge 0 \    \land \ \mathop{\mathrm{inf}}\limits_{s\in S_u} B(s)\ge K' \  \land \ \mathop{\mathrm{sup}}\limits_{s\in S_u} B(s) \le K
\end{equation}
using the interval bound propagation approach  ~\cite{Interval_Bound_Propagation,interval_pr}.  
\Cref{eq:thm-fin-upper-exp-unsafe}  is satisfied straightforwardly, Because candidate NBCs are neural networks that are Lipschitz continuous~\cite{RuanHK18}. 
For the condition in \Cref{eq:thm-fin-upper-exp-moni}, 
~\Cref{theorem:verify_fini_upper_exp} reduces validation from infinite states to finite ones, which is easier to check. 

\begin{theorem}
\label{theorem:verify_fini_upper_exp}
Given an $M_{\mu}$ and a function $B:S\rightarrow\Rset$, 
we have  $\expv_{\delta\sim \mu}[B(f(s,\pi(s+\delta)))\mid s]-B(s)\le -\epsilon $ if the formula below 
\begin{align}
\expv_{\delta\sim \mu}[B(f(\tilde{s},\pi(\tilde{s}+\delta)))\mid \tilde{s}]-B(\tilde{s})\le -\epsilon 
\end{align} holds 
for any state  $\tilde{s}\in \tilde{S}\setminus \tilde{S}_u$, where  $\zeta=  \tau\cdot L_B\cdot (1+L_f\cdot (1+L_\pi))$ with  $L_f,L_\pi, L_B$ being the Lipschitz constants of $f,\pi$ and $B$,  respectively.
\end{theorem} 

\begin{proof}
Given a state $s\in S$, let $\tilde{s}$ be such that $||\tilde{s}-s||_1\le \tau$. 
By the Lipschitz continuities, we have that
\begin{align*}
&\expv_{\delta\sim \mu}[B(f(s,\pi(s+\delta)))] \\
\le & \expv_{\delta\sim \mu}[B(f(\tilde{s},\pi(\tilde{s}+\delta)))]
+||f(\tilde{s},\pi(\tilde{s}+\delta))-f(s,\pi(s+\delta))||_1\cdot L_B \\
\le& \expv_{\delta\sim \mu}[B(f(\tilde{s},\pi(\tilde{s}+\delta)))]
+||(\tilde{s},\pi(\tilde{s}+\delta))-(s,\pi(s+\delta))||_1\cdot L_B\cdot L_f   \\
\le& \expv_{\delta\sim \mu}[B(f(\tilde{s},\pi(\tilde{s}+\delta)))]
 + ||\tilde{s}-s||_1\cdot L_B\cdot L_f \cdot (1+L_\pi) \\
\le& \expv_{\delta\sim \mu}[B(f(\tilde{s},\pi(\tilde{s}+\delta)))]
+ \tau\cdot L_B\cdot L_f \cdot (1+L_\pi)
\end{align*}
and 
\begin{align*}
-B(s)\le -B(\tilde{s})+||\tilde{s}-s||_1\cdot L_B\le -B(\tilde{s})+\tau\cdot L_B,
\end{align*}
Thus, we can derive that
\begin{align*}
& \expv_{\delta\sim \mu}[B(f(s,\pi(s+\delta)))]-B(s) +\epsilon \\
\le &\expv_{\delta\sim \mu}[B(f(\tilde{s},\pi(\tilde{s}+\delta)))] -B(\tilde{s})
+ \tau\cdot L_B\cdot L_f \cdot (1+L_\pi)  +\tau\cdot L_B + \epsilon \\
\le & \expv_{\delta\sim \mu}[B(f(\tilde{s},\pi(\tilde{s}+\delta)))] - B(\tilde{s}) +\zeta + \epsilon  \\
\le & 0
\end{align*}
\end{proof}

\paragraph{Validating NBCs in \Cref{thm:almost-safe-induc}.} A candidate NBC can be validated if it meets the conditions in ~\Cref{eq:thm-almost-indc-initial,eq:thm-almost-indc-unsafe,eq:thm-almost-indc-decrease}. 
For the first two conditions, we can check 

\begin{equation}
 \mathop{\mathrm{sup}}\limits_{s\in S_0} B(s)\le 0 \    \land \ \mathop{\mathrm{inf}}\limits_{s\in S_u} B(s) > 0
\end{equation}
using the interval bound propagation approach  ~\cite{Interval_Bound_Propagation,interval_pr}.  
For the condition in \Cref{eq:thm-almost-indc-decrease}, 
~\Cref{theorem:K_inductive} reduces validation from infinite states to finite ones, which is easier to check. 

\begin{theorem}
\label{theorem:K_inductive}
Given an $M_{\mu}$ and a function $B:S\rightarrow\Rset$, 
we have  $\textstyle\bigwedge_{0\le i<k} (B(g_{\pi,f}^{i}(s,\Delta^i))\le 0)\Longrightarrow B(g_{\pi,f}^k(s,\Delta^k))\le 0  \  \forall (s,\Delta^i)\in S\times W^i$ if the formula below 
\begin{align}
B(f(\tilde{s},\pi(\tilde{s}+\delta)))) -B(\tilde{s}) + \zeta \le 0
\end{align} holds 
for any state $\tilde{s}\in \tilde{S}$, where  $\zeta=  \tau\cdot L_B\cdot (1+L_f\cdot (1+L_\pi))$ with  $L_f,L_\pi, L_B$ being the Lipschitz constants of $f,\pi$ and $B$,  respectively.
\end{theorem} 

\begin{proof}
Given a state $s\in S$, let $\tilde{s}$ be such that $||\tilde{s}-s||_1\le \tau$. 
By the Lipschitz continuities, we have that
\begin{align*}
&B(f(s,\pi(s+\delta))) \\
\le & B(f(\tilde{s},\pi(\tilde{s}+\delta)))
+||f(\tilde{s},\pi(\tilde{s}+\delta))-f(s,\pi(s+\delta))||_1\cdot L_B \\
\le& B(f(\tilde{s},\pi(\tilde{s}+\delta)))
+||(\tilde{s},\pi(\tilde{s}+\delta))-(s,\pi(s+\delta))||_1\cdot L_B\cdot L_f   \\
\le& B(f(\tilde{s},\pi(\tilde{s}+\delta)))
 + ||\tilde{s}-s||_1\cdot L_B\cdot L_f \cdot (1+L_\pi) \\
\le& B(f(\tilde{s},\pi(\tilde{s}+\delta)))
+ \tau\cdot L_B\cdot L_f \cdot (1+L_\pi)
\end{align*}
and 
\begin{align*}
-B(s)\le -B(\tilde{s})+||\tilde{s}-s||_1\cdot L_B\le -B(\tilde{s})+\tau\cdot L_B,
\end{align*}
Thus, we can derive that
\begin{align*}
& B(f(s,\pi(s+\delta)))-B(s)\\
\le &B(f(\tilde{s},\pi(\tilde{s}+\delta))) -B(\tilde{s})
+ \tau\cdot L_B\cdot L_f \cdot (1+L_\pi)  +\tau\cdot L_B \\
\le & B(f(\tilde{s},\pi(\tilde{s}+\delta))) - B(\tilde{s}) +\zeta  \\
\le & 0
\end{align*}
 \end{proof}

\section{Implementation Details and Additional Experimental Results}

\subsection{Benchmarks and Experimental Settings}

To demonstrate the generality of our approach, we train systems with different activation functions and network structures of the planted NNs, using different DRL algorithms such as DQN~\cite{DQN} and DDPG~\cite{DDPG}. 
\Cref{table:benchmarks_setting} gives the details of training settings.

\begin{table}[t]
	\centering
	\footnotesize
	\setlength{\tabcolsep}{13pt}
	\caption{Experimental settings.}
	\begin{tabular}{l|r r r r r r r}
		\hline
		\textbf{Task}&\centering \textbf{Dim.}&\centering \textbf{Alg.} &\centering \textbf{A.F.}&\centering \textbf{Size}&\textbf{A.T.} &\textbf{S.P.}  \\
		\hline
		CP & 4 & DQN & ReLU & $3\times200$  & Dis.  & Gym \\
		PD & 3 & DDPG & Sigmoid & $2\times200$  & Cont.  & Gym \\
		Tora & 4 & DDPG & Sigmoid & $2\times200$  & Cont.  & Verisig \\
		B1 & 2 & DDPG & ReLU & $2\times100$  & Cont.  & Verisig \\
		\hline
	\end{tabular}
	\begin{tablenotes}
		\footnotesize \item \textbf{Remarks.}
		\textbf{Dim.}: dimension; 
		\textbf{Alg.}: DRL algorithm; 
		\textbf{A.F.}: activation function;   
		\textbf{A.T.}: action type; 
		\textbf{S.P.}: sources of problems;
		\textbf{Dis.}: discrete; 
		\textbf{Cont.}: continuous.
	\end{tablenotes}
	\vspace{-2mm}
	\label{table:benchmarks_setting}
\end{table}

\begin{enumerate}

\item \textbf{CartPole (CP).}
A pole is attached by an un-actuated joint to a cart. The goal of training is to learn a controller that prevents the pole from falling over by applying a force of $+1$ or $-1$ to the cart.

\item \textbf{Pendulum (PD).} A pendulum that can rotate around an endpoint is delineated. Starting from a random position, the pendulum shall swing up and stay upright. 

\item \textbf{Tora.}
A cart is attached to a wall with a spring. It is free to move on a frictionless surface. Inside the cart, there is an arm free to rotate about an axis. 
The controller's goal is to stabilize the system at the equilibrium state where all the variables are 0.

\item \textbf{B1.} A  classic nonlinear systems, where agents in both systems aim to avoid dangerous areas from the preset initial state space.

\end{enumerate}

\begin{table}[h]
	\centering
	\footnotesize
	\setlength{\tabcolsep}{2.5pt}
		\caption{Synthesis time of different NBCs (in seconds).}
	\begin{tabular}{l|r r r r r r r r r r}
        \hline
        Task & L.B.i. & $k$-L.B.i. & U.B.i. & $k$-U.B.i. &$l$-L.B.f. & $e$-L.B.f.&$l$-U.B.f. &$e$-U.B.f. & Quali. & $k$-Quali. \\
        \hline
        CP & 2318.5  &  1876.0 & 2891.9  & 2275.3  & 3127.5 & 3359.0& 3277.8 & 3509.1& 1755.2 & 1059.3   \\
        PD & 1941.6  &  1524.0 & 2282.7  & 1491.5  &2015.5 &2218.3 & 2076.8& 2272.2 & 1433.7 &  811.9   \\
		Tora & 280.3  &  218.5 & 895.1  &  650.7  & 396.4& 429.6 & 340.6 &501.7 &623.0 &  415.8 \\
		B1 & 587.4  &  313.6 & 1127.3  &  840.1 & 775.9&824.7 & 894.2&1184.6 &1022.3 & 566.7 \\
		\hline
	\end{tabular}
	\label{table:Synthesis_time_all}
\end{table}

\subsection{Efficiency of Training and Validating NBCs.}

\Cref{table:Synthesis_time_all} shows synthesis time of different NBCs in seconds, i.e., L.B.i.(NBCs for lower bounds on infinite-time safety), $k$-L.B.i.($k$-inductive lower bounds on infinite-time safety), U.B.i.(NBCs for upper bounds on infinite-time safety), $k$-U.B.i.($k$-inductive upper bounds on infinite-time safety), $l$-L.B.f.(NBCs for linear lower bounds on finite-time safety), $e$-L.B.f.(NBCs for exponential lower bounds on finite-time safety),  $l$-U.B.f.(NBCs for linear upper bounds on finite-time safety), $e$-U.B.f.(NBCs for exponential upper bounds on finite-time safety), Quali.(NBCs for almost-surely safety), and $k$-Quali.(NBCs for $k$-inductive almost-surely safety).

In general,   high-dimensional systems e.g., CP (4-dimensional state space)  cost more time than low-dimensional ones, e.g, B1 (2-dimensional state space). 
That is because the validation step suffers from the curse of high-dimensionality~\cite{Stability_Guarantees}.

%% file: basic-math.tex
We start by reviewing some notions from probability theory.
\vskip 2pt 
\noindent\textbf{Random Variables and Stochastic Processes.}
A probability space is a triple ($\Omega, {\cal F}, {\mathbb P}$), where $\Omega$ is a non-empty sample space, ${\cal F}$ is a
$\sigma$-algebra over $\Omega$, and  $\mathbb P(\cdot)$ is a probability measure over $\cal F$, i.e. a function $\mathbb P$: ${\cal F}  \rightarrow [0,1]$ that satisfies the following properties: (1) ${\mathbb P}(\emptyset)=0$, (2)${\mathbb P}(\Omega - A)=1-{\mathbb P}[A]$, and (3) ${\mathbb P} (\textstyle\cup_{i=0}^{\infty} A_i)= \textstyle\sum_{i=0}^{\infty}{\mathbb P}(A_i)$ for  any
sequence $\{A_i\}_{i=0}^{\infty}$ of pairwise disjoint sets in ${\cal F}$.

Given a probability space ($\Omega, {\cal F}, {\mathbb P}$), a
random variable is a function $X: \Omega \rightarrow {\mathbb R} \cup \{\infty\}$ that is $\cal F$-measurable, i.e., for each $a \in {\mathbb R}$ we have that $\{\omega \in \Omega | X(\omega) \leq a \} \in {\cal F}$.
Moreover, a discrete-time stochastic process is a sequence $\{X_n\}_{n=0}^{\infty}$ of random variables in ($\Omega, {\cal F}, {\mathbb P}$).

\vskip 2pt 
\noindent\textbf{Conditional Expectation.}
Let ($\Omega, {\cal F}, {\mathbb P}$) be a probability space and $X$ be a random variable in ($\Omega, {\cal F}, {\mathbb P}$). The expected value of the random variable $X$,
denoted by ${\mathbb E}[X]$, is the Lebesgue integral of $X$ wrt $\mathbb P$. If the range
of $X$ is a countable set $A$, then ${\mathbb E}[X]= \textstyle\sum_{\omega \in A}\omega \cdot {\mathbb P}(X=\omega)$. 
Given a sub-sigma-algebra ${\cal F}' \subseteq {\cal F}$, a conditional expectation of $X$ for the given ${\cal F}'$ is a ${\cal F}'$-measurable random variable $Y$ such
that, for any $A \in {\cal F}'$, we have:
	\begin{equation}
	\begin{split}
     {\mathbb E}[X \cdot {\mathbb I}_{A}]={\mathbb E}[Y \cdot {\mathbb I}_{A}]
	\end{split}
	\end{equation}

\noindent Here, ${\mathbb I}_{A}: \Omega \rightarrow \{0,1\}$  is an indicator function of $A$, defined as ${\mathbb I}_{A}(\omega)=1$ if $\omega \in A$ and ${\mathbb I}_{A}(\omega)=0$ if $\omega \notin A$. 
Moreover, whenever the conditional expectation exists, it is also almost-surely unique, i.e., for any two ${\cal F}'$-measurable random variables $Y$ and $Y'$ which are conditional expectations of $X$ for given ${\cal F}'$, we have that ${\mathbb P}(Y=Y')=1$.

\vskip 2pt 
\noindent\textbf{Filtrations and Stopping Times.}
A filtration of the probability space ($\Omega, {\cal F}, {\mathbb P}$) is an infinite sequence $\{{{\cal F}_n}\}_{n=0}^{\infty}$ such that
for every $n$, the triple ($\Omega, {\cal F}_n, {\mathbb P}$) is a probability space and ${\cal F}_n \subseteq {\cal F}_{n+1} \subseteq {\cal F}$. 
A stopping time with respect to a filtration $\{{{\cal F}_n}\}_{n=0}^{\infty}$  is a random variable $T: \Omega \rightarrow {\mathbb N}_0 \cup \{\infty\}$ such that, for every $i \in {\mathbb N}_0$, it holds that $\{\omega \in \Omega |T(\omega) \leq i \} \in {\cal F}_i$. 
Intuitively, $T$ returns the time step at which some stochastic process shows a desired behavior and should be “stopped”.

A discrete-time stochastic process  $\{X_n\}_{n=0}^{\infty}$ in ($\Omega, {\cal F}, {\mathbb P}$) is adapted to a filtration $\{{{\cal F}_n}\}_{n=0}^{\infty}$, if
for all $n \geq 0$, $X_n$ is a random variable in ($\Omega, {\cal F}_n, {\mathbb P}$).

\vskip 2pt 
\noindent\textbf{Stopped Process.} Let $\{X_n\}_{n=0}^{\infty}$ be a stochastic process adapted to a filtration $\{\mathcal{F}_n\}_{n=0}^\infty$ and let $T$ be a stopping time w.r.t. $\{\mathcal{F}_n\}_{n=0}^\infty$. The stopped process $\{\tilde{X}_n\}_{n=0}^\infty$ is defined by
$$
\tilde{X}_n=
\begin{cases}
X_n, & \text{ for } n<T, \\
X_T, & \text{ for } n\ge T.
\end{cases}
$$

\vskip 2pt 
\noindent\textbf{Martingales.}
A discrete-time stochastic process $\{X_n\}_{n=0}^\infty$ to a filtration $\{{{\cal F}_n}\}_{n=0}^{\infty}$ is a martingale (resp. supermartingale, submartingale) if for all $n \geq 0$, ${\mathbb E}[|X_n|] < \infty$ and it holds almost surely (i.e., with probability 1) that ${\mathbb E}[X_{n+1}|{\cal F}_n] = X_n$ (resp. ${\mathbb E}[X_{n+1}|{\cal F}_n] \leq X_n$, ${\mathbb E}[X_{n+1}|{\cal F}_n] \geq X_n$).

\vskip 2pt 
\noindent\textbf{Ranking Supermartingales.}
Let $T$ be a stopping time w.r.t. a filtration $\{{{\cal F}_n}\}_{n=0}^{\infty}$. 
A discrete-time stochastic process $\{X_n\}_{n=0}^\infty$ 
w.r.t. a stopping time $T$ 
is a ranking supermartingale (RSM) if for all $n \geq 0$, ${\mathbb E}[|X_n|] < \infty$ and there exists $\epsilon >0$ such that it holds almost surely (i.e., with probability 1) that $X_n \geq 0$ and ${\mathbb E}[X_{n+1}|{\cal F}_n] \leq X_n-\epsilon \cdot {\mathbb I}_{T>n}$.

\begin{theorem}[Hoeffding's Inequality on Supermartingales~\cite{hoeffding1994probability}]\label{thm:hoeffding}
     Let $\{X_n\}_{n\in\Nset_0}$ be a supermartingale w.r.t. a filtration $\{\mathcal{F}_n\}_{n\in\Nset_0}$, and $\{[a_n,b_n]\}_{n\in\Nset_0}$ be a sequence of non-empty intervals in $\Rset$. If $X_0$ is a constant random variable and $X_{n+1}-X_{n}\in [a_n,b_n]$ a.s. for all $n\in\Nset_0$, then 
 \[
 \probm(X_n-X_0\ge \lambda)\le \mathrm{exp}\left(-\frac{2\lambda^2}{\sum_{k=0}^{n-1}(b_k-a_k)^2} \right)
 \]
 for all $n\in\Nset_0$ and $\lambda>0$.
\end{theorem}

\begin{theorem}[Optional Stopping Theorem (OST)~\cite{williams1991}]\label{thm:OST}
   Consider a stopping time $T$ w.r.t. a filtration $\{\mathcal{F}_n\}_{n=0}^{\infty}$ and a martingale (resp. supermartingale, submartingale) $\{X_n\}_{n=0}^{\infty}$ adapted to  $\{\mathcal{F}_n\}_{n=0}^{\infty}$. Then $\expv [|X_T|]<\infty$ and $\expv[X_T]=\expv [X_0]$ (resp. $\expv[X_T]\le \expv[X_0]$, $\expv[X_T]\ge \expv[X_0]$) if one of the following conditions holds:
   \begin{itemize}
       \item $T$ is bounded, i.e., $T<c$ for some constant $c$;
       \item $\expv[T]<\infty$, and there exists a constant $c$ such that for all $n\in \Nset$, $|X_{n+1}-X_n|\le c$;
       \item The stopped process $\{\tilde{X}_n\}_{n=0}^{\infty}$ w.r.t. $T$ is bounded, i.e., there exists some constant $c$ such that $|\tilde{X}_n|\le c$ for all $n\in\Nset$.
   \end{itemize}
\end{theorem}

\begin{theorem}[Ville's Inequality~\cite{ville1939etude}]\label{thm:ville}
For any nonnegative supermartingale $\{X_n\}_{n=0}^\infty$ and any real number $c>0$,
\[
\probm [\mathrm{sup}_{n\ge 0} X_n\ge c]\le \frac{\expv[X_0]}{c}.
\]  
\end{theorem}
This theorem is also called Doob's nonnegative supermartingale inequality.

\begin{theorem}[Discrete Version of Gronwall's Inequality~\cite{gronwall1919note}]\label{thm:gronwall}
Consider a real number sequence $\{u_n\}_{n=0}^\infty$ such that 
\[
u_{n+1}\le a\cdot u_n +b\ \forall n\ge 0
\]
where $a>0$ and $b\in\Rset$. Then
\[
u_{n+1}\le a^n \cdot u_0+\sum_{k=1}^n a^{n-k}\cdot b\ \forall n\ge 0
\]
    
\end{theorem}

%% file: barrier-proofs.tex
Below we fix a DNN-controlled system $M_\mu=(S,S_0,A,\pi,f,R,\mu)$ with an unsafe set $S_u\in S$ such that $S_0\cap S_u=\emptyset$. We also define $\mathrm{1}_A(x)=1$ if $x\in A$ and $\mathrm{1}_A(x)=0$ if $x\not\in A$, where $x\in\Rset$ is a random variable and $A\subseteq \Rset$ is a set.

\subsection{Proof for \Cref{subsec:quali}}

\textbf{\Cref{thm:almost-safe} (Almost-Surely Safety).}
Given an $M_\mu$  with an initial set $S_0$ and an unsafe set $S_u$, if there exists a barrier certificate $B:S\rightarrow \Rset$ such that for some constant $\lambda\in (0,1]$, the following conditions hold:
     \begin{align*}
       & B(s)\le 0 \ &\text{for all }s\in S_0, & \tag{\ref{eq:thm-almost-initial}} \\
       & B(s)>0 \ & \text{for all } s\in S_u, & \tag{\ref{eq:thm-almost-unsafe}} \\
       & B(f(s,\pi(s+\delta)))-B(s)+\lambda\cdot B(s)\le 0  & \text{for all } (s,\delta)\in S\times W,  & \tag{\ref{eq:thm-almost-decrease}}
    \end{align*}
    then the system $M_\mu$ is almost-surely safe, i.e., $\forall s_0\in S_0. \omega\in\Omega_{s_0} \Longrightarrow \omega_t \not\in S_u \ \forall t\in\Nset. $

\begin{proof}
    We prove it by contradiction. Assume that there exists a barrier certificate $B$ satisfying conditions (\ref{eq:thm-almost-initial})-(\ref{eq:thm-almost-decrease}), but the system is unsafe, i.e., there are a time step $T>0$ and an initial state $s_0\in S_0$ such that a trajectory $\omega\in\Omega_{s_0}$ starting from $s_0$ satisfies $\omega_t \in S$ for all $t\in [0,T]$ and $\omega_T\in S_u$. Condition (\ref{eq:thm-almost-decrease}) implies that for any state $s\in S$ such that $B(s)\le 0$ and a noise $\delta \in W$, the value of $B$ at the next step is no more than zero, i.e., $B(f(s,\pi(s+\delta)))\le 0$. As a result, $B(\omega_T)$ must be no more than zero, which is contradictory to Condition (\ref{eq:thm-almost-unsafe}). Therefore, the system with a BC is almost-surely safe. \qed
\end{proof}

\subsection{Proofs for \Cref{subsec:quanti}}

\textbf{\Cref{thm:inf-lower-safe} (Lower Bounds on Infinite-time Safety).}
   Given an $M_\mu$  with an initial set $S_0$ and an unsafe set $S_u$, if there exists a barrier certificate $B:S\rightarrow \Rset$ such that for some constant $\epsilon\in [0,1]$,
    the following conditions hold:
     \begin{align*}
     & B(s)\ge 0 \ & \text{for all } s\in S,
       \tag{\ref{eq:thm-inf-nonnegative}}\\ 
    & B(s) \le \epsilon \ & \text{for all } s\in S_0,  \tag{\ref{eq:thm-inf-lower-initial}}\\
       & B(s)\ge 1 \ & \text{for all } s\in S_u,
       \tag{\ref{eq:thm-inf-lower-unsafe}} \\ 
       & \expv_{\delta\sim \mu}[B(f(s,\pi(s+\delta)))\mid s]-B(s)\le 0\  & \text{for all }s\in S \setminus S_u , 
       \tag{\ref{eq:thm-inf-lower-decrease}}
    \end{align*}
   then the safety probability over infinite time horizons is bounded from below by
    \begin{equation}
     \forall s-0\in S_0.   \probm_{s_0}\left[\{\omega\in\Omega_{s_0}\mid \omega_t\not\in S_u \ \text{for all } t\in \in \Nset  \} \right]\ge 1-B(s_0).
        \tag{\ref{eq:thm-inf-lowerbound}}
    \end{equation}

\begin{proof}
First, we construct a stochastic process $\{X_t\}_{t\ge 0}$ over the set $\Omega$ of all trajectories such that $X_t(\omega):=B(\omega_t)$ for any $\omega\in\Omega$. Let $\kappa$ be the first time that a trajectory enters the unsafe set $S_u$, i.e., $\kappa(\omega):=\mathrm{inf} \{t\in\Nset \mid \omega_t \in S_u\}$. Then we define the stopped process $\{\tilde{X}_t\}_{t\ge 0}$ by
$$
\tilde{X}_t=
\begin{cases}
X_t, & \text{ for } t<\kappa, \\
X_\kappa, & \text{ for } t\ge \kappa.
\end{cases}
$$
Next, we want to prove that $\{\tilde{X}_t\}_{t\ge 0}$ is a non-negative supermartingale. For $t<\kappa$, we have that
\begin{eqnarray*}
   \expv[\tilde{X}_{t+1}\mid \tilde{\mathcal{F}}_t] &= \expv[X_{t+1}\mid \mathcal{F}_t] \le X_t = \tilde{X}_t,
\end{eqnarray*}
where the inequality is obtained by Condition~(\ref{eq:thm-inf-lower-decrease}). For $t\ge \kappa$, we have that
\begin{align*}
    \expv[\tilde{X}_{t+1}\mid \tilde{\mathcal{F}}_t] &= X_\kappa=\tilde{X}_t. 
\end{align*}
Combined with Condition~(\ref{eq:thm-inf-nonnegative}), we can conclude that $\{\tilde{X}_t\}_{t\ge 0}$ is a non-negative supermartingale.
According to Condition (\ref{eq:thm-inf-lower-unsafe}), we have that $S_u\subseteq \{s\in S\mid B(s)\ge 1 \}$.  Therefore, for any initial state $s_0\in S_0$, it follows that
\begin{align*}
   & \probm_{s_0}\left[ \{ \omega\in\Omega_{s_0}\mid \omega_t\in S_u \text{ for some } t\in\Nset  \}  \right] \\
   \le \  & \probm_{s_0}\left[ \{  \omega\in\Omega_{s_0}\mid \mathrm{sup}_{t\in\Nset} X_t(\omega)\ge 1 \}  \right] \\
   = \ & \probm_{s_0}\left[ \{  \omega\in\Omega_{s_0}\mid \mathrm{sup}_{t\in\Nset} \tilde{X}_t(\omega)\ge 1 \}  \right] \\ \le \ & \tilde{X}_0 (\omega)= X_0(\omega) 
\end{align*}
where the last inequality is derived by Ville's Inequality (see \Cref{thm:ville}). 
Finally, by means of complementation, we have the lower bound of \Cref{eq:thm-inf-lowerbound}. \qed

\end{proof}

\textbf{\Cref{thm:inf-upper-safe} (Upper Bounds on Infinite-time Safety).}
Given an $M_\mu$  with an initial set $S_0$ and an unsafe set $S_u$, if there exists a barrier certificate $B:S\rightarrow \Rset$ such that for some constants $\gamma\in (0,1)$, $0\le \epsilon'<\epsilon\le 1$, the following conditions hold:
     \begin{align*}
     & 0\le B(s)\le 1 \ & \text{for all } s\in S,
       \tag{\ref{eq:thm-inf-upper-bound}}\\ 
       & B(s)\ge \epsilon \ &\text{ for all } s\in S_0, \tag{\ref{eq:thm-inf-upper-initial}} \\
         & B(s)\le \epsilon' \ &\text{ for all } s\in S_u,
         \tag{\ref{eq:thm-inf-upper-unsafe}}\\
       & B(s)-\gamma\cdot\expv_{\delta\sim \mu}[B(f(s,\pi(s+\delta)))\mid s]\le 0\  & \text{for all }s\in S \setminus S_u, 
       \tag{\ref{eq:thm-inf-upper-decrease}}
    \end{align*}
then the safety probability over infinite time horizons is bounded from above by
    \begin{equation}
     \forall s_0\in S_0.   \probm_{s_0}\left[\{ \omega\in\Omega_{s_0}\mid \omega_t\not\in S_u \ \text{for all } t\in \Nset  \} \right]\le 1-B(s_0).
        \tag{\ref{eq:thm-inf-upperbound}}
    \end{equation}

\begin{proof}
    Construct a stochastic process $\{X_t\}_{t\ge 0}$ over the set $\Omega$ of all trajectories such that $X_t(\omega):=B(s_t)$ for any $\omega\in\Omega$. Let $\kappa$ be a stopping time at which the state first enters $S_u$, i.e., $\kappa(\omega):=\mathrm{inf} \{t\in\Nset \mid \omega_t \in S_u\}$. Then construct a new stochastic process $\{Y_t\}_{t\ge 0}$ such that $Y_t:=\gamma^t X_t$.  The stopped process $\{\tilde{Y}_t\}_{t\ge 0}$ is defined by
$$
\tilde{Y}_t=
\begin{cases}
Y_t, & \text{ for } t<\kappa, \\
Y_\kappa, & \text{ for } t\ge \kappa.
\end{cases}
$$
Next, we want to prove that $\{\tilde{Y}_t\}_{t\ge 0}$ is a submartingale. For $t<\kappa$, we have that
\begin{align*}
   \expv[\tilde{Y}_{t+1}\mid \tilde{\mathcal{F}}_t] &= \expv[Y_{t+1}\mid \mathcal{F}_t] \\
   &= \expv[\gamma^{t+1}X_{t+1}\mid \mathcal{F}_t]=\gamma^{t+1} \expv [X_{t+1}\mid \mathcal{F}_t] \\
   &= \gamma^t\cdot \gamma\expv[X_{t+1}\mid \mathcal{F}_t] \ge \gamma^t X_t=\tilde{Y}_t
\end{align*}
where the inequality is obtained by Condition~(\ref{eq:thm-inf-upper-decrease}). For $t\ge \kappa$, we have that
\begin{align*}
    \expv[\tilde{Y}_{t+1}\mid \tilde{\mathcal{F}}_t]=Y_\kappa= \tilde{Y}_t
\end{align*}
By Condition~(\ref{eq:thm-inf-upper-bound}), $\tilde{Y}_t\in [0,1]$. Thus, $\{\tilde{Y}_t\}_{t\ge 0}$ is a submartingale satisfying the third condition of the Optional Stopping Theorem (OST) (see \Cref{thm:OST}). By applying OST, we obtain that 
\begin{eqnarray*}
X_0&=&\expv[\tilde{Y}_0]\le \expv[\tilde{Y}_\kappa]=\expv[\gamma^{\kappa} X_{\kappa}] \\
&=& \expv[\gamma^{\kappa}X_{\kappa}\mid \kappa<\infty] \probm_{s_0}(\kappa<\infty)+\expv[\gamma^{\kappa}X_{\kappa}\mid \kappa=\infty] \probm_{s_0}(\kappa=\infty)\\
&=&\expv[\gamma^{\kappa} X_\kappa\mid \kappa<\infty]\probm_{s_0}(\kappa<\infty)+0\le \probm_{s_0}(\kappa<\infty) \\
&=& \probm_{s_0}\left[\{\omega\in \Omega_{s_0}\mid \omega_t\in S_u \text{ for some } t\in \Nset \}\right]
\end{eqnarray*}
where the fourth equality is derived from the fact that $\gamma^\kappa X_\kappa=0$ for $\kappa=\infty$ as $\gamma\in (0,1)$ and the second inequality stems from the fact that $\gamma^\kappa X_\kappa\le 1$.
Finally, by means of complementation, we obtain the upper bound of \Cref{eq:thm-inf-upperbound}. \qed
\end{proof}

\subsection{Proofs for \Cref{subsec:quanti2}}

\textbf{\Cref{thm:fin-lower-safe-linear} (Linear Lower Bounds on Finite-time Safety).}
 Given an $M_\mu$  with an initial set $S_0$ and an unsafe set $S_u$, if there exists a barrier certificate $B:S\to \Rset$ such that for some constants $\lambda > \epsilon  \ge 0$ and $c \ge 0$, the following conditions hold:
        \begin{align*} 
        & B(s)\ge 0 \ & \text{ for all }s\in S, \tag{\ref{eq:thm-fin-lower-linear-nonnegative} }\\
       & B(s)\le \epsilon \ &\text{for all }s\in S_0, \tag{\ref{eq:thm-fin-lower-linear-initial}}\\
       & B(s)\ge \lambda \ & \text{for all } s\in S_u, \tag{\ref{eq:thm-fin-lower-linear-unsafe}}\\ 
       & \expv_{\delta\sim \mu}[B(f(s,\pi(s+\delta)))\mid s]-B(s)\le c\  & \text{for all }s\in S, \tag{\ref{eq:thm-fin-lower-linear-decrease}}
        \end{align*}
 then the safety probability over a finite time horizon $T$ is bounded from below by
 \begin{equation*}
     \forall s_0\in S_0.   \probm_{s_0}\left[\{\omega\mid \omega\in\Omega_{s_0},\omega_t\not\in S_u \ \text{for all } t\le T \} \right]\ge 1-(B(s_0)+cT)/\lambda . 
    \end{equation*}

\begin{proof}
     Construct a stochastic process $\{X_t\}_{t\ge 0}$ over the set $\Omega$ of all trajectories such that $X_t(\omega):=B(s_t)$ for any $\omega\in\Omega$. Let $\kappa$ be a stopping time at which the state first enters $S_u$, i.e., $\kappa(\omega):=\mathrm{inf} \{t\in\Nset \mid \omega_t \in S_u\}$. Then construct a new stochastic process $\{Y_t\}_{t\ge 0}$ such that $Y_t:=X_t+ (\kappa-t)c$. The stopped process $\{\tilde{Y}_t\}_{t\ge 0}$ is defined by
$$
\tilde{Y}_t=
\begin{cases}
Y_t, & \text{ for } t<\kappa, \\
Y_\kappa, & \text{ for } t\ge \kappa.
\end{cases}
$$
Next, we want to prove that $\{\tilde{Y}_t\}_{t\ge 0}$ is a non-negative supermartingale.
According to Condition (\ref{eq:thm-fin-lower-linear-nonnegative}), $\{\tilde{Y}_t\}_{t\ge 0}$ is non-neagtive. For $t<\kappa$, we have that
\begin{align*}
   \expv[\tilde{Y}_{t+1}\mid \tilde{\mathcal{F}}_t]&= \expv[Y_{t+1}\mid \mathcal{F}_t]=\expv[X_{t+1}+(\kappa-t-1)c\mid \mathcal{F}_t] \\
   &= \expv[X_{t+1}\mid \mathcal{F}_t]+ (\kappa-t)c-c\\
   &\le X_t+c+(\kappa-t)c-c\\
   &= X_t+(\kappa-t)c=Y_t=\tilde{Y}_t
\end{align*}
where the inequality is obtained by Condition (\ref{eq:thm-fin-lower-linear-decrease}). For $t\ge \kappa$, we have that
\begin{align*}
    \expv[\tilde{Y}_{t+1}\mid \tilde{\mathcal{F}}_t]=Y_\kappa=\tilde{Y}_t
\end{align*}
Therefore, we can conclude that $\{\tilde{Y}_t\}_{t\ge 0}$ is a non-negative supermartingale. By Ville's Inequality (\Cref{thm:ville}), we have that for a constant $\lambda'>0$,
\begin{align*}
    \probm[\mathrm{sup}_{t\in [0,\kappa]} Y_t\ge \lambda'] =\probm [\mathrm{sup}_{t\in\Nset} \tilde{Y}_t\ge \lambda']\le \frac{X_0+c \kappa}{\lambda'}
\end{align*}
By the definition, $Y_t\ge \lambda'\Leftrightarrow X_t\ge \lambda'-(\kappa-t)c:=r_t$, so we have that $\probm[\mathrm{sup}_{t\in [0,\kappa]} X_t \ge \mathrm{sup}_{t\in [0,\kappa]}r_t]\le \frac{X_0+c\kappa}{\lambda'}$. As the maximum of $r_t$ in $[0,\kappa]$ is $\lambda'$, let $\lambda=\lambda'$, we can derive that
\begin{align*}
\probm[\mathrm{sup}_{t\in [0,\kappa]} X_t\ge \lambda]\le \frac{X_0+c\kappa}{\lambda},
\end{align*}
which by Condition (\ref{eq:thm-fin-lower-linear-unsafe}) implies that
\begin{align*}
  \forall s_0\in S_0.  \probm_{s_0}[\{ \omega\in\Omega_{s_0}\mid \omega_t\in S_u \text{ for some } t\le \kappa \}]\le \frac{X_0(\omega)+c\kappa}{\lambda}
\end{align*}
Finally, by means of complementation, we obtain the lower bound of \Cref{eq:fin-lowerbound-linear}.  \qed
\end{proof}

\textbf{\Cref{thm:fin-lower-safe-exp} (Exponential Lower Bounds on Finite-time Safety).}
 Given an $M_\mu$  with an initial set $S_0$ and an unsafe set $S_u$, if there exists a function $B:S\rightarrow \Rset$ such that for some constants $\alpha> 0, \beta\in\Rset$,  
 and $\gamma\in [0,1)$, the following conditions hold:
     \begin{align*}
      & B(s)\ge 0 \ & \text{ for all }s\in S, \tag{\ref{eq:thm-fin-lower-exp-nonnegative} }\\
       & B(s)\le \gamma \ &\text{for all }s\in S_0, \tag{\ref{eq:thm-fin-lower-exp-initial} }\\
       & B(s) \ge 1 \ & \text{for all } s\in S_u, \tag{\ref{eq:thm-fin-lower-exp-unsafe} }\\ 
       & \alpha \expv_{\delta\sim \mu}[B(f(s,\pi(s+\delta)))\mid s]-B(s) \le \alpha \beta \  & \text{for all }s\in S\setminus S_u. \tag{\ref{eq:thm-fin-lower-exp-moni}}   
    \end{align*}
then the safety probability over a finite time horizon $T$ is bounded from below by
   \begin{equation*}
     \forall s_0\in S_0.   \probm_{s_0}\left[\{\omega\mid \omega\in\Omega_{s_0},\omega_t\not\in S_u \ \text{for all } t\le T \} \right]\ge 1-\frac{\alpha\beta}{\alpha-1} + (\frac{\alpha\beta}{\alpha-1}-B(s_0))\cdot \alpha^{-T}. \tag{\ref{eq:fin-lower-expbound}}
    \end{equation*}

\begin{proof}
Construct a stochastic process $\{X_t\}_{t\ge 0}$ over the set $\Omega$ of all trajectories such that $X_t(\omega):=\omega_t$ for any $\omega\in\Omega$. Let $\kappa$ be the first time that a trajectory enters the unsafe set $S_u$, i.e., $\kappa(\omega):=\mathrm{inf} \{t\in\Nset \mid \omega_t \in S_u\}$. Define the stopped process $\{\tilde{X}_t\}_{t\ge 0}$ by
$$
\tilde{X}_t=
\begin{cases}
X_t, & \text{ for } t<\kappa, \\
X_\kappa, & \text{ for } t\ge \kappa.
\end{cases}
$$
According to Condition (\ref{eq:thm-fin-lower-exp-moni}) and the law of total expectation, we have that 
\begin{align}
& \alpha \expv[B(\tilde{X}_{t+1})\mid \tilde{\mathcal{F}}_t]-B(\tilde{X}_t)\le \alpha\beta \nonumber \\
\Rightarrow \ & \expv[B(\tilde{X}_{t+1})\mid \tilde{\mathcal{F}}_t] \le \frac{B(\tilde{X}_t)}{\alpha}+\beta \nonumber \\
\Rightarrow \ &  \expv[B(\tilde{X}_{t+1})] \le \frac{\expv[B(\tilde{X}_t)]}{\alpha}+\beta \label{eq:gronwalltype}
\end{align}
Next, define a new stochastic process $\{Y_{t}\}_{t\ge 0}$ such that $Y_t:=\expv [B(\tilde{X}_t)]$. By~\Cref{eq:gronwalltype}, we obtain that $Y_{t+1}\le \frac{Y_t}{\alpha}+\beta$.
By the discrete version of Gronwall's Inequality (see \Cref{thm:gronwall}), we have that 
\begin{align}
    Y_{t} \le\ & \alpha^{-t}\cdot Y_0+\sum_{k=1}^t\left(\prod_{i=k+1}^t \alpha^{-1} \right)\cdot \beta \nonumber =\alpha^{-t}\cdot Y_0+\sum_{k=1}^t \left[\alpha^{-(t-k)}\cdot \beta\right] \nonumber\\
   =\ & \alpha^{-t}\cdot Y_0+ \alpha^{-t}\cdot \beta \cdot\sum_{k=1}^{t} \alpha^k \nonumber= \alpha^{-t}\cdot Y_0+ \alpha^{-t}\cdot \beta \cdot \frac{\alpha(1-\alpha^t)}{1-\alpha} \nonumber\\
   =\ & \alpha^{-t}\cdot Y_0+ \frac{\alpha\beta}{\alpha-1}-\frac{\alpha\beta}{\alpha-1}\cdot \alpha^{-t} \label{eq:gronwallbound}
\end{align}
Pick a time step $T>0$. It is observed that $\{\omega\in\Omega \mid X_t(\omega)\in S_u \text{ for some } t\le T \}=\{\omega\in \Omega\mid \tilde{X}_T\in S_u\}$. Thus, we have that
\begin{align*}
& \probm [X_t\in S_u \text{ for some }t\le T] \\
=&\probm [\tilde{X}_T\in S_u]=\expv [\mathrm{1}_{S_u} (\tilde{X}_T)] \\
\le &  \expv[ B(\tilde{X}_T)] \\
\le & \alpha^{-T}\cdot B(X_0)+ \frac{\alpha\beta}{\alpha-1}-\frac{\alpha\beta}{\alpha-1}\cdot \alpha^{-T}
\end{align*}
where the first inequality is obtained by Conditions (\ref{eq:thm-fin-lower-exp-nonnegative}) and (\ref{eq:thm-fin-lower-exp-unsafe}), the second inequality is derived by \Cref{eq:gronwallbound}.
Finally, by means of complementation, we obtain the lower bound. 
\end{proof}

\textbf{\Cref{thm:fin-upper-safe-linear} (Linear Upper Bounds on Finite-time Safety) }
 Given an $M_\mu$  with an initial set $S_0$ and an unsafe set $S_u$, if there exists a barrier function $B:S\rightarrow \Rset$ such that for some constants $\beta\in (0,1), \beta<\alpha<1+\beta, c\ge 0$, the following conditions hold:
     \begin{align*}
      & B(s)\ge 0 \ & \text{ for all }s\in S, \tag{\ref{eq:thm-fin-upper-linear-nonnegative}} \\
       & B(s)\le \beta \ &\text{for all }s\in S\setminus S_u, \tag{\ref{eq:thm-fin-upper-exp-initial} }\\
       & \alpha \le B(s) \le 1+\beta \ & \text{for all } s\in S_u, \tag{\ref{eq:thm-fin-upper-linear-unsafe} }\\ 
       & \expv_{\delta\sim \mu}[B(f(s,\pi(s+\delta)))\mid s]-B(s) \ge c \  & \text{for all }s\in S\setminus S_u. \tag{\ref{eq:thm-fin-upper-linear-moni}}   
    \end{align*}
then the safety probability over a finite time horizon is bounded from above by
     \begin{equation*}
      \forall s_0\in S_0.\   \probm_{s_0}\left[\{\omega\mid \omega\in\Omega_{s_0},\omega_t\not\in S_u \ \text{for all } t\le T \} \right]\le  1-B(s_0)-\frac{1}{2}c\cdot T+\beta.   
    \end{equation*}

\begin{proof}
Construct a stochastic process $\{X_t\}_{t\ge 0}$ over the set $\Omega$ of all trajectories such that $X_t(\omega):=\omega_t$ for any $\omega\in\Omega$. Let $\kappa$ be the first time that a trajectory enters the unsafe set $S_u$, i.e., $\kappa(\omega):=\mathrm{inf} \{t\in\Nset \mid \omega_t \in S_u\}$. Define the stopped process $\{\tilde{X}_t\}_{t\ge 0}$ by
$$
\tilde{X}_t=
\begin{cases}
X_t, & \text{ for } t<\kappa, \\
X_\kappa, & \text{ for } t\ge \kappa.
\end{cases}
$$
Pick a time step $T>0$. It is observed that $\{\omega\in\Omega \mid X_t(\omega)\in S_u \text{ for some } t\le T \}=\{\omega \in\Omega \mid \tilde{X}_T\in S_u\}$. Thus, we have that
\begin{align}
& \probm [X_t\in S_u \text{ for some }t\le T] \nonumber \\
=\ &\probm [\tilde{X}_T\in S_u]=\expv [\mathrm{1}_{S_u} (\tilde{X}_T)] \label{eq:stopped-observe1} 
\end{align}
Moreover, let $0\le T_1\le T_2\le T$, it is found that $\{\omega\in\Omega\mid \tilde{X}_{T_1}\in S_u\}\subseteq \{\omega \in\Omega \mid \tilde{X}_{T_2}\in S_u\}$. Therefore, we have that
\begin{align}\label{eq:stopped-observe2}
    \probm[\tilde{X}_{T_1}\in S_u]\le \probm[\tilde{X}_{T_2}\in S_u]
\end{align}
By Conditions (\ref{eq:thm-fin-upper-exp-initial}) and (\ref{eq:thm-fin-upper-linear-unsafe}), we have that 
\begin{align}\label{eq:newobserve}
 B(\tilde{X}_t)\le \mathrm{1}_{S_u}(\tilde{X}_t)+\beta \text{ for any } t\ge 0   
\end{align}
By Condition (\ref{eq:thm-fin-upper-linear-moni}) and the law of total expectation, we obtain that 
\begin{align*}
   & \expv[B(\tilde{X}_{t+1})\mid \tilde{\mathcal{F}}_t]-B(\tilde{X}_t)\ge c \\
   \Rightarrow \  & \expv[B(\tilde{X}_{t+1})\mid \tilde{\mathcal{F}}_t] \ge B(\tilde{X}_t) +c \\
   \Rightarrow \  & \expv[B(\tilde{X}_{t+1})] \ge \expv[B(\tilde{X}_t)] +c
\end{align*}
By the iterative calculation, we have that $\expv[B(\tilde{X}_{t})] \ge B(X_0)+c\cdot t$.
Then we can derive that
\begin{align*}
   \probm[\tilde{X}_T\in S_u]=&\expv [\mathrm{1}_{S_u}(\tilde{X}_T)] \\
   \ge & \frac{\sum_{t=0}^T \expv[\mathrm{1}_{S_u}(\tilde{X}_t)] }{T+1} \\
   \ge & \frac{\sum_{t=0}^T (\expv[B(\tilde{X}_t)]-\beta)}{T+1} \\
   \ge & \frac{\sum_{t=1}^T B(X_0)+c\cdot t}{T+1}-\beta \\
   = & B(X_0)+\frac{1}{2}c\cdot T-\beta
\end{align*}
where the first equality is obtained by \Cref{eq:stopped-observe1}, the first inequality stems from \Cref{eq:stopped-observe2} and the second inequality is obtained by~\Cref{eq:newobserve}. Finally, by means of complementation, we obtain the upper bound.
\qed

\end{proof}

\textbf{\Cref{thm:fin-upper-safe-exp} (Exponential Upper Bounds on Finite-Time Safety).}
  Given an $M_\mu$  with an initial set $S_0$ and an unsafe set $S_u$, if there exists a barrier function $B:S\to \Rset$ such that for some constants $K'\le K<0$, $\epsilon>0$ and a non-empty interval $[a,b]$, the following conditions hold:
         \begin{align*}
       & B(s)\ge 0 \ &\text{for all }s\in S\setminus S_u, \tag{\ref{eq:thm-fin-upper-exp-nonnegative}}\\
       & K'\le B(s) \le K \ & \text{for all } s\in S_u, \tag{\ref{eq:thm-fin-upper-exp-unsafe}}\\ 
       & \expv_{\delta\sim \mu}[B(f(s,\pi(s+\delta)))\mid s]-B(s)\le -\epsilon \  & \text{for all }s\in S\setminus S_u,  \tag{\ref{eq:thm-fin-upper-exp-moni}}\\
       &  a\le B(f(s,\pi(s+\delta)))-B(s) \le b \ & \text{for all } s\in S\setminus S_u \text{ and } \delta \sim \mu \tag{\ref{eq:thm-fin-upper-exp-diffbound}}, 
    \end{align*}
 then the safety probability over a finite time horizon $T$ is bounded from above by
    \begin{equation*}
      \forall s_0\in S_0.  \probm_{s_0}\left[\{\omega\mid \omega\in\Omega_{s_0},\omega_t\not\in S_u \ \text{for all } t\le T \} \right]\le exp(-\frac{2(\epsilon \cdot T- B(s_0))^2}{T\cdot (b-a)^2}). 
    \end{equation*}

\begin{proof}
Let $\kappa$ be a stopping time over the set $\Omega$ of all trajectories at which the state first enters $S_u$, i.e., $\kappa(\omega):=\mathrm{inf} \{t\in\Nset \mid \omega_t \in S_u\}$.  Define the stopped process $\{\tilde{X}_t\}_{t\ge 0}$ by
$$
X_t=
\begin{cases}
B(s_t), & \text{ for } t<\kappa, \\
B(s_\kappa), & \text{ for } t\ge \kappa.
\end{cases}
$$
Next, we prove that $\{X_t\}_{t\ge 0}$ is ranking supermartingale w.r.t. $\kappa$, i.e., ${\mathbb E}[X_{t+1}|{\cal F}_t] \leq X_t-\epsilon \cdot \mathrm{1}_{\kappa>t}$ for all $t\ge 0$. To prove this inequality, we consider two cases.
\begin{itemize}
      \item  When $\kappa>t$, we have that $X_t=B(s_t)$. By Condition (\ref{eq:thm-fin-upper-exp-moni}), we obtain that $\expv [X_{t+1}\mid \mathcal{F}_t]=\expv_{\delta\sim\mu} [B(f(s_t,\pi(s_t+\delta)))\mid s_t]\le B(s_t)-\epsilon=X_t-\epsilon$. 
      \item When $\kappa\le t$, we have that $X_t=B(s_\kappa)$. Then we can obtain that $\expv [X_{t+1}\mid \mathcal{F}_t]=X_\kappa=X_t$. 
  \end{itemize}  
  According to the two cases, we can conclude that ${\mathbb E}[X_{t+1}|{\cal F}_t] \leq X_t-\epsilon \cdot \mathrm{1}_{\kappa>t}$ for all $t \ge 0$.
  Hence, we prove that $\{X_t\}_{t\ge 0}$ is a ranking supermartingale w.r.t. $\kappa$.
  
Construct a new stochastic process $\{Y_t\}_{t\ge 0}$ such that $Y_t:= X_t+\epsilon\cdot \mathrm{min}\{\kappa,t\}$. Then we want to prove that $\{Y_t\}_{t\ge 0}$  is a supermatirngale and $Y_{t+1}-Y_t\in [a+\epsilon, b+\epsilon]$. First, we have that
  \begin{align*}
\expv [Y_{t+1}\mid \mathcal{F}_t] &= \expv [X_{t+1}+\epsilon \cdot \mathrm{min}\{t+1,\kappa\}\mid \mathcal{F}_t]  \\
&= \expv [X_{t+1} \mid \mathcal{F}_t ]+\epsilon \cdot \expv [\mathrm{min}\{t+1,\kappa\} 
 \mid\mathcal{F}_t]  \\
 &\le X_t-\epsilon\cdot \mathrm{1}_{\kappa> t} +  \epsilon \cdot \expv [\mathrm{min}\{t+1,\kappa\} \mid\mathcal{F}_t]  
\end{align*}
Second, we consider two cases:
\begin{itemize}
    \item If $\kappa>t$, then $\mathrm{min}\{t+1,\kappa\}=t+1$ and $\expv [\mathrm{min}\{t+1,\kappa\} \mid\mathcal{F}_t]=t+1$. We can thus derive that
    \begin{align*}
\expv [Y_{t+1}\mid \mathcal{F}_t] 
 &\le X_t-\epsilon\cdot \mathrm{1}_{\kappa>t} +  \epsilon \cdot \expv [\mathrm{min}\{t+1,\kappa\} \mid\mathcal{F}_t]  \\
 &= X_t-\epsilon+\epsilon\cdot (t+1)  \\
 &= X_t+\epsilon \cdot t \\
 &=  Y_t
\end{align*}
    \item If $\kappa\le t$, then $\mathrm{min}\{t+1,\kappa\}=\kappa=\kappa\cdot \mathrm{1}_{\kappa\le t} $. Since the event $\kappa\cdot \mathrm{1}_{\kappa\le t}$ is measurable in $\mathcal{F}_t$, $\expv [\kappa\cdot \mathrm{1}_{\kappa\le t}\mid \mathcal{F}_t]=\kappa\cdot \mathrm{1}_{\kappa\le t}$. We can thus derive that
    \begin{align*}
\expv [Y_{t+1}\mid \mathcal{F}_t] 
 &\le X_t-\epsilon\cdot \mathrm{1}_{\kappa> t} +  \epsilon \cdot \expv [\mathrm{min}\{t+1,\kappa\} \mid\mathcal{F}_t]  \\
 &= X_t-0+\epsilon\cdot \kappa  \\
 &=  Y_t
\end{align*}    
\end{itemize}
Hence, $\{Y_t\}_{t\ge 0}$ is a supermartingale. Moreover, since $\kappa\le t$ implies $X_{t+1}=X_t$, we have that $X_{t+1}-X_t=\mathrm{1}_{\kappa> t}\cdot (X_{t+1}-X_t)$. Thus, we can obtain that
\begin{align*}
    Y_{t+1}-Y_t &= X_{t+1}-X_t+ \epsilon \cdot (\mathrm{min}\{t+1,\kappa\}-\mathrm{min}\{t,\kappa \}) \\
    &= X_{t+1}-X_t+ \epsilon \cdot \mathrm{1}_{\kappa> t} \\
    &= \mathrm{1}_{\kappa> t}\cdot (X_{t+1}-X_t+ \epsilon)
\end{align*}
where the second equality is derived by the fact that $\mathrm{min}\{t+1,\kappa\}-\mathrm{min}\{t,\kappa \}=\mathrm{1}_{\kappa> t}$. It follows that $Y_{t+1}-Y_t \in [a+\epsilon, b+\epsilon]$.

We use Hoeffding's Inequality on Supermartingales (see \Cref{thm:hoeffding}) to prove our result. Pick a time step $T>0$. Let $\nu=\epsilon\cdot T-X_0$ and $\hat{\nu}=\epsilon\cdot\mathrm{min}\{T,\kappa\}-X_0$. Note that $\nu=\hat{\nu}$ whenever $\kappa>T$. By Conditions (\ref{eq:thm-fin-upper-exp-nonnegative}) and (\ref{eq:thm-fin-upper-exp-unsafe}), we have that $\kappa>T$ iff $X_T\ge 0$. For any $s_0\in S_0$, we can derive that
\begin{align*}
 \probm_{s_0}(\kappa>T)&=\probm_{s_0}(X_T\ge 0\wedge \kappa>T)  \\
 &=\probm_{s_0}((X_T+\hat{\nu}\ge \nu)\wedge \kappa>T) \\
 &\le \probm_{s_0}(X_T+\hat{\nu}\ge \nu) \\
 &= \probm_{s_0}(Y_T-Y_0\ge \epsilon\cdot T-X_0) \\
 &\le \mathrm{exp}\left(\frac{-2\cdot (\epsilon\cdot T-B(s_0))^2}{T\cdot (b-a)^2} \right) 
\end{align*}
for all $T>\frac{B(s_0)}{\epsilon}$.
Finally, as $\{\omega\mid \kappa(\omega)>T \}=\{\omega\mid \omega_t\not\in S_u \text{ for all } t\le T \}$, we have that
\begin{eqnarray*}
   \probm_{s_0}(\kappa>T)=\probm_{s_0}\left[\{ \omega\in\Omega_{s_0}\mid \omega_t\not\in S_u \ \text{for all } t\le T \} \right], 
\end{eqnarray*}
which implies the upper bound. 
\qed
\end{proof}

%% file: k-induction-proofs.tex
Fix a DNN-controlled system $M_\mu=(S,S_0,A,\pi,f,R,\mu)$ with an unsafe set $S_u\in S$ such that $S_0\cap S_u=\emptyset$.

\subsection{Proof for \Cref{subsec:k-induc-almost}}

\textbf{\Cref{thm:almost-safe-induc} ($k$-inductive Variant of Almost-Sure Safety).}
     Given an $M_\mu$  with an initial set $S_0$ and an unsafe set $S_u$.
    If there exists a $k$-inductive barrier certificate $B:S\rightarrow \Rset$ such that the following conditions hold:
     \begin{align*}
       & \bigwedge_{0\le i<k} B(g_{\pi,f}^{i}(s,\Delta^i))\le 0 \ & \forall (s,\Delta^i)\in S_0\times W^i, \tag{\ref{eq:thm-almost-indc-initial}} \\
       & B(s)>0 \ & \forall s\in S_u,  \tag{\ref{eq:thm-almost-indc-unsafe}} \\
       &\bigwedge_{0\le i<k} (B(g_{\pi,f}^{i}(s,\Delta^i))\le 0)\Longrightarrow B(g_{\pi,f}^k(s,\Delta^k))\le 0  & \forall (s,\Delta^i)\in S\times W^i,  \tag{\ref{eq:thm-almost-indc-decrease} }
    \end{align*}
    then the system $M_\mu$ is almost-surely safe, i.e., $\forall s_0\in S_0. \omega\in\Omega_{s_0} \Longrightarrow \omega_t \not\in S_u \ \forall t\in\Nset.$

\begin{proof}
    We prove it by contradiction. Assume that there exists a $k$-inductive barrier certificate $B$ satisfying Conditions (\ref{eq:thm-almost-indc-initial}) to (\ref{eq:thm-almost-indc-decrease}), but the system is unsafe, i.e., there exist a time step $T>0$ and an initial state $s_0\in S_0$ such that a trajectory starting from $s_0$ enters the unsafe set $S_u$ at $T$, which means $s_T\in S_u$. By Condition (\ref{eq:thm-almost-indc-unsafe}), we have that $B(s_T)>0$. Condition (\ref{eq:thm-almost-indc-initial}) implies that $B(s_0)\le 0, B(s_1)\le 0, \dots, B(s_{k-1})\le 0$, which means the state sequences starting from the safe set will remain in the safe set for the next $k-1$ consecutive time steps. By Condition (\ref{eq:thm-almost-indc-decrease}), we have that for any $k$ consecutive time steps, if the system is safe, then the system will still be safe in the $(k+1)$-th time step. Then by applying the $k$ induction principle with Condition (\ref{eq:thm-almost-indc-initial}) as the base case and Condition (\ref{eq:thm-almost-indc-decrease}) as the inductive hypothesis, we can derive that $B(s_t)\le 0$ for all $t\in\Nset$. This is contradictory to Condition (\ref{eq:thm-almost-indc-unsafe}). Therefore, the system with a $k$-inductive BC is almost-surely safe. \qed
\end{proof}

\subsection{Proofs for \Cref{subsec:induc-quanti}}

\textbf{\Cref{thm:k-ind-lower} ($k$-inductive Lower Bounds on Unbounded-time Safety).}
   Given an $M_\mu$  with an initial set $S_0$ and an unsafe set $S_u$.
  If there exists a k-inductive barrier certificate $B:S\to \Rset$ 
  such that for some constants $k\in\Nset_{\ge 1}$, $\epsilon\in [0,1]$ and $c\ge 0$, the following conditions hold:
     \begin{align*}
     & B(s) \ge 0 \ & \text{for all } s\in S \tag{\ref{eq:thm-inf-lower-induc-nonnegative}} \\
       & B(s)\le \epsilon \ &\text{for all }s\in S_0, \tag{\ref{eq:thm-inf-lower-induc-initial}} \\
       & B(s)\ge 1 \ & \text{for all } s\in S_u, \tag{\ref{eq:thm-inf-lower-induc-unsafe}} \\ 
       & \expv_{\delta\sim\mu}[B(f(s,\pi(s+\delta)))\mid s]-B(s)\le c \  & \text{for all } s\in  S , \tag{\ref{eq:thm-inf-lower-induc-c-martin}} \\
       & \expv_{\Delta^k\sim \mu^k}[B(g_{\pi,f}^k(s,\Delta^k))\mid s]-B(s)\le 0 \  & \text{for all } s\in  S, \tag{\ref{eq:thm-inf-lower-induc-k-martin}}
    \end{align*}  
then the safety probability over infinite time horizons is bounded from below by
   \begin{equation*}
    \forall s_0\in S_0.    \probm_{s_0}\left[\{\omega\mid \omega_0\in\Omega_{s_0},\omega_t\not\in S_u \ \text{ for all } t\in \Nset  \} \right]\ge 1-kB(s_0)-\frac{k(k-1)c}{2}. 
    \end{equation*}

\begin{proof}
Let $\kappa$ be the first time that a trajectory enters the unsafe set $S_u$, i.e., $\kappa(\omega):=\mathrm{inf} \{t\in\Nset \mid \omega_t \in S_u\}$. Construct a stochastic process $\{X_t\}_{t\ge 0}$ over the set $\Omega$ of all trajectories such that 
for any $\omega\in\Omega$.  
$$
X_t(\omega)=
\begin{cases}
B(s_t), & \text{ for } t<\kappa, \\
B(s_\kappa), & \text{ for } t\ge \kappa.
\end{cases}
$$
According to Condition (\ref{eq:thm-inf-lower-induc-unsafe}), we have that $S_u\subseteq \{s\in S\mid B(s)\ge 1 \}$. Therefore, for any initial state $s_0\in S_0$, it follows that
\begin{align*}
  & \probm_{s_0}[\{\omega\mid \omega\in \Omega_{s_0}, \omega_t\in S_u \text{ for some } t\in \Nset\}] \\
  \le\  & \probm_{s_0}[\{\omega\mid \omega\in\Omega_{s_0}, \mathrm{sup}_{t\in\Nset} X_t(\omega)\ge 1\}] 
\end{align*}
Then construct $k$ new stochastic processes $\{Y^i_t\}_{t\ge 0}$'s for all $i\in [0,k-1]$ such that $Y^i_t:=X_{i+tk}$. Intuitively, given a process $\{X_t\}_{t\ge 0}$, each new process $\{Y^i_t\}_{t\ge 0}$ consists of the states in $\{X_t\}_{t\ge 0}$ after every $k$ steps, starting from $X_i$. Next, we want to prove that each $\{Y^i_t\}_{t\ge 0}$ is a non-negative supermartingale. For $i+tk<\kappa$, we have that
\begin{align*}
    \expv[Y^i_{t+1}\mid \mathcal{F}_t]=\expv[X_{i+(t+1)k}\mid \mathcal{F}_t] \le X_{i+tk}=Y^i_t
\end{align*}
where the inequality is obtained by Condition (\ref{eq:thm-inf-lower-induc-k-martin}). For $i+tk\ge \kappa$, we have that
\begin{align*}
    \expv[Y^i_{t+1}\mid \mathcal{F}_t]= X_{i+tk}=Y^i_t.
\end{align*}
Combined with Condition (\ref{eq:thm-inf-lower-induc-nonnegative}), we can conclude that $\{Y^i_t\}_{t\ge 0}$ is a non-negative supermartingale for all $i\in [0,k-1]$. Then by Boole's inequality, we have that
\begin{align*}
   & \probm_{s_0}[\{ \omega\in\Omega_{s_0}\mid \mathrm{sup}_{t\in\Nset} X_t(\omega)\ge 1\}]  \\
   \le\ & \sum_{i=0}^{k-1}  \probm_{s_0}[\{ \omega\in\Omega_{s_0}\mid \mathrm{sup}_{t\in\Nset} X_{i+tk}(\omega)\ge 1\}]   \\
   = \ & \sum_{i=0}^{k-1}  \probm_{s_0}[\{ \omega\in\Omega_{s_0}\mid \mathrm{sup}_{t\in\Nset} Y^i_t(\omega)\ge 1\}] \\
   \le \ & \sum_{i=0}^{k-1} \expv [X_i]
\end{align*}
where the last inequality is obtained by Doob's nonnegative supermartingale inequality. 
By applying Condition (\ref{eq:thm-inf-lower-induc-c-martin}) and the law of total expectation recursively, we obtain that 
\begin{align*}
  &  \expv[X_{t+1}\mid \mathcal{F}_t]\le X_t+c \\
  \Rightarrow \ &  \expv[X_{t+1}]\le \expv[X_t]+c \\
  \Rightarrow \ & \expv[X_{t+1}] \le X_0+c\cdot (t+1)
\end{align*}
Therefore, we can derive that 
\begin{align*}
& \probm_{s_0}[\{ \omega\in \Omega_{s_0}\mid \omega_t\in S_u \text{ for some } t\in \Nset\}] \\
\le\ & X_0+\sum_{i=1}^{k-1} (X_0+ic) = kX_0+\frac{k(k-1)c}{2} 
\end{align*}
Finally, by means of complementation, we obtain the upper bound.
\qed
\end{proof}

\begin{remark}
    To make the probabilistic bound above
    non-trivial, the value of $k$ should be bounded by
    \begin{equation*}
        1\le k \le \frac{(c-2B(s_0))+\sqrt{4B(s_0)^2+c^2-4c(2+B(s_0))^2}}{2c}
    \end{equation*}
\end{remark}
\textbf{\Cref{thm:k-induc-upper} ($k$-inductive Upper Bounds on Unbounded-time Safety).}
Given an $M_\mu$  with an initial set $S_0$ and an unsafe set $S_u$, if there exists a barrier certificate $B:S\rightarrow \Rset$ such that for some constant $\gamma\in (0,1)$, $0\le \epsilon'<\epsilon\le 1$, $c\le 0$ the following conditions hold:
     \begin{align*}
     & 0\le B(s)\le 1 \ & \text{for all } s\in S
       \tag{\ref{eq:thm-inf-upper-induc-bound}}\\ 
       & B(s)\ge \epsilon \ & \text{ for all } s\in S_0, \tag{\ref{eq:thm-inf-upper-induc-initial}}\\
         & B(s)\le \epsilon' \ & \text{ for all } s\in S_u, \tag{\ref{eq:thm-inf-upper-induc-unsafe}}\\
        &  \expv_{\delta\sim \mu}[B(f(s,\pi(s+\delta)))\mid s]-B(s)\ge c \ 
 & \text{ for all } s\in S \tag{\ref{eq:thm-inf-upper-induc-c-martin} }\\
       & B(s)-\gamma^k \cdot\expv_{\Delta^k\sim \mu^k}[B(g_{\pi,f}^k(s,\Delta^k))\mid s]\le 0\  & \text{for all }s\in S \setminus S_u 
       \tag{\ref{eq:thm-inf-upper-induc-decrease}}
    \end{align*}
then the safety probability over infinite time horizons is bounded from above by
  \begin{equation*}
       \forall s_0\in S_0. \probm_{s_0}\left[\{ \omega\in\Omega_{s_0}\mid\omega_t\not\in X_u \ \text{ for all } t\in \Nset  \} \right]\le 1-k B(s_0)-\frac{k(k-1)c}{2}.  
    \end{equation*}

\begin{proof}
Construct a stochastic process $\{X_t\}_{t\ge 0}$ over the set $\Omega$ of all trajectories such that $X_t(\omega):=B(s_t)$ for any $\omega\in\Omega$. Let $\kappa$ be a stopping time at which the state first enters $S_u$, i.e., $\kappa(\omega):=\mathrm{inf} \{t\in\Nset \mid \omega_t \in S_u\}$. Then construct a new stochastic process $\{Y_t\}_{t\ge 0}$ stopped at $\kappa$ such that 
$$
Y_t=
\begin{cases}
\gamma^t X_t, & \text{ for } t<\kappa, \\
\gamma^\kappa X_\kappa, & \text{ for } t\ge \kappa.
\end{cases}
$$
It is observed that $\{\omega\in\Omega \mid X_t(\omega)\in S_u \text{ for some } t\in\Nset \}=\{\omega \in\Omega \mid \kappa(\omega)<\infty\}$. Thus, we have that
\begin{align}
 \probm [X_t\in S_u \text{ for some }t\in\Nset] =\probm [\kappa<\infty] \label{eq:stopped-observe-new}
\end{align}
Moreover, we can derive that
\begin{align*}
   \probm [\kappa<\infty] \ge\  &\expv[\gamma^\kappa X_{\kappa}\mid \kappa<\infty]\probm[\kappa<\infty] +0 \\
   = \ & \expv[\gamma^\kappa X_{\kappa}\mid \kappa<\infty] \probm[\kappa<\infty]+\expv[\gamma^\kappa X_\kappa\mid \kappa=\infty] \probm[\kappa=\infty] \\
   = \ & \expv[\gamma^\kappa X_\kappa] =\expv[Y_\kappa]
\end{align*}
where the first inequality is obtained by the fact $\gamma\in (0,1)$ and Condition (\ref{eq:thm-inf-upper-induc-bound}), and the second equality holds as $\gamma^\kappa=0$ when $\kappa=\infty$.

Next, we construct $k$ new stochastic processes $\{Z^i_t\}_{t\ge 0}$'s for $i \in [0,1]$ such that
$Z^i_t:=Y_{i+tk}$. Intuitively, given a process $\{Y_t\}_{t\ge 0}$, each new process $\{Z^i_t\}_{t\ge 0}$ consists of the states in $\{Y_t\}_{t\ge 0}$ after every $k$ steps, starting from $Y_i$. We want to prove that each $\{Z_t\}_{t\ge 0}$ is a submartingale. For $i+tk<\kappa$, we have that
\begin{align*}
    \expv[Z^i_{t+1}\mid \mathcal{F}_t]=\ & \expv[\gamma^{i+(t+1)k} X_{i+(t+1)k} \mid \mathcal{F}_t] \\
    =\ &\gamma^{i+tk}\cdot 
    \gamma^k \expv[X_{i+(t+1)k}\mid \mathcal{F}_t] \ge \gamma^{i+tk} X_{i+tk} =Z^i_t
\end{align*}
where the inequality is obtained by Condition (\ref{eq:thm-inf-upper-induc-decrease}). 
For $i+tk\ge \kappa$, we have that
\begin{align*}
    \expv[Z^i_{t+1}\mid \mathcal{F}_t]=Y_{i+tk}=Z^i_t
\end{align*}
Hence, $\{Z^i_t\}_{t\ge 0}$ is a submartingale for all $i\in [0,k-1]$. By Condition (\ref{eq:thm-inf-upper-induc-bound}), $\{Z^i_t\}_{t\ge 0}$ satisfies the conditions of the Optional Stopping Theorem (OST). Then by applying OST, we obtain that
\begin{align*}
    \probm[\kappa<\infty]\ge \expv[Y_\kappa]=\sum_{i=0}^{k-1} \expv[Z^i_\kappa]\ge \sum_{i=0}^{k-1} \expv [Z_0^i]=\sum_{i=0}^{k-1} \expv [X_i]
\end{align*}
By applying Condition (\ref{eq:thm-inf-upper-induc-c-martin}) and the law of total expectation recursively, we have that
\begin{align*}
 &\expv[X_{t+1}\mid\mathcal{F}_t]\ge X_t+c \\
 \Rightarrow \ & \expv[X_{t+1}]\ge \expv[X_t]+c \\
 \Rightarrow \ & \expv[X_{t+1}]\ge X_0 + (t+1)c
\end{align*}
Therefore, we have that
\begin{align*}
\probm[\kappa<\infty]\ge X_0+\sum_{i=1}^{k-1} (X_0+ic)= kX_0+\frac{k(k-1)c}{2}
\end{align*}
Finally, by means of complementation, we obtain the upper bound.
\qed
\end{proof}

%% file: main.bbl
\begin{thebibliography}{10}
\providecommand{\url}[1]{\texttt{#1}}
\providecommand{\urlprefix}{URL }
\providecommand{\doi}[1]{https://doi.org/#1}

\bibitem{AbateAEGP21}
Abate, A., Ahmed, D., Edwards, A., Giacobbe, M., Peruffo, A.: {FOSSIL:} a
  software tool for the formal synthesis of lyapunov functions and barrier
  certificates using neural networks. In: HSCC'21. pp. 24:1--24:11. {ACM}
  (2021)

\bibitem{AbateDKKP18}
Abate, A., David, C., Kesseli, P., Kroening, D., Polgreen, E.: Counterexample
  guided inductive synthesis modulo theories. In: CAV'18. pp. 270--288.
  Springer (2018)

\bibitem{amir2021towards}
Amir, G., Schapira, M., Katz, G.: Towards scalable verification of deep
  reinforcement learning. In: FMCAD'21. pp. 193--203. IEEE (2021)

\bibitem{AnandM0Z22}
Anand, M., Murali, V., Trivedi, A., Zamani, M.: k-inductive barrier
  certificates for stochastic systems. In: {HSCC}'22. pp. 12:1--12:11. {ACM}
  (2022)

\bibitem{AsadiC0GM21}
Asadi, A., Chatterjee, K., Fu, H., Goharshady, A.K., Mahdavi, M.: Polynomial
  reachability witnesses via stellens{\"{a}}tze. In: {PLDI}'21. pp. 772--787.
  {ACM} (2021)

\bibitem{bacci2021verifying}
Bacci, E., Giacobbe, M., Parker, D.: Verifying reinforcement learning up to
  infinity. In: IJCAI'21. pp. 2154--2160 (2021)

\bibitem{bacci2020probabilistic}
Bacci, E., Parker, D.: Probabilistic guarantees for safe deep reinforcement
  learning. In: FORMATS'20. pp. 231--248. Springer (2020)

\bibitem{bacci2022verified}
Bacci, E., Parker, D.: Verified probabilistic policies for deep reinforcement
  learning. In: NFM'22. pp. 193--212. Springer (2022)

\bibitem{Stability_Guarantees}
Berkenkamp, F., Turchetta, M., Schoellig, A.P., Krause, A.: Safe model-based
  reinforcement learning with stability guarantees. In: NeurIPS. pp. 908--918
  (2017)

\bibitem{DBLP:conf/sas/Brain0KS15}
Brain, M., Joshi, S., Kroening, D., Schrammel, P.: Safety verification and
  refutation by k-invariants and k-induction. In: SAS'15. pp. 145--161.
  Springer (2015)

\bibitem{GYM}
Brockman, G., Cheung, V., Pettersson, L., Schneider, J., Schulman, J., Tang,
  J., Zaremba, W.: {OpenAI Gym} (2016), arXiv:1606.01540

\bibitem{calinescu2012self}
Calinescu, R., Ghezzi, C., Kwiatkowska, M., Mirandola, R.: Self-adaptive
  software needs quantitative verification at runtime. Communications of the
  ACM  \textbf{55}(9),  69--77 (2012)

\bibitem{carr2021task}
Carr, S., Jansen, N., Topcu, U.: Task-aware verifiable {RNN}-based policies for
  partially observable markov decision processes. Artif. Intell. Res.
  \textbf{72},  819--847 (2021)

\bibitem{ChakarovS13}
Chakarov, A., Sankaranarayanan, S.: Probabilistic program analysis with
  martingales. In: CAV'13. pp. 511--526. Springer (2013)

\bibitem{DBLP:conf/popl/ChatterjeeFNH16}
Chatterjee, K., Fu, H., Novotn{\'{y}}, P., Hasheminezhad, R.: Algorithmic
  analysis of qualitative and quantitative termination problems for affine
  probabilistic programs. In: {POPL'16}. pp. 327--342. {ACM} (2016)

\bibitem{dawson2023safe}
Dawson, C., Gao, S., Fan, C.: Safe control with learned certificates: A survey
  of neural lyapunov, barrier, and contraction methods for robotics and
  control. IEEE Trans. Robot.  \textbf{39},  1749 -- 1767 (2023)

\bibitem{deshmukh2019learning}
Deshmukh, J., Kapinski, J., Yamaguchi, T., Prokhorov, D.: Learning deep neural
  network controllers for dynamical systems with safety guarantees. In:
  ICCAD'19. pp.~1--7. IEEE (2019)

\bibitem{DBLP:conf/sas/DonaldsonHKR11}
Donaldson, A.F., Haller, L., Kroening, D., R{\"{u}}mmer, P.: Software
  verification using k-induction. In: SAS'11. pp. 351--368. Springer (2011)

\bibitem{DBLP:journals/pacmpl/FengCSKKZ23}
Feng, S., Chen, M., Su, H., Kaminski, B.L., Katoen, J., Zhan, N.: Lower bounds
  for possibly divergent probabilistic programs. Proc. {ACM} Program. Lang.
  \textbf{7}({OOPSLA1}),  696--726 (2023)

\bibitem{DBLP:conf/cav/FengC00Z20}
Feng, S., Chen, M., Xue, B., Sankaranarayanan, S., Zhan, N.: Unbounded-time
  safety verification of stochastic differential dynamics. In: CAV'20. pp.
  327--348. Springer (2020)

\bibitem{Interval_Bound_Propagation}
Gowal, S., Dvijotham, K., Stanforth, R., Bunel, R., Qin, C., Uesato, J.,
  Arandjelovic, R., Mann, T.A., Kohli, P.: On the effectiveness of interval
  bound propagation for training verifiably robust models. CoRR
  \textbf{abs/1810.12715} (2018)

\bibitem{gronwall1919note}
Gronwall, T.H.: Note on the derivatives with respect to a parameter of the
  solutions of a system of differential equations. Annals of Mathematics pp.
  292--296 (1919)

\bibitem{hahn20192019}
Hahn, E.M., Hartmanns, A., Hensel, C., Klauck, M., Klein, J.,
  K{\v{r}}et{\'\i}nsk{\`y}, J., Parker, D., Quatmann, T., Ruijters, E.,
  Steinmetz, M.: The 2019 comparison of tools for the analysis of quantitative
  formal models: ({QComp} 2019 competition report). In: TACAS'19. pp. 69--92.
  Springer (2019)

\bibitem{verify_communication_protocols}
Hamers, R., Jongmans, S.: Discourje: Runtime verification of communication
  protocols in clojure. In: TACAS'20. pp. 266--284 (2020)

\bibitem{hensel2021probabilistic}
Hensel, C., Junges, S., Katoen, J.P., Quatmann, T., Volk, M.: The probabilistic
  model checker storm. International Journal on Software Tools for Technology
  Transfer pp. 1--22 (2021)

\bibitem{hoeffding1994probability}
Hoeffding, W.: Probability inequalities for sums of bounded random variables.
  The collected works of Wassily Hoeffding pp. 409--426 (1994)

\bibitem{huang2022polar}
Huang, C., Fan, J., Chen, X., Li, W., Zhu, Q.: Polar: A polynomial arithmetic
  framework for verifying neural-network controlled systems. In: ATVA'22. pp.
  414--430. Springer (2022)

\bibitem{ivanov2021verisig}
Ivanov, R., Carpenter, T., Weimer, J., Alur, R., Pappas, G., Lee, I.: Verisig
  2.0: Verification of neural network controllers using taylor model
  preconditioning. In: CAV'21. pp. 249--262. Springer (2021)

\bibitem{jin2022trainify}
Jin, P., Tian, J., Zhi, D., et~al.: Trainify: A {CEGAR}-driven training and
  verification framework for safe deep reinforcement learning. In: CAV'22. pp.
  193--218. Springer (2022)

\bibitem{kwiatkowska2011prism}
Kwiatkowska, M., Norman, G., Parker, D.: Prism 4.0: Verification of
  probabilistic real-time systems. In: CAV'11. pp. 585--591. Springer (2011)

\bibitem{kwiatkowska2022probabilistic}
Kwiatkowska, M., Norman, G., Parker, D.: Probabilistic model checking and
  autonomy. Annu. Rev. Control Robot. Auton. Syst.  \textbf{5},  385--410
  (2022)

\bibitem{lavaei2022safety}
Lavaei, A., Soudjani, S., Frazzoli, E.: Safety barrier certificates for
  stochastic hybrid systems. In: ACC'22. pp. 880--885. IEEE (2022)

\bibitem{stab_martingales}
Lechner, M., Zikelic, D., Chatterjee, K., Henzinger, T.A.: Stability
  verification in stochastic control systems via neural network
  supermartingales. In: AAAI'22. pp. 7326--7336 (2022)

\bibitem{lillicrap2015continuous}
Lillicrap, T., Hunt, J., Pritzel, A., Heess, N., Erez, T., Tassa, Y., Silver,
  D., Wierstra, D.: Continuous control with deep reinforcement learning. CoRR
  \textbf{abs/1509.02971} (2015)

\bibitem{DDPG}
Lillicrap, T., Hunt, J., Pritzel, A., Heess, N., Erez, T., Tassa, Y., Silver,
  D., Wierstra, D.: Continuous control with deep reinforcement learning. In:
  ICLR (2016)

\bibitem{mathiesen2022safety}
Mathiesen, F.B., Calvert, S.C., Laurenti, L.: Safety certification for
  stochastic systems via neural barrier functions. IEEE Control Syst. Lett.
  \textbf{7},  973--978 (2022)

\bibitem{meng2021reactive}
Meng, Y., Qin, Z., Fan, C.: Reactive and safe road user simulations using
  neural barrier certificates. In: IROS'21. pp. 6299--6306. IEEE (2021)

\bibitem{DQN}
Mnih, V., Kavukcuoglu, K., Silver, D., Graves, A., Antonoglou, I., Wierstra,
  D., Riedmiller, M.A.: Playing atari with deep reinforcement learning. CoRR
  \textbf{abs/1312.5602} (2013)

\bibitem{PeruffoAA21}
Peruffo, A., Ahmed, D., Abate, A.: Automated and formal synthesis of neural
  barrier certificates for dynamical models. In: TACAS'21. pp. 370--388.
  Springer (2021)

\bibitem{PrajnaJ04}
Prajna, S., Jadbabaie, A.: Safety verification of hybrid systems using barrier
  certificates. In: Alur, R., Pappas, G.J. (eds.) {HSCC}'04. pp. 477--492.
  Springer (2004)

\bibitem{prajna2007framework}
Prajna, S., Jadbabaie, A., Pappas, G.J.: A framework for worst-case and
  stochastic safety verification using barrier certificates. IEEE Trans.
  Automat. Contr.  \textbf{52}(8),  1415--1428 (2007)

\bibitem{prajna2005necessity}
Prajna, S., Rantzer, A.: On the necessity of barrier certificates. IFAC
  Proceedings Volumes  \textbf{38}(1),  526--531 (2005)

\bibitem{RuanHK18}
Ruan, W., Huang, X., Kwiatkowska, M.: Reachability analysis of deep neural
  networks with provable guarantees. In: IJCAI'18. pp. 2651--2659 (2018)

\bibitem{SalamatiLSZ21}
Salamati, A., Lavaei, A., Soudjani, S., Zamani, M.: Data-driven safety
  verification of stochastic systems via barrier certificates. In: {ADHS}'21.
  pp. 7--12. Elsevier (2021)

\bibitem{samek2021explaining}
Samek, W., Montavon, G., Lapuschkin, S., et~al.: Explaining deep neural
  networks and beyond: A review of methods and applications. Proc. of the IEEE
  \textbf{109}(3),  247--278 (2021)

\bibitem{seshia2022toward}
Seshia, S.A., Sadigh, D., Sastry, S.S.: Toward verified artificial
  intelligence. Communications of the ACM  \textbf{65}(7),  46--55 (2022)

\bibitem{sha2021synthesizing}
Sha, M., Chen, X., Ji, Y., Zhao, Q., Yang, Z., Lin, W., Tang, E., Chen, Q., Li,
  X.: Synthesizing barrier certificates of neural network controlled continuous
  systems via approximations. In: DAC'21. pp. 631--636. IEEE (2021)

\bibitem{SteinhardtT12}
Steinhardt, J., Tedrake, R.: Finite-time regional verification of stochastic
  non-linear systems. Int. J. Robotics Res.  \textbf{31}(7),  901--923 (2012)

\bibitem{TakisakaOUH18}
Takisaka, T., Oyabu, Y., Urabe, N., Hasuo, I.: Ranking and repulsing
  supermartingales for reachability in probabilistic programs. In: {ATVA}'18.
  pp. 476--493. Springer (2018)

\bibitem{TakisakaOUH21}
Takisaka, T., Oyabu, Y., Urabe, N., Hasuo, I.: Ranking and repulsing
  supermartingales for reachability in randomized programs. {ACM} Trans. Prog.
  Lang. Syst.  \textbf{43}(2),  5:1--5:46 (2021)

\bibitem{tschaikowski2014tackling}
Tschaikowski, M., Tribastone, M.: Tackling continuous state-space explosion in
  a markovian process algebra. Theoretical Computer Science  \textbf{517},
  1--33 (2014)

\bibitem{DBLP:conf/lics/UrabeHH17}
Urabe, N., Hara, M., Hasuo, I.: Categorical liveness checking by corecursive
  algebras. In: LICS'17. pp. 1--12. {IEEE} (2017)

\bibitem{ville1939etude}
Ville, J.: Etude critique de la notion de collectif  (1939)

\bibitem{sensors}
Wan, X., Zeng, L., Sun, M.: Exploring the vulnerability of deep reinforcement
  learning-based emergency control for low carbon power systems. In: IJCAI'22.
  pp. 3954--3961 (2022)

\bibitem{williams1991}
Williams, D.: Probability with martingales. Cambridge university press (1991)

\bibitem{WinklerGK22}
Winkler, T., Gehnen, C., Katoen, J.: Model checking temporal properties of
  recursive probabilistic programs. In: FOSSACS'22. pp. 449--469. Springer
  (2022)

\bibitem{xia2022accelerated}
Xia, J., Hu, M., Chen, X., Chen, M.: Accelerated synthesis of neural
  network-based barrier certificates using collaborative learning. In:
  Proceedings of the 59th ACM/IEEE Design Automation Conference. pp. 1201--1206
  (2022)

\bibitem{interval_pr}
Xu, K., Shi, Z., Zhang, H., Wang, Y., Chang, K., Huang, M., et~al.: Automatic
  perturbation analysis for scalable certified robustness and beyond. In:
  NeurIPS'20 (2020)

\bibitem{newframwork}
Xue, B.: A new framework for bounding reachability probabilities of
  continuous-time stochastic systems. CoRR  \textbf{abs/2312.15843} (2023)

\bibitem{DBLP:journals/tac/0001FZ20}
Xue, B., Fr{\"{a}}nzle, M., Zhan, N.: Inner-approximating reachable sets for
  polynomial systems with time-varying uncertainties. {IEEE} Trans. Autom.
  Control.  \textbf{65}(4),  1468--1483 (2020)

\bibitem{DBLP:conf/amcc/0001LZF21}
Xue, B., Li, R., Zhan, N., Fr{\"{a}}nzle, M.: Reach-avoid analysis for
  stochastic discrete-time systems. In: ACC'21. pp. 4879--4885. {IEEE} (2021)

\bibitem{xue2023reach}
Xue, B., Zhan, N., Fr{\"a}nzle, M.: Reach-avoid analysis for polynomial
  stochastic differential equations. {IEEE} Trans. Autom. Control.  (2023)

\bibitem{yang2021iterative}
Yang, Z., Zhang, Y., Lin, W., Zeng, X., Tang, X., Zeng, Z., Liu, Z.: An
  iterative scheme of safe reinforcement learning for nonlinear systems via
  barrier certificate generation. In: CAV'21. pp. 467--490. Springer (2021)

\bibitem{zeng2023safety}
Zeng, X., Yang, Z., Zhang, L., Tang, X., Zeng, Z., Liu, Z.: Safety verification
  of nonlinear systems with bayesian neural network controllers. In: AAAI'23.
  pp. 15278--15286 (2023)

\bibitem{vulnerability_DRL}
Zhang, H., Gu, J., Zhang, Z., Du, L., et~al.: Backdoor attacks against deep
  reinforcement learning based traffic signal control systems. Peer Peer Netw.
  Appl.  \textbf{16}(1),  466--474 (2023)

\bibitem{optimal_attack}
Zhang, H., Chen, H., Boning, D.S., Hsieh, C.: Robust reinforcement learning on
  state observations with learned optimal adversary. In: ICLR'21 (2021)

\bibitem{SAMDP}
Zhang, H., Chen, H., Xiao, C., Li, B., Liu, M., Boning, D.S., Hsieh, C.: Robust
  deep reinforcement learning against adversarial perturbations on state
  observations. In: NeurIPS'20. pp. 21024--21037 (2020)

\bibitem{ZHAO_NBC}
Zhao, H., Qi, N., Dehbi, L., Zeng, X., Yang, Z.: Formal synthesis of neural
  barrier certificates for continuous systems via counterexample guided
  learning. {ACM} Trans. Embed. Comput. Syst.  \textbf{22}(5s),  146:1--146:21
  (2023)

\bibitem{zhao2020synthesizing}
Zhao, H., Zeng, X., Chen, T., Liu, Z.: Synthesizing barrier certificates using
  neural networks. In: HSCC'20. pp. 1--11 (2020)

\bibitem{reach_avoid_martingale}
Zikelic, D., Lechner, M., Henzinger, T.A., Chatterjee, K.: Learning control
  policies for stochastic systems with reach-avoid guarantees. In: AAAI'23. pp.
  11926--11935 (2023)

\end{thebibliography}
